\documentclass[11pt,twoside]{article}

\usepackage{amsthm}
\usepackage{amsmath}
\usepackage{amssymb,bbm}

\newtheorem{lemma}{Lemma}
\newtheorem{theorem}{Theorem}
\newtheorem{proposition}{Proposition}

\def\det{\mathrm{det}} 

\def\det{\mathrm{det}} 

\newcommand{\abf}{\mathbf{a}}

\newcommand{\bbf}{\mathbf{b}}

\newcommand{\Mcal}{\mathcal{M}}

\newcommand{\Lcal}{\mathcal{L}}

\newcommand{\KL}{\mathrm{KL}}

\newcommand{\tr}{\mathsf{tr}}
\newcommand{\detsf}{\mathsf{det}}
\newcommand{\Idsf}{\mathsf{Id}}

\DeclareMathOperator*{\argmin}{arg\,min}
\newcommand{\diag}{\operatorname{\mathbf{diag}}}

\newcommand{\Sigmawtd}{\widetilde{\Sigma}}

\newcommand{\piol}{\overline{\pi}}
\newcommand{\alphaol}{\overline{\alpha}}
\newcommand{\betaol}{\overline{\beta}}
\newcommand{\Lambdaol}{\overline{\Lambda}}

\usepackage[utf8]{inputenc} 
\usepackage[T1]{fontenc}    
\usepackage{booktabs}       
\usepackage{amsfonts}       
\usepackage{nicefrac}       

\usepackage{subfig}
\usepackage{epsf}
\usepackage{epsfig}
\usepackage{fancyhdr}
\usepackage{graphics}
\usepackage{graphicx}
\usepackage{psfrag}
\usepackage{fullpage}
\usepackage{pdfpages}
\usepackage{nicefrac}

\usepackage[numbers]{natbib} 
\usepackage{url}
\usepackage{hyperref}
\hypersetup{
    colorlinks=true,
    linkcolor=magenta,
    citecolor=blue,
    pagebackref=true
}

\usepackage{amsthm}
\usepackage{amsmath}
\usepackage{amssymb,bbm}
\usepackage{caption}
\usepackage{textcomp}
\usepackage{siunitx}
\usepackage{subfig}
\usepackage{algorithmicx}
\usepackage{algorithm}
\usepackage{multirow}
\usepackage{multicol}
\usepackage{cleveref}
\usepackage{wrapfig}
\usepackage{algpseudocode}

\begin{document}
\begin{center}

{\bf{\LARGE{On Barycenter Computation: Semi-Unbalanced Optimal Transport-based Method on Gaussians}}}

\vspace*{.2in}
 {\large{
 \begin{tabular}{ccccc}
 Ngoc-Hai Nguyen$^{\star,\diamond}$& Dung Le$^{\star,\dagger}$& Hoang-Phi Nguyen$^{\diamond}$& Tung Pham$^{\diamond}$& Nhat Ho$^{\dagger}$
 \end{tabular}
}}

\vspace*{.2in}

\begin{tabular}{c}
The University of Texas, Austin$^{\dagger}$; VinAI Research$^{\diamond}$
\end{tabular}
\vspace*{.2in}

\today

\vspace*{.2in}

\begin{abstract}
  We explore a robust version of the barycenter problem among $n$ centered Gaussian probability measures, termed Semi-Unbalanced Optimal Transport (SUOT)-based Barycenter, wherein the barycenter remains fixed while the others are relaxed using Kullback-Leibler divergence. We develop optimization algorithms on Bures-Wasserstein manifold, named the Exact Geodesic Gradient Descent and Hybrid Gradient Descent algorithms. While the Exact Geodesic Gradient Descent method is based on computing the exact closed form of the first-order derivative of the objective function of the barycenter along a geodesic on the Bures manifold, the Hybrid Gradient Descent method utilizes optimizer components when solving the SUOT problem to replace outlier measures before applying the Riemannian Gradient Descent. We establish the theoretical convergence guarantees for both methods and demonstrate that the Exact Geodesic Gradient Descent algorithm attains a dimension-free convergence rate. Finally, we conduct experiments to compare the normal Wasserstein Barycenter with ours and perform an ablation study.
\end{abstract}

\let\thefootnote\relax\footnotetext{$^{\star}$ Ngoc-Hai Nguyen and Dung Le contributed equally to this work.}
\end{center}

\section{Introduction}
Aggregating multiple data sources has garnered significant interest in 
data science and artificial intelligence due to its fundamental role and wide range of applications. When we have to work with data distribution, one of the useful aggregations is the barycenter of those distributions. In the context of Optimal Transport, where the distance between distributions is defined as the optimal cost to transport masses,  this problem is known as the Wasserstein Barycenter problem, which is known to have several applications, e.g., image processing \cite{ferradans2014regularized} \cite{simon2020barycenters}, image restoration \cite{mignon2023semi}, time-series modeling \cite{cheng2021dynamical}, domain adaptation \cite{montesuma2021wasserstein}, graph representation \cite{simou2020node2coords}, signal processing~\cite{simou2019graph}, and medical multi-modal large language model~\cite{med_llm}.  

However, a common challenge arises as data often exhibits more complexity, especially in real-world scenarios where noises and outliers are prevalent, that could distort the final results of any statistical procedure. In this paper, 
we work on the robustness of the barycenter when the data measures contain noise, akin to extracting the true mean of clean distributions. It is known that a relaxed version of OT, which is Unbalanced Optimal Transport (UOT) \cite{liero2018optimal}, is able to reduce the effect of contamination thus produce robust estimation of the OT cost between corrupted distributions. The penalty function often used in the UOT is the KL divergence, which has been well-studied \cite{pham2020unbalanced, robust_OT} and preferred because of its pleasant mathematical properties and computational advantage over other divergences.  In particular, when data distributions are Gaussian, the Wasserstein distance \cite{peyre2019computational}, even the UOT cost with KL divergence has closed form solution \cite{le2022entropic}. Motivated by those nice mathematical formulas of OT and UOT costs, we aim to find out the barycenter of contaminated Gaussian distributions. The most challenging problem is that the minimization is taken over Bures manifold \cite{bhatia2009positive}, the manifold of symmetric positive definite matrices, that has a positive curvature.

\vspace{0.3 em}
\noindent
\textbf{Contributions.} Our contributions can be summarized as follows. 
We propose a framework for robust barycenter computation using the debiased effect of Semi-Unbalanced Optimal Transport (SUOT) to measure distances between barycenter and other measures. This approach aims to maintain the integrity of the true barycenter even in the presence of contamination. Our contributions include:

\begin{enumerate}
    \item We embark on the Gaussians measures $\alpha_i(s) = m_{\alpha_i} \mathcal{N}(\textbf{a}_i, \Sigma_{\alpha_i})$ residing $\mathbb{R}^d$. Motivated from previous work \cite{altschuler2021averaging} \cite{han2021riemannian} showing close form for Wasserstein distance between Gaussians, we prove that our SUOT optimization problem also has a closed-form solution. We derive two approaches for robust barycenter computation: a hybrid method using the scheme from \cite{chewi2020gradient} and a Riemannian Gradient Descent method that computes the exact Wasserstein gradient, treating our framework as a function with the barycenter as the variable.

    \item Theoretically, we provide the proof for the Gaussian form guarantee of the SUOT-based Barycenter, along with a complete proof and analysis of the convergence guarantees for the two proposed algorithms.
    
    \item Finally, we conduct experiments to demonstrate the debiased effect of our framework compared to traditional Wasserstein distance methods and highlight the advantages of our methods. We also conduct an ablation study to examine the convergence of our algorithms, comparing them against steepest descent methods that do not involve closed-form gradient calculations.
\end{enumerate}
\textbf{Organization.} In Section~\ref{sec:Background}, we recall some necessary background about OT and UOT, as well as an optimization on the Riemannian manifold. Section~\ref{sec:semi_unbalanced_OT_barycenter} proposes a framework for barycenter computation based on distance from Semi-Unbalanced Optimal Transport. Then in Section~\ref{sec:center_Gaussian} we delve into the realm of Gaussian measures, with the assumption that the Gaussians are centered at the origin, effectively working within the space of symmetric positive definite matrices, later we propose and analyze one hybrid algorithm and one first-order algorithm for robust barycenter computation. In Section~\ref{sec:numerical_experiment} we carry out numerical experiments to illustrate the performance of our approach. Section~\ref{sec:conclusion} is the conclusion and further discussion. Proofs of the theoretical results are in the supplementary material. 

\section{Background}
\label{sec:Background}
In this section, we provide necessary backgrounds for the Wasserstein Distance, Unbalanced Optimal Transport, Riemannian Optimization on Bures-Wasserstein manifold (SPD manifold) with Gradient Descent.

\vspace{0.3 em}
\noindent
\textbf{Notations.} Let $\mathcal{P}_{2, ac} \left(\mathbb{R}^d\right)$ denote the space of absolutely continuous probability measures on $\mathbb{R}^d$ with finite second moment. For any $\mu, \nu \in \mathcal{P}_{2, ac} \left(\mathbb{R}^d\right)$, the generalized Kullback-Leibler divergence is defined as $\KL(\mu \| \nu)= \int_{\mathbb{R}^d} \mu \log \left(\frac{\mu}{\nu}\right) \mathrm{d}x$. We denote the set of symmetric matrices by $\mathbb{S}^d$, symmetric semi-positive definite matrices by $\mathbb{S}_{+}^d$, and symmetric positive definite matrices by $\mathbb{S}_{++}^d$. The singular values of $\Sigma \in \mathbb{S}_{+}^d$ in descending order are $\left\{\lambda_i(\Sigma)\right\}_{i=1}^d$. The Gaussian measure on $\mathbb{R}^d$ with mean $m \in \mathbb{R}^d$ and covariance matrix $\Sigma \in \mathbb{S}_{++}^d$ is denoted $\mathcal{N}(m, \Sigma)$. The identity matrix in $\mathbb{R}^{d \times d}$ is $\Idsf$. For a vector $x \in \mathbb{R}^d$, $\diag(x)$ is the diagonal matrix with $x$ on the diagonal. We use $\|.\|_2$ as $l_2$-norm of vector in $\mathbb{R}^d$ and $\| . \|_F $ as Frobenius norm of matrix in $\mathbb{R}^{d \times d}$. Further notations will be defined as necessary.

\subsection{Wasserstein distance}
We start with a formal definition of $p$-Wasserstein distance. Given probability measures $\mu,\nu\in \mathcal{P}_{2, ac} \left(\mathbb{R}^d\right)$ with finite second moment, the $p$-Wasserstein distance between $\mu$ and $\nu$ is defined as~\cite{peyre2019computational, villani2009optimal}:
\begin{align*}
    W_p^p(\mu, \nu):=\inf _{\pi \in \Pi(\mu, \nu)} \mathbb{E}_\pi\|X-Y\|_p^p,
\end{align*}
where $X, Y$ are independent random vectors in $\mathbb{R}^d$ such that $(X, Y) \sim \pi$, $\Pi(\mu, \nu)$ denotes the set of couplings of $\mu$ and $\nu$, i.e., the set of probability measures on $\mathbb{R}^d \times \mathbb{R}^d$ whose marginals are respectively $\mu$ and $\nu$. If $\mu$ and $\nu$ admit densities with respect to the Lebesgue measure on $\mathbb{R}^d$, then the infimum is attained, and the optimal coupling is supported on the graph of a map, i.e., there exists a map $T: \mathbb{R}^d \rightarrow \mathbb{R}^d$ such that for $\pi$-a.e. $(x, y) \in \mathbb{R}^d \times \mathbb{R}^d$, it holds that $y=T(x)$. The map $T$ is called the optimal transport map from $\mu$ to $\nu$. 
The Wasserstein distance has a closed-form expression for Gaussians \cite{altschuler2021averaging}
\begin{align*}
    &W_2^2(\mathcal{N}(m, \Sigma) , \mathcal{N}(m', \Sigma')) = \| m - m' \|_2^2 + \tr\Big(\Sigma+\Sigma^{\prime}-2\left[\Sigma^{1 / 2} \Sigma^{\prime} \Sigma^{1 / 2}\right]^{1 / 2}\Big) .
\end{align*}
In this paper, we mainly work with centered Gaussians, which can be identified by their covariance matrices. Throughout, all Gaussians of interest are non-degenerate i.e. their covariances are non-singular. If there is no confusion, for a distance function $d$, we write $d\left(\Sigma, \Sigma^{\prime}\right)$ as the distance of two centered Gaussians with corresponding covariance matrices $\Sigma$, $\Sigma^{\prime}$ (for example: $W_2\left(\Sigma, \Sigma^{\prime}\right)$). Similarly, $T_{\Sigma \rightarrow \Sigma^{\prime}}$ denotes optimal transport map between these Gaussians. 

\vspace{0.3 em}
\noindent
\textbf{Barycenter Problem.} Let $P$ be a probability measure over $\mathcal{P}_{2,ac}\left(\mathbb{R}^d\right)$. Then, the Wasserstein Barycenter of $P$ is a solution of
\begin{align*}
    \underset{\mu \in \mathcal{P}_{2,ac}\left(\mathbb{R}^d\right)}{\operatorname{minimize}} \quad \int_{\mathbb{R}^d} W_2^2(\mu, \cdot) \mathrm{d} P .
\end{align*}
Cuturi et.al \cite{cuturi2014fast} formed the basis for barycenter computation, proposing an algorithm based on the dual problem and applying it to clustering and perturbed images. A related notion of average is the entropically-regularized Wasserstein Barycenter of $P$, which is defined to be a solution of
\begin{align*}
    \underset{\mu \in \mathcal{P}_{2,ac}\left(\mathbb{R}^d\right)}{\operatorname{minimize}} \quad \int_{\mathbb{R}^d} W_2^2(\mu, \cdot) \mathrm{d} P+ \mathcal{H}(\mu),
\end{align*}
where $\mathcal{H}$ is (negative) differential entropy i.e. $\mathcal{H}(\mu):=\int_{\mathbb{R}^d} \log \left(\frac{\mathrm{d} \mu}{\mathrm{d} x}\right) \mathrm{d} \mu(x)$ \cite{chizat2023doubly}. In \cite{benamou2015iterative}, the authors introduced Iterative Bregman Projections for solving entropic OT and applied it to the weighted barycenter problem. In addition, work of \cite{bonneel2015sliced} extended the barycenter problem to Radon barycenters (using the Radon transform) and Sliced barycenters (based on Sliced Wasserstein distance), with applications in color manipulation.

\subsection{Unbalanced Optimal Transport}
The Wasserstein distance cannot be defined on two measures  with unequal total masses, the marginal constraints are not satisfied.  Thus, the Unbalanced Optimal Transport (UOT) \cite{liero2018optimal} replaces those constraints by a regularizer. 
The most popular regularizer is 
Kullback–Leibler divergence $\mathrm{KL} ( \cdot | \cdot)$ which is interpreted as the average difference of the number of bits required for encoding samples of a measure using a code optimized for another measure rather than original one. Formally, we have minimization problem
\begin{align*}
     \inf _{\pi \in \mathcal{M}^+\left(\mathbb{R}^d \times \mathbb{R}^d\right)} & \mathbb{E}_\pi\|X-Y\|_2^2 +\tau \mathrm{KL}\left(\pi_x \| \mu\right) +\tau \mathrm{KL}\left(\pi_y \| \nu\right),
\end{align*}
where $X, Y$ are independent random vectors in $\mathbb{R}^d$ such that $(X, Y) \sim \pi$, $\mathcal{M}^+\left(\mathbb{R}^d \times \mathbb{R}^d\right)$ is the set of all positive measures in $\mathbb{R}^d \times \mathbb{R}^d$ and $\pi_x , \pi_y$ are the marginal distributions of the coupling $\pi$ corresponding to $\mu, \nu$, respectively. $\tau>0$ is regularized parameter. \cite{chizat2016scaling} proposed an Iterative Scaling Algorithm for solving UOT, utilizing a proximal operator for KL divergence. \cite{chapel2021unbalanced} considered UOT as a non-negative penalized linear regression problem and used a majorization-minimization method for its numerical solution. UOT is referred to as a robust extension of OT, because of its ability to handle ``unbalanced" measures characterized by varying masses. Moreover, its utility extends beyond managing global mass variations. Particularly, UOT demonstrates enhanced resilience to local mass fluctuations, encompassing outliers (which are discarded prior to transportation) and missing components \cite{sejourne2023unbalanced}.

\subsection{Riemannian Optimization with Gradient Descent}
As mentioned above, the Wasserstein distance between Gaussians has  closed-form which depends only on the means and covariances matrices. Those matrices, which are semi-definite positive (SPD) matrices, equipped with the Wasserstein distance, form a Riemannian manifold, named Bures manifold (or SPD manifold) \cite{bhatia2009positive}.
Optimization problems constrained by SPD matrices are common in computer vision and machine learning \citep{gao2020learning}, but these constraints make the problems challenging due to the manifold's non-zero curvature. Applying Euclidean gradient-based algorithms is ineffective as they do not conform the SPD manifold's geometry \cite{gao2020learning, absil2008optimization}. Instead, Riemannian gradient-based algorithms are used, leveraging the tangent space to determine directions and employing Riemannian operators to navigate the manifold. In our case, the Riemannian gradient space at point $\Sigma$ of SPD manifolds denoted by $T_{\Sigma} \mathbb{S}_{++}^d$ is identified with the space $\mathbb{S}^d$. The inner product for two matrices $A, B$, is defined as $\langle A, B\rangle_{\Sigma}:=\operatorname{tr}\left(A^{\top} \Sigma B\right)$ which would induce the tangent space norm $\| . \|_{\Sigma}$. The Riemannian exponential map 
$\mathsf{Exp} _\Sigma ( \cdot )$: $T_\Sigma \mathbb{S}_{++}^d \rightarrow \mathbb{S}_{++}^d$ maps a tangent vector to a constant-speed geodesic, while the logarithmic map $\mathsf{Log} _\Sigma ( \cdot) : \mathbb{S}_{++}^d \rightarrow T_\Sigma \mathbb{S}_{++}^d$ is its inverse. Specifically, they are given by
\begin{align*}
    \mathsf{Exp} _{\Sigma} (X)&=\left(\Idsf+ X \right) \Sigma\left(\Idsf+ X \right)\\
    \mathsf{Log}_{\Sigma} (\Sigma^{\prime})&=T_{\Sigma \rightarrow \Sigma^{\prime}}-\Idsf,
\end{align*}
respectively. For more detail, see the book \cite{ambrosio2005gradient}.

\vspace{0.3 em}
\noindent
\textbf{Riemannian Gradient Descent}:  On optimizing a differentiable function $f: \mathcal{M} \rightarrow \mathbb{R}$, denote traditional Euclidean gradient as $\nabla f$ and Wasserstein gradient as $\operatorname{grad} f$ is a tangent vector that satisfies for any $\mu \in T_x \mathcal{M},\langle\operatorname{grad} f(x), \mu\rangle_x=\mathrm{D}_\mu f(x)$, where $\mathrm{D}_\mu f(x)$ is the directional derivative of $f(x)$ along $\mu$. Instead of using traditional gradient descent update $\mu^{(t+1)}= \mu^{(t)} - \eta_t \nabla f(\mu^{(t)})$, Riemannian Gradient Descent \cite{chewi2020gradient} reads the update 
\begin{align*}
\mu^{(t+1)}=\mathsf{Exp}_{\mu^{(t)}}\Big(-\eta_t \operatorname{grad} f\big(\mu^{(t)}\big)\Big),  
\end{align*}
for some step size $\eta_t$. In this way, it ensures that the updates are along the geodesic and stay on the manifolds \cite{gao2020learning}. In the context of the SPD manifold, we also call it as Wasserstein gradient.

\section{Semi-Unbalanced Optimal Transport-based Barycenter}
\label{sec:semi_unbalanced_OT_barycenter}
We define the distance $W^2_{2_{\operatorname{SUOT}}}(\alpha, \beta, \tau)$ between two measures $\alpha$ and $\beta$ as
\begin{align} \label{uot distance} 
&W^2_{2_{\operatorname{SUOT}}}(\alpha, \beta, \tau) :=\inf _{\pi \in \mathcal{M}^+\left(\mathbb{R}^d \times \mathbb{R}^d\right)} \mathbb{E}_\pi\|X-Y\|_2^2+\tau \mathrm{KL}\left(\pi_x \| \alpha\right),
\end{align}
where $X, Y$ are independent random vectors in $\mathbb{R}^d$ such that $(X, Y) \sim \pi, Y \sim \beta$ while $\tau>0$ is regularized parameter, and $\pi_x$ is the marginal distribution of the coupling $\pi$ corresponding to $\alpha$. The intuition behind relaxing only one marginal constraint lies in the context of computing distances from measures to their barycenter. In this scenario, we desire the barycenter to remain fixed while allowing the measures to be softened. This relaxation facilitates the detection of outliers by adapting the measures slightly, ensuring that the barycenter remains representative of the overall distribution without being overly influenced by individual outliers. Now consider set of probability measure $(\alpha_i)_{i=1}^n$. To address the above problem, we proposed the empirical SUOT-based Barycenter which is  a minimization due to probability measure $\beta$
    \begin{align} \label{uot framework}
        \argmin_{\beta} L\Big( (\alpha_i)_{i=1}^n, \beta, \tau \Big) := \frac{1}{n} \sum_{i=1}^n W^2_{2_{\operatorname{SUOT}}}(\alpha_i, \beta, \tau).
    \end{align}
This SUOT-based Barycenter approach provides a robust framework for identifying a central measure that is resilient to outliers and variations within the individual measures. In our case, the data are contaminated, and we do not have access to the true data distributions, thus  KL divergence is used to penalize the contaminated data distributions. The UOT is employed when both given marginal distributions are corrupted. Our target is to find the barycenter of the set of true data distributions, thus we do not need to penalize the barycenter solution. Hence, the Semi-UOT is applied to solve our problem. Toy example in \Cref{compare} demonstrates an explanation to the use of Semi-UOT approach.

\section{Studies On Centered Gaussians}
\label{sec:center_Gaussian}
In this section, our paper explores SUOT-based Barycenter computation $\beta$ for centered Gaussian distributions $\left( \alpha_i\right)_{i=1}^n$. \Cref{sec: 4.1} proves that SUOT has a closed-form solution for such Gaussians while preserving their Gaussian form for barycenter. \Cref{sec: 4.2} introduces the Exact Geodesic Gradient Descent algorithm, utilizing the closed-form Wasserstein gradient of SUOT distance and providing convergence guarantees. \Cref{sec: 4.3} presents a Hybrid Gradient Descent algorithm, alternating between SUOT minimization and Riemannian Gradient Descent for normal Wasserstein distance.
\subsection{Semi-Unbalanced Optimal Transport has a Closed-Form Expression for Gaussians} \label{sec: 4.1}
 Previous work \cite{agueh2011barycenters} showed that the normal Wasserstein barycenter of Gaussians distributions, for the $\ell_2$ cost is also a Gaussian. In our work, we also demonstrate a similar result for SUOT-based Barycenter.
\begin{theorem} \label{theo: bary is gaussian}
    Let $(\alpha_i)_{i=1}^n$ be zero-mean Gaussian distributions in $\mathbb{R}^d$ having covariance matrices $(\Sigma_i)_{i=1}^n$. Let $\Sigma_\beta \in \mathbb{S}_{++}^{d}$ is a SPD matrix. Consider the problem SUOT-based Barycenter in \Cref{uot framework}
    \begin{align*}
        \argmin_{\beta \in \mathcal{P}(\Sigma_\beta)}  \frac{1}{n} \sum_{i=1}^n W^2_{2_{\operatorname{SUOT}}}(\alpha_i, \beta, \tau).
    \end{align*}
    where $\mathcal{P}(\Sigma_\beta)$ be the set of zero-mean probability distribution in $\mathbb{R}^d$ having covariance matrix $\Sigma_\beta$. Then, $\beta$ is itself a Gaussian.
\end{theorem}
The result of Theorem~\ref{theo: bary is gaussian} asserts that the Gaussian form of the barycenter is preserved when the Kullback–Leibler (KL) divergence is incorporated into the formula. This holds because the minimum of the KL divergence, given specified means and covariance matrices, results in a Gaussian distribution. We give the proof for the \Cref{theo: bary is gaussian} in Appendix \ref{proof for bary is gaussian}. Then, Theorem \ref{UOT Plan} below shows the closed form for the solution of problem \eqref{uot distance} in the case of Gaussians.
\begin{theorem} \label{UOT Plan}
    Consider two Gaussian measures with masses in $\mathbb{R}^d$:
$
\mathbf{\alpha}=m_\alpha \mathcal{N}\left(\mathbf{a}, \Sigma_\alpha\right) \text { and } \beta =m_\beta \mathcal{N}\left(\mathbf{b}, \Sigma_\beta\right).
$
 Assume that $\Sigma_\alpha, \Sigma_\beta$ are SPD matrices. Consider minimization problem \eqref{uot distance}.
Assume that optimal solution $\pi^*$ is a positive measure such that $\pi^*=m_\pi \overline{\pi}$ with $\overline{\pi}$ is a probability measure with mean $\left(\mathbf{a}_x, \mathbf{b}\right)$ and covariance matrix
\begin{align*}
\Sigma_\pi=\left(\begin{array}{cc}
\Sigma_x &K_{x \beta} \\
K_{x \beta}^{\top} & \Sigma_\beta
\end{array}\right).
\end{align*}
We denote 
\begin{align*}
    \Sigma_{\alpha, \tau} & = \Idsf + \frac{\tau}{2} \Sigma_{\alpha}^{-1},\qquad  
    \Sigma_{\alpha,\tau,\beta} = \Sigma_{\beta}^{-\frac{1}{2}} \Sigma_{\alpha, \tau} \Sigma_\beta^{-\frac{1}{2}}, \\
    \Sigma_\gamma & =\frac{\tau}{2} \Idsf+\frac{1}{2} \Sigma_{\alpha, \tau, \beta}^{-1}\left(\Idsf+\left(\Idsf+ \tau \Sigma_{\alpha, \tau, \beta}\right)^{\frac{1}{2}}\right), \\
    \mathbf{S}_1 & = \frac{\tau}{2}\Sigma_{\alpha,\tau,\beta}^{-1} + \frac{1}{2}\Sigma_{\alpha,\tau,\beta}^{-2}\left(\Idsf + \left(\Idsf + 2 \tau \Sigma_{\alpha,\tau,\beta} \right)^{\frac{1}{2}} \right) \\
    \mathbf{S}_2 & = \tr(\Sigma_\gamma) +  \tr(\Sigma_\beta) + \tr\left(\Big[\Sigma_{\alpha, \tau, \beta}^{-1} \Sigma_\gamma\Big]^{\frac{1}{2}}\right), \\
    \mathbf{S}_3 & = - \frac{\tau}{2} \log \left(\det\Big[\Sigma_\gamma \Sigma_{\alpha, \tau, \beta}^{-1} \Sigma_\beta^{-1} \Sigma_\alpha^{-1}\Big]\right), \\
    \mathbf{S}_4 & = \left( \left\|\Sigma_{\alpha, \tau}^{-1}\right\|^2_F + 1 \right ) \Idsf -2 \Sigma_{\alpha, \tau}^{-1}+ \frac{\tau}{2} \Sigma_{\alpha, \tau}^{-1} \Sigma_\alpha^{-1} \Sigma_{\alpha, \tau}^{-1}, \\
    \mathbf{S}_5 & = (\mathbf{a} - \mathbf{b})^{\top} \mathbf{S}_4 (\mathbf{a}-\mathbf{b}) - \frac{\tau d}{2}, \\
    \Upsilon & = \mathbf{S}_2 + \mathbf{S}_3 + \mathbf{S}_5.
\end{align*}
Then, we find that 
\begin{align*}
\mathbf{a}_x  & = \Sigma_{\alpha, \tau}^{-1}(\mathbf{b}-\mathbf{a}) + \mathbf{a}, \quad \Sigma_x  = \Sigma_{\beta}^{-\frac{1}{2}} \mathbf{S}_1 \Sigma_{\beta}^{-\frac{1}{2}}, \\
K_{x,\beta} & = \Lambda, \quad m_{\pi} = m_{\alpha} \exp \left\{\frac{-\Upsilon}{\tau}\right\},
\end{align*}
where $\Lambda$ is a diagonal matrix containing singular values of $\Sigma_{\beta}^{\frac{1}{2}} \Sigma_x^{\frac{1}{2}}$ in descending order. Moreover, for two centered Gaussian distribution $\alpha, \beta \quad (m_\alpha = m_\beta = 1; \mathbf{a} = \mathbf{b} = 0)$, we have
\begin{align*}
W^2_{2_{\operatorname{SUOT}}}(\Sigma_\alpha, \Sigma_\beta) = W_2^2(\Sigma_x, \Sigma_\beta) + \tau \mathrm{KL}(\Sigma_x \| \Sigma_\alpha),
\end{align*}
\end{theorem}
Our proof follows strategy in \cite{le2022entropic} and is given in Appendix \ref{proof for uot_distance}. On the other hand, we present a proposition linking our optimizers $\Sigma_x$ with $\Sigma_\beta$ and $\Sigma_\alpha$ through hyperparameter $\tau$ in \Cref{Ablation study}. In Theorem \ref{UOT Plan}, $m_{\alpha}$ and $m_{\beta}$ are the scales of  Gaussian measures, when the scale is equal 1, we obtain Gaussian distribution. It also becomes apparent in the proof that SUOT distance between two entities can be succinctly expressed as the sum of the 2-Wasserstein distance and a KL divergence term. 
In the same way, we could derive the closed-form for sparse solution achieved by Semi-Unbalanced Entropic Optimal Transport, which is given by adding a regularization term $\mathrm{KL} \left(\pi \| \alpha \otimes \beta \right)$. In particular, we have Theorem \ref{UOT Entropic} as follows.
\begin{theorem} \label{UOT Entropic}
    Consider two centered Gaussian measures in $\mathbb{R}^d$ :
    \begin{align*}
        \mathbf{\alpha}= \mathcal{N}\left(\mathbf{0}, \Sigma_\alpha\right) \text { and } \beta = \mathcal{N}\left(\mathbf{0}, \Sigma_\beta\right).
    \end{align*}
Assume that $\Sigma_\alpha, \Sigma_\beta$ are SPD matrices. Consider minimization problem
\begin{align*}
& W^2_{2_{\operatorname{SUOT}}, \delta} (\alpha, \beta;\tau) := \min _{\pi \in \mathcal{M}^+\left(\mathbb{R}^d \times \mathbb{R}^d\right)} \mathbb{E}_\pi\|X-Y\|^2+\tau \mathrm{KL}\left(\pi_x \| \alpha\right) + \delta \mathrm{KL} \left(\pi \| \alpha \otimes \beta \right), 
\end{align*}
where $X, Y$ are independent random vectors in $\mathbb{R}^d$ such that $(X, Y) \sim \pi, Y \sim \beta$ while $\tau>0$ is regularized parameter, and $\pi_x$ is the marginal distribution of the coupling $\pi$ corresponding to $\alpha$. We note that $\Sigma_{\alpha \otimes \beta} = \left(\begin{array}{cc}
\Sigma_\alpha &\mathbf{0}_{d \times d} \\
\mathbf{0}_{d \times d} & \Sigma_\beta
\end{array}\right)$. 
Assume that optimal solution $\pi^*$ is a probability measure with mean $\textbf{0}_{d \times d}$ and covariance matrix
\begin{align*}
\Sigma_\pi=\left(\begin{array}{cc}
\Sigma_x &K_{x \beta} \\
K_{x \beta}^{\top} & \Sigma_\beta
\end{array}\right).
\end{align*}
Moreover, $\delta$ is small enough that all eigenvalues of $\Sigma_\beta^{\frac{1}{2}} \Sigma_x^{\frac{1}{2}}$  are not smaller than $\frac{\delta}{4}$. Denote
\begin{align*}
    \Sigma_{\alpha, \beta, \tau, \delta} = \Sigma_\beta^{-\frac{1}{2}} \left( \Idsf + \frac{\tau + \delta}{2} \Sigma_\alpha^{-1}  \right) \Sigma_\beta^{-\frac{1}{2}}.
\end{align*}
Then we have
\begin{align*}
    K_{x \beta} & = \Sigma_x^{\frac{1}{2}} \Sigma_\beta^{\frac{1}{2}} - \frac{\delta}{4} \Idsf, \\
    \Sigma_{x} & = \Sigma_\beta^{-\frac{1}{2}} \bigg[ \frac{\tau}{2}\Sigma_{\alpha,\beta,\tau, \delta}^{-1} + \frac{1}{2}\Sigma_{\alpha,\beta,\tau, \delta}^{-2} \mathbf{S} \bigg] \Sigma_\beta^{-\frac{1}{2}},
\end{align*}
where $\mathbf{S} = \Big[\Idsf + \big(\Idsf + (2\tau + 3\delta) \Sigma_{\alpha,\beta,\tau, \delta} \big)^{\frac{1}{2}} \Big]$.
\end{theorem}
The proof is given in Appendix \ref{proof for uot_distance}. We observe that both hyperparameters, $\tau$ and $\delta$, contribute to the structure of $\Sigma_x$. Next, we will present our algorithms: Exact Geodesic Gradient Descent and Hybrid Gradient Descent, which focus on updating the covariance matrices of the barycenter for centered Gaussians. Without loss of generality, we assume that all weights are equal.

\subsection{Exact Geodesic Gradient Descent for SUOT-based Barycenter} \label{sec: 4.2}
We show that $W^2_{2_{\operatorname{SUOT}}}\left(\alpha, \beta, \tau \right)$ has a closed form for its Wasserstein gradient on Bures manifold, then it is more precise to apply Riemannian Gradient Descent on SPD manifold for our barycenter problem.
\begin{theorem} \label{UOT Derivative}
Consider $W^2_{2_{\operatorname{SUOT}}}(\Sigma_\alpha, \Sigma_\beta, \tau)$ where $\Sigma_\alpha, \tau$ are fixed and $\Sigma_\beta$ is seen as the variable. Then the Wasserstein gradient of $W^2_{2_{\operatorname{SUOT}}}(\Sigma_\alpha, \Sigma_\beta, \tau)$ with respect to $\Sigma_\beta$ on Bures manifold is formulated by
\begin{align*}
    2 \Idsf - \Big( 2\Sigma_{\alpha,\tau}^{-1} + \frac{1}{2} (U+ \tau M)\Big ) + \frac{3}{2} \tau \Sigma_\beta^{-1} + \frac{\tau^2}{2}(P+Q),
\end{align*}
where
\begin{align*}
&\widetilde{\Sigma}_{\beta,\alpha,\tau}  = \bigg\{\Big[\Sigma_{\alpha,\tau}^{-\frac{1}{2}}\Sigma_{\beta}\Sigma_{\alpha,\tau}^{-\frac{1}{2}} \Big]^2 + \tau \Big[\Sigma_{\alpha,\tau}^{-\frac{1}{2}}\Sigma_{\beta }\Sigma_{\alpha,\tau}^{-\frac{1}{2}} \Big]\bigg\}^{\frac{1}{2}}, \\
& M = \Sigma_{\alpha,\tau}^{-\frac{1}{2}}\widetilde{\Sigma}^{-1}_{\beta,\alpha,\tau}\Sigma_{\alpha,\tau}^{-\frac{1}{2}}, \quad 
 U = \Sigma_{\alpha,\tau}^{-1}\Sigma_\beta M+M\Sigma_\beta \Sigma_{\alpha,\tau}^{-1}, \\
 &V = \Big[ \Idsf + \tau \Sigma_{\alpha,\tau}^{\frac{1}{2}} \Sigma_{\beta}^{-1} \Sigma_{\alpha,\tau}^{\frac{1}{2}}\Big]^{\frac{1}{2}},\\
 &P  = \Sigma_\beta^{-1} \Sigma_{\alpha,\tau}^{\frac{1}{2}} \big[\Idsf + V\big]^{-1}\Sigma_{\alpha,\tau}^{\frac{1}{2}}\Sigma_\beta^{-1}, \\
& Q  = \Sigma_{\alpha,\tau}^{\frac{1}{2}} \big[\Idsf + V\big]^{-1}  \Sigma_{\alpha,\tau}^{\frac{1}{2}} \Sigma_\beta^{-2}.
\end{align*}
\end{theorem}
This directly leads to the subsequent formula for a condition of barycenter of $\left(\Sigma_{\alpha_i} \right)_{i=1}^n$ by summing the first derivatives of the components. The proof is given in Appendix \ref{proof for uot derivative}. Consequently, thanks to the closed form of the Wasserstein gradient, we derive the details of the Exact Geodesic Bures-Wasserstein Gradient Descent used to solve the SUOT-based Barycenter problem, as presented in Algorithm $\mathbf{1}$ with the initial point $\Sigma_{\beta^{(0)}}$ and learning rate $\eta$.
\begin{algorithm}[t!] \label{algo: exact}
\caption{Exact Geodesic Bures-Wasserstein Gradient Descent}
\begin{algorithmic}
    \Require $\mathcal{P} = \left(\mathcal{N}(\mathbf{0},\Sigma_{\alpha_i})\right)_1^n, \mathcal{N}(\mathbf{0},\Sigma_{\beta}^{(0)}), \eta, T, \epsilon$
    \For{$k = 1, \ldots, T$}
            \State 
            $\mathbf{G}^{(k)}_1 = 2 \Idsf + \frac{3}{2} \tau \Sigma_\beta^{(k-1)}$
            \State $M^{(k)}_i = \Sigma_{\alpha_i,\tau}^{-\frac{1}{2}}[\widetilde{\Sigma}^{(k-1)}_{\beta,\alpha_i,\tau}]^{-1}\Sigma_{\alpha_i,\tau}^{-\frac{1}{2}}$ 
            \State $U^{(k)}_i = \Sigma_{\alpha_i,\tau}^{-1}\Sigma_\beta^{(k-1)} M^{(k)}_i+M^{(k)}_i\Sigma_\beta^{(k-1)} \Sigma_{\alpha_i,\tau}^{-1}$
            \State $\mathbf{G}^{(k)}_2 = \nicefrac{1}{n} \sum_{i=1}^n  \big[2\Sigma_{\alpha_i,\tau}^{-1} + \frac{1}{2}(U^{(k)}_i+ \tau M^{(k)}_i) \big]$
            \State $V^{(k)}_i = \Big[ \Idsf + \tau \Sigma_{\alpha_i,\tau}^{\frac{1}{2}} [\Sigma_{\beta}^{(k-1)}]^{-1} \Sigma_{\alpha_i,\tau}^{\frac{1}{2}}\Big]^{\frac{1}{2}}$
            \State $P^{(k)}_i \!=\![\Sigma_{\beta}^{(k-1)}]^{-1} \Sigma_{\alpha_i,\tau}^{\frac{1}{2}} \big[\Idsf + V^{(k)}_i\big]^{-1}\Sigma_{\alpha_i,\tau}^{\frac{1}{2}}[\Sigma_{\beta}^{(k-1)}]^{-1}$
            \State $Q^{(k)}_i = \Sigma_{\alpha_i,\tau}^{\frac{1}{2}} \big[\Idsf + V^{(k)}_i\big]^{-1}  \Sigma_{\alpha_i,\tau}^{\frac{1}{2}} [\Sigma_{\beta}^{(k-1)}]^{-2}$
            \State $\mathbf{G}^{(k)}_3 =  \nicefrac{1}{n} \sum_{i=1}^n (P^{(k)}_i+Q^{(k)}_i)$
            \State $\mathbf{G}^{(k)} = \eta \left(\mathbf{G}^{(k)}_1 - \mathbf{G}^{(k)}_2 + \frac{\tau^2}{2} \mathbf{G}^{(k)}_3 \right)$
            \State $\Sigma_{\beta}^{(k)} = \mathbf{G}^{(k)} \Sigma_{\beta}^{(k-1)} \mathbf{G}^{(k)}$
    \If{$\left\| W^2_{2_{\operatorname{SUOT}}} 
    \left( \mathcal{P}, \Sigma_{\beta}^{(k-1)} \right)
    \!-\! W^2_{2_{\operatorname{SUOT}}} 
    \left( \mathcal{P}, \Sigma_{\beta}^{(k)} \right)
    \right\| \! \leq \epsilon$}
        \State \textbf{Output:} $\Sigma_{\beta} = \Sigma_{\beta}^{(k)}$ which is the solution of the barycenter problem.
    \EndIf
    \EndFor
    \State \textbf{Output:} $\Sigma_{\beta} = \Sigma_{\beta^{(T)}}$ which is the solution of the barycenter problem.
\end{algorithmic}
\end{algorithm}

\textbf{Remark 1:} Our consideration focusing on centered Gaussians does not lose the generality. In fact, the update equation for the descent step decomposes into two parts: one for the mean and one for the covariance matrix. However, the updated equation for the mean is straightforwardly inferred from Theorem \ref{UOT Plan}. Specifically, by denoting
\begin{align*}
   M_i = \left( \left\|\Sigma_{\alpha_i, \tau}^{-1}\right\|_F^2 + 1 \right ) \Idsf -2 \Sigma_{\alpha_i, \tau}^{-1}+ \frac{\tau}{2} \Sigma_{\alpha_i, \tau}^{-1} \Sigma_{\alpha_i}^{-1} \Sigma_{\alpha_i, \tau}^{-1},
\end{align*}
then taking the first derivative of $W^2_{2_{\operatorname{SUOT}}}$ with respect to $\mathbf{b}$ is equivalent to taking the first derivative of $(\mathbf{a}_i - \mathbf{b})^T M_i (\mathbf{a}_i - \mathbf{b})$ with respect to $\mathbf{b}$, which is $-2 M_i (\mathbf{a} - \mathbf{b})$. Summing over $\left(\alpha_i \right)_{i=1}^n$ yields
\begin{align*}
    \mathbf{b} = \Big( \sum_{i=1}^n M_i\Big)^{-1} \Big(\sum_{i=1}^n M_i \mathbf{a}_i \Big).
\end{align*} 

\begin{theorem} \label{theo: exact converges}
    Suppose we apply Exact Geodesic Bures-Wasserstein Gradient Descent Algorithm with starting points $\Sigma_\beta^{(0)} \in \mathcal{K}_{[1 / \rho, \rho]} := \{\Sigma \in \mathbb{S}_{++}^d | \frac{1}{\rho} \leq \lambda_i(\Sigma) \leq \rho \quad \forall i = 1, \ldots,d \}$  for fixed $\rho$ with learning rate $\eta$ and all the updated matrices lie in $\mathcal{K}_{[1 / \rho, \rho]}$ 
    , then the algorithm converges to an optimal solution $\Sigma_\beta^{\star}$. Moreover, we have convergence guarantees at k-th iteration
    \begin{align*}
    \mathcal{D}\left(\Sigma_\beta^{(k)}\right) \leq \left(1- \frac{8\tau^2 \eta(1 - \frac{\eta}{2})}{\rho(\rho^2 + 2 \tau \rho)^{\frac{3}{2}}}\right)^k \mathcal{D}\left(\Sigma_\beta^{(0)}\right),
\end{align*}
where $\mathcal{D}\left(\Sigma_\beta^{(k)}\right) = L\left(\Sigma_\beta^{(k)}\right)-L\left(\Sigma_\beta^{\star}\right)$ is distance from objective function at k-th iteration to the optimal value.
\end{theorem}

We give the full proof in Appendix \ref{proof for exact converges}. The condition on the set $\mathcal{K}_{[1 / \rho, \rho]}$ is natural in the context of finite distributions and is also employed in \cite{altschuler2021averaging}. This condition is critical for ensuring a closed-form expression for the strong convexity parameter, which is necessary for estimating the convergence rate. Our \Cref{theo: exact converges} establishes the algorithm's geometric convergence, with the added novelty that the convergence rate is dimension-free, a conclusion also reached in \cite{altschuler2021averaging}.

\subsection{Hybrid Algorithm for Barycenter UOT} \label{sec: 4.3}
\begin{figure}{}
    \centering
    \includegraphics[width=0.6\linewidth]{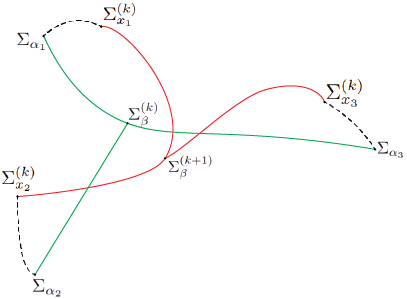}
    \caption{Overview of hybrid updated iteration}
    \label{hybrid idea}
\end{figure}
This part is motivated by \cite{chewi2020gradient} within the context of detecting robust terms. Our algorithm has two alternated steps: one step is to find the minimizer under the regularization of KL divergence, the other step is to apply the Riemannian Gradient Descent as in the work of \cite{chewi2020gradient}. Figure \ref{hybrid idea} demonstrates the step in Algorithm $\mathbf{2}$. Given set of SPD matrices $\left(\Sigma_{\alpha_i}\right)_{i=1}^n$ and starting point $\Sigma_\beta^{(0)}$; at $k$-th update $\Sigma_\beta^{(k)}$, we first find $\Sigma_{x_i}^{(k)}$ which are minimizer of $W_2^2(\Sigma_{x_i}^{(k)},\Sigma_{\beta}^{(k)}) + \tau \text{KL}(\Sigma_{x_i}^{(k)}\|\Sigma_{\alpha_i})$ through \Cref{UOT Plan}. Then given $\left(\Sigma_{x_i}^{(k)}\right)_{i=1}^n$, we find $\Sigma_{\beta}^{(k+1)}$ as Wasserstein Barycenter of them. In every iteration, the total objective function decreases, so we will arrive at a solution. 

\begin{algorithm}[t!] 
\centering
\caption{Hybrid Bures-Wasserstein Gradient Descent}
\begin{algorithmic}
    \Require $\mathcal{P} = \left(\mathcal{N}(\mathbf{0},\Sigma_{\alpha_i})\right)_1^n, \mathcal{N}(\mathbf{0},\Sigma_{\beta}^{(0)}), T, \epsilon$
    \For{$k = 1, \ldots, T$}
        \For{$i = 1, \ldots, n$}
            \State $\Sigma_{\alpha_i,\tau,\beta}^{(k-1)} = \left(\Sigma_{\beta}^{(k-1)}\right)^{-\frac{1}{2}} \left(\Idsf + \frac{\tau}{2} \Sigma_{\alpha_i}^{-1}\right)\left(\Sigma_{\beta}^{(k-1)}\right)^{-\frac{1}{2}}$
            \State $\Sigma_{\gamma_i}^{(k)} \! = \! \frac{\tau}{2} \Idsf \! + \! \frac{1}{2} \left(\Sigma_{\alpha_i, \tau, \beta}^{(k-1)}\right)^{-1}\Big[\Idsf + \left(\Idsf + 2 \tau \Sigma_{\alpha_i,\tau,\beta}^{(k-1)}\right)^{\frac{1}{2}}\Big]$
            \State $\Sigma_{x_i}^{(k)} = \left(\Sigma_{\beta}^{(k-1)}\right)^{-\frac{1}{2}} \left(\Sigma_{\alpha_i, \tau, \beta}^{(k-1)}\right)^{-1} \Sigma_{\gamma_i}^{(k)} \left(\Sigma_{\beta}^{(k-1)}\right)^{-\frac{1}{2}}$
        \EndFor
        \State $\Sigma_{\beta, x_i}^{(k-1)} =  \Big[ \left(\Sigma_{\beta}^{(k-1)}\right)^{\frac{1}{2}} \Sigma_{x_i}^{(k)} \left(\Sigma_{\beta}^{(k-1)}\right)^{\frac{1}{2}}\Big] \forall i = 1, \dots, n$
        \State $\mathbf{S}^{(k)} = \frac{1}{n} \sum_{i=1}^n \left(\Sigma_{\beta}^{(k-1)}\right)^{-\frac{1}{2}} \left(\Sigma_{\beta, x_i}^{(k-1)}\right)^{\frac{1}{2}} \left(\Sigma_{\beta}^{(k-1)}\right)^{-\frac{1}{2}}$
        \State $\Sigma_{\beta}^{(k)} = \mathbf{S}^{(k)} \Sigma_{\beta}^{(k-1)} \mathbf{S}^{(k)}$
    \If{$\left\| W^2_{2_{\operatorname{SUOT}}} 
    \left( \mathcal{P}, \Sigma_{\beta}^{(k-1)} \right)
    \!-\! W^2_{2_{\operatorname{SUOT}}} 
    \left( \mathcal{P}, \Sigma_{\beta}^{(k)} \right)
    \right\| \! \leq \epsilon$}
        \State \textbf{Output:} $\Sigma_{\beta} = \Sigma_{\beta}^{(k)}$ which is the solution of the barycenter problem.
    \EndIf
    \EndFor
    \State \textbf{Output:} $\Sigma_{\beta} = \Sigma_{\beta^{(T)}}$ which is the solution of the barycenter problem.
\end{algorithmic}
\end{algorithm}

\begin{theorem} \label{theo: hybrid converges}
    Suppose we apply Hybrid Bures-Wasserstein Gradient Descent Algorithm with starting points $\Sigma_\beta^{(0)} \in \mathcal{K}_{[1 / \rho, \rho]} := \{\Sigma \in \mathbb{S}_{++}^d | \frac{1}{\rho} \leq \lambda_i(\Sigma) \leq \rho \quad \forall i = 1, \ldots,d \}$  for fixed $\rho$ and all the updated matrices lie in $\mathcal{K}_{[1 / \rho, \rho]}$, then the algorithm converges to an optimal solution.
\end{theorem}
The full proof is given in Appendix \ref{proof for hybrid converges}. This is meaningful when considering a scenario where we have contaminated Gaussian data, denoted as $\Sigma_\alpha$s, and assume the true underlying data follows a Gaussian distribution, $\Sigma_x$s. Intuitively, it is preferable to recover the true data before performing calculations on it.

\section{Numerical Experiments}
\label{sec:numerical_experiment}
In this section, we provide numerical evidence regarding our presented SUOT-based Barycenter with normal Wasserstein Barycenter, as well as some ablation studies on optimization methods. To generate a sample of the SPD matrix, we suggest the strategy in \cite{chewi2020gradient}. Let $A_i$ be a matrix with entries are i.i.d. samples from $\mathcal{N}(0, \sigma^2)$. Our random sample on the Bures manifold is then given by taking a symmetric version of $A_i$ as $\frac{A_i+A_i^{\top}}{2}$, then applying matrix exponential function $\operatorname{\textbf{expm}}(A)$ : $\mathbb{S}^d \rightarrow \mathbb{S}_{++}^d$ to this symmetric version. Particularly, if $P \diag(\lambda_i) P^{\top}$ is the SVD decomposition of $A$, then
\begin{align*}
    \operatorname{\textbf{expm}}(A)=P \diag (e^{\lambda_i}) P^{\top}.
\end{align*}
Figure \ref{compare} shows the normal Wasserstein barycenter and our SUOT-based Barycenter in the presence of an outlier. We depict two 2D contour Gaussians and add noise to the left one by computing the weighted sum of the two PDFs with an outlier Gaussian. We then calculate and plot both the normal Wasserstein and SUOT-based Barycenters. The SUOT-based Barycenter is less affected by noise and more closely resembles the barycenter of the two original Gaussians compared to the normal Wasserstein barycenter. Specifically, the first subfigure presents two clean Gaussian distributions and their barycenter. In the next subfigure, noise (represented by an additional Gaussian) is added to the left Gaussian, demonstrating the challenge of recovering the uncontaminated barycenter. The third subfigure shows the Wasserstein Barycenter, which deviates from the true barycenter seen in the first subfigure. In contrast, the fourth subfigure displays the SUOT-based Barycenter, which more accurately reflects the true barycenter despite the added noise. To quantify this, we compute the standard Wasserstein distance between the true barycenter and both the Wasserstein and SUOT-based Barycenters, obtaining values of 0.2673 and 0.06, respectively. 

\begin{figure*}[!htp]
    \centering
    \subfloat{\includegraphics[width=0.25\linewidth]{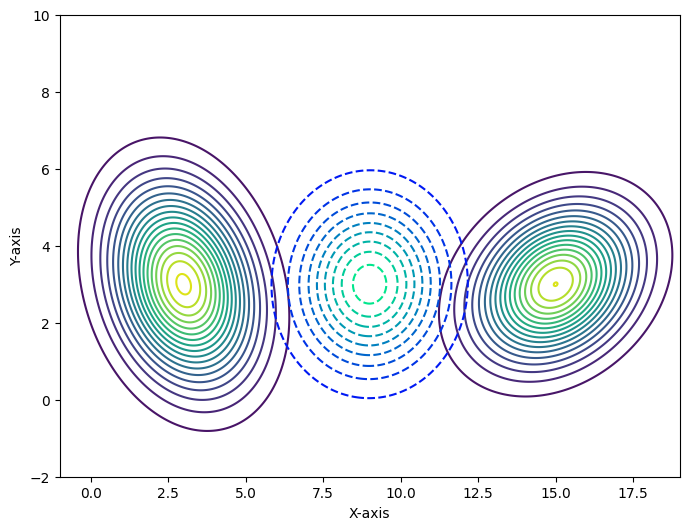}\label{compare}}
    \subfloat{\includegraphics[width=0.25\linewidth]{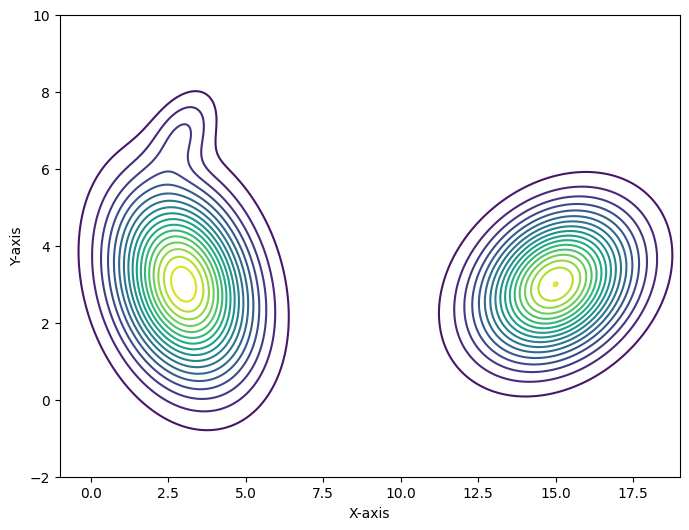}}
    \subfloat{\includegraphics[width=0.25\linewidth]{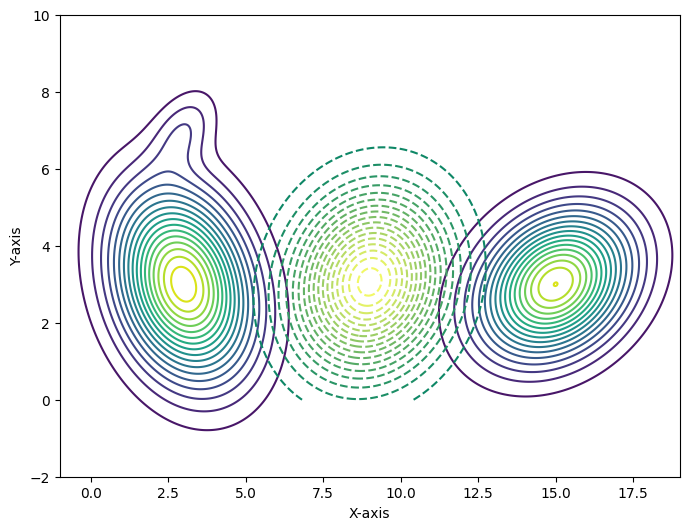}}
    \subfloat{\includegraphics[width=0.25\linewidth]{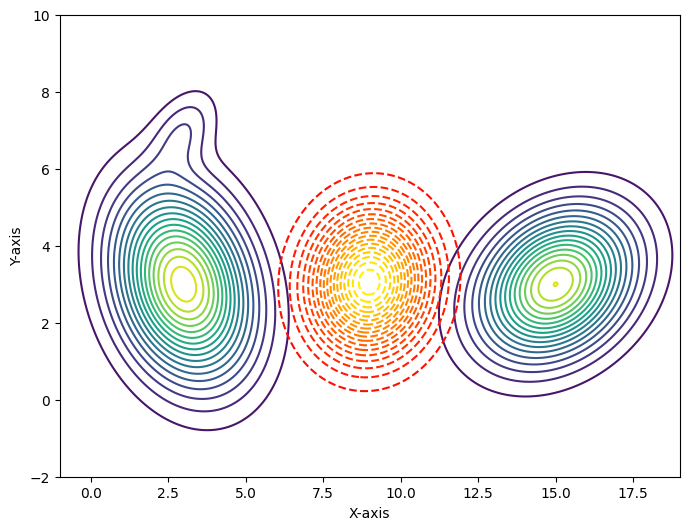}}
    \caption{2D Contour Plot Gaussian Mixture Distribution. From left to right: two Gaussians with their barycenter (\textcolor{blue}{blue}); noise is added to one Gaussian on the left; normal Wasserstein Barycenter (\textcolor{green}{green}); SUOT-based Barycenter (\textcolor{red}{red}).}
\end{figure*}
For the optimization study, we run Gradient Descent (GD) and Stochastic Gradient Descent (SGD) to observe the loss objective's behaviour over iterations. We generate a dataset of
 20 centered Gaussians in $\mathbb{R}^5$, with covariance matrices as diagonal SPD matrices for efficiency and tractability. SGD uses each sample matrix only once. Figure \ref{rate} shows the convergence of GD and SGD for distributions on the Bures manifold. 
\begin{figure*}[!t]
    \centering
    \subfloat{\includegraphics[width=0.35\linewidth]{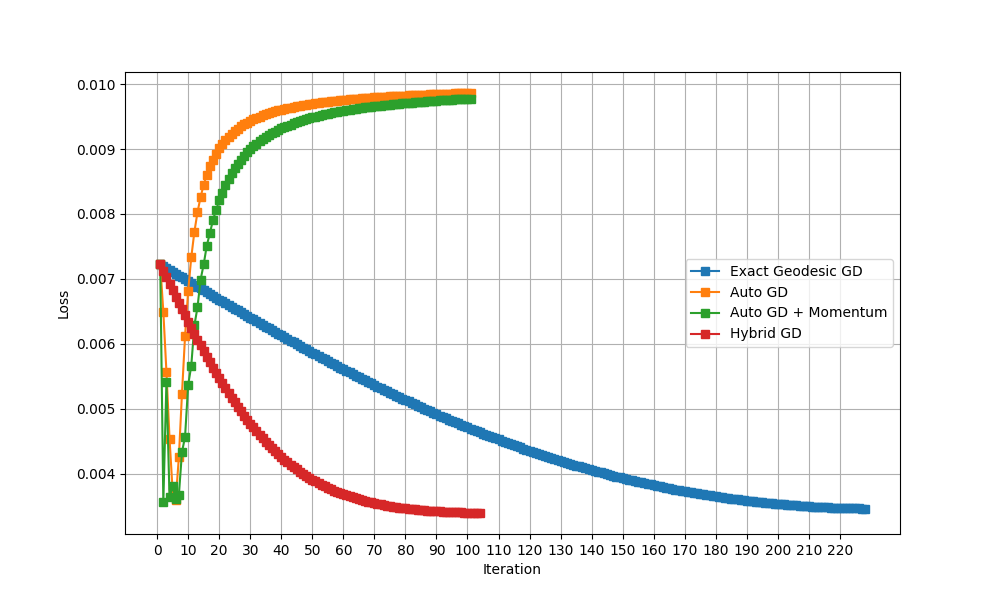}\label{rate}}
    \subfloat{ \includegraphics[width=0.35\linewidth]{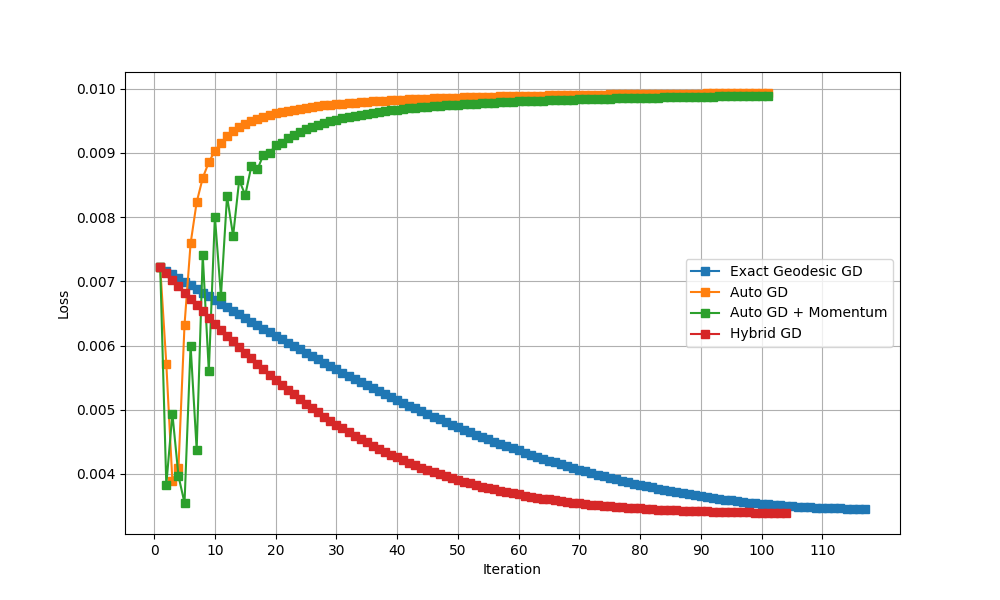}} 
    \subfloat{\includegraphics[width=0.35\linewidth]{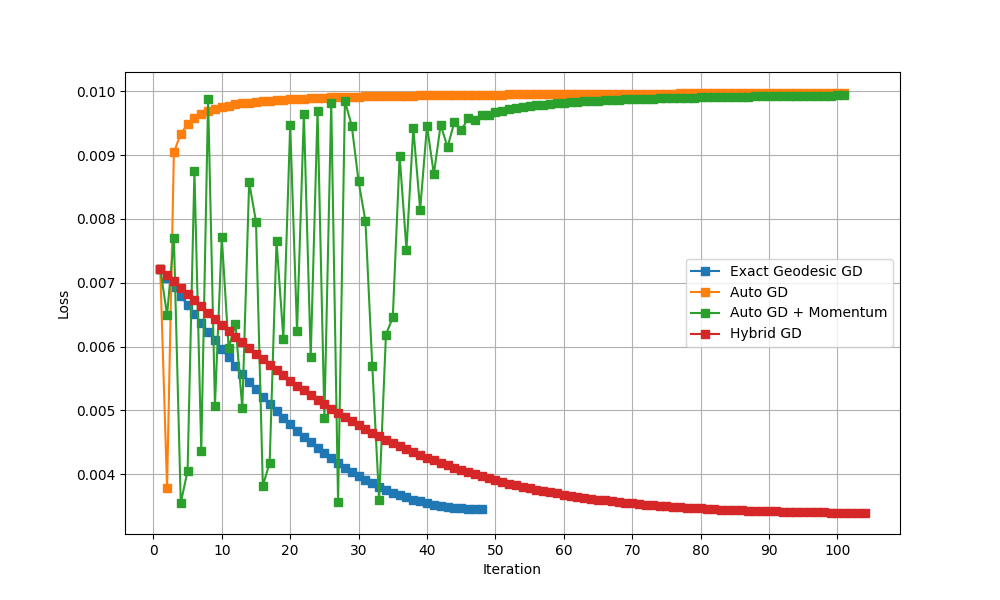}} \\
    \subfloat{\includegraphics[width=0.35\linewidth]{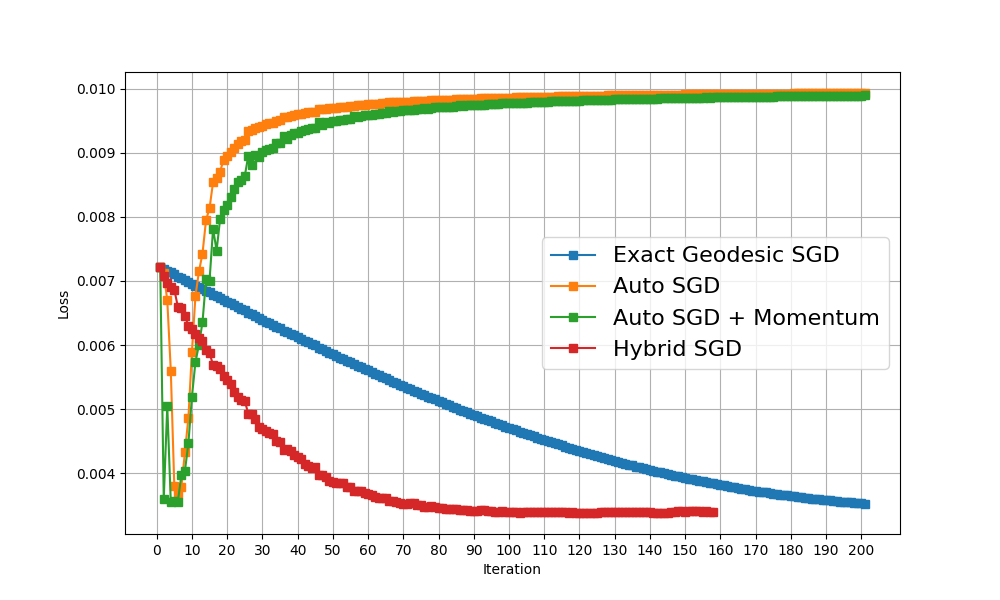}}
    \subfloat{\includegraphics[width=0.35\linewidth]{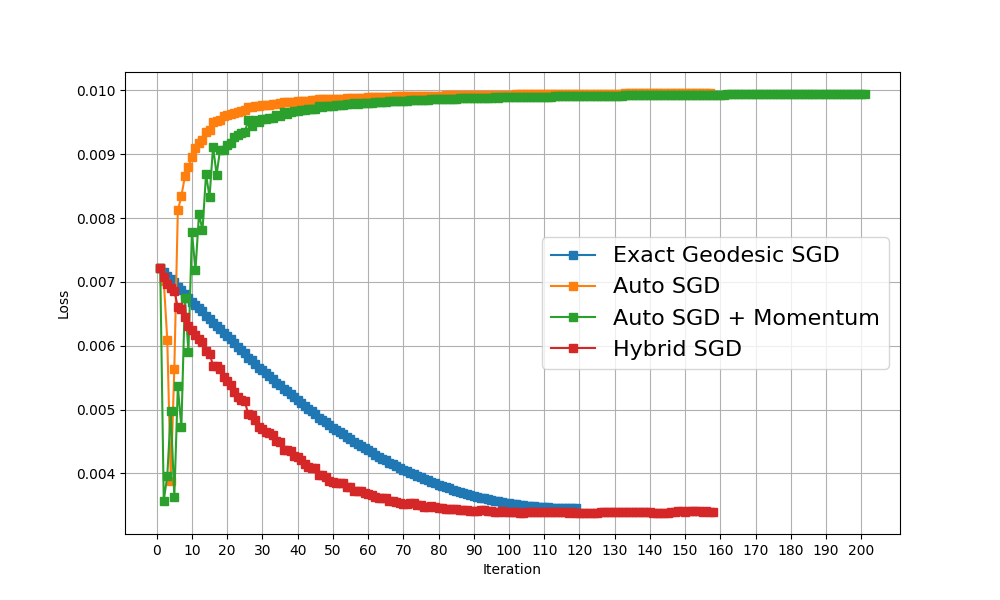}}
    \subfloat{\includegraphics[width=0.35\linewidth]{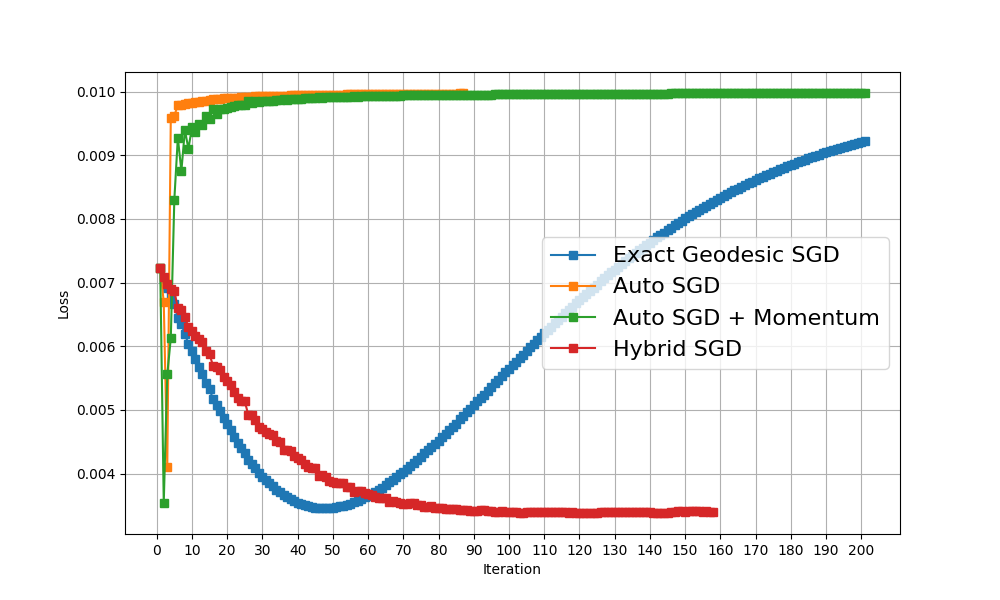}}  
    \caption{Loss on $L(\Sigma_\beta)$ through iterations with different step sizes. (Top): GD. \textcolor{red}{red}: Hybrid Gradient Descent, \textcolor{blue}{blue}: Exact Geodesic Gradient Descent, \textcolor{orange}{orange}: Auto Geodesic Gradient Descent, \textcolor{green}{green}: Auto Geodesic Gradient Descent with Momentum. (Bottom): SGD. \textcolor{red}{red}: Hybrid Stochastic Gradient Descent, \textcolor{blue}{blue}: Exact Geodesic Stochastic Gradient Descent, \textcolor{orange}{orange}: Auto Stochastic Geodesic Gradient Descent, \textcolor{green}{green}: Auto Stochastic Geodesic Gradient Descent with Momentum.}
\end{figure*}
In this experiment, we compare four optimization methods: Hybrid Gradient Descent (our), Exact Geodesic Gradient Descent (our), Gradient Descent (combined with momentum) where the Euclidean gradient $\nabla F(\Sigma)$ is auto-derived by \textbf{Pytorch} library, then the Riemannian gradient is approximated by $ 2\left(\nabla L(\Sigma) \Sigma + \left(\nabla L(\Sigma) \Sigma\right)^T \right)$ \cite{han2021riemannian}. The step sizes are chosen to be 0.1 , 0.2 and 0.5 for exact method and 1 for hybird method. For each run, we start with a fixed $\Sigma_{\beta^{(0)}} \in \mathbb{S}_{++}^d$ and use the same stopping condition. All the optimal solutions in this part are computed using the \textbf{Pytorch} library. All the experiments are conducted on a server with 2 GPU Tesla v100-sxm2-32GB RAM. For both GD and SGD with various step sizes, the objective values of our methods consistently find a descent direction and quickly converge to a solution. This behaviour is not observed when using normal Riemannian gradient descent with auto-approximated Wasserstein gradient. Additionally, we provide an ablation study of the dependence between the SUOT-based Barycenter and parameter $\tau$ in the Appendix \ref{Ablation study}.

\section{Related work} \label{compare to other work}
This section is to clarify the difference between this paper and other works which derive closed-form expressions for OT/EOT/unbalanced EOT between Gaussian measures \cite{janati2020entropic}, \cite{mallasto2022entropy}. First of all, \cite{janati2020entropic} obtained the closed-form expression for unbalanced OT solutions, but their approach needs to utilize the dual-form of the objective function. Our approach attacks directly to the objective function and obtains a shorter proof. Similarly, \cite{mallasto2022entropy} also used the dual-form of the objective function to obtain the solutions by solving a system of equation of derivatives. We acknowledge that \cite{chewi2020gradient} also worked in barycenter problem of Wasserstein distance, but their gradient descent algorithm is easy to obtain in that particular setting. To be more precisely, they only need to find the derivative of $\frac{1}{2}W_2^2(A + t C,B)$ with respect to $t$ at $t = 0$ where $A,B$ and $C$ are points on Bures manifold. This is similar to find out the derivative of $\frac{1}{2}\|x + t a\|^2$, where $x$ and $a$ are vector in the Euclidean space.  In our problem, we need to work out  the derivative of minimum of a function defined on SPD matrices in which some operations, like square root of a matrix, are  obviously not simple. Thus, we need to employ additional tools such as Lyapunov's equation.

\section{Conclusion}
\label{sec:conclusion}
In this paper, we accomplish a direction in robust barycenter computation, based on the solution of Semi-Unbalanced Optimal Transport when distributions are Gaussians with symmetric positive define covariance matrices. We state a closed-form expression for SUOT distance between two unbalanced Gaussians with different masses. Building upon this result, we propose two methods for computing robust barycenter of Gaussian measures, including Hybrid Gradient Descent and Exact Geodesic Gradient Descent Algorithm on SPD manifold. We demonstrate through both theoretical results and practice that our methods consistently converge. Our work introduces a new closed-form distance and its derivative on the SPD manifold, which has been central to several impactful applications. This contribution provides manifold-based tools for broader research, extending beyond just barycenter computation. Our work still has limitations, such as the absence of practical applications. Additionally, the experiments are conducted only on small synthetic datasets. We plan to address these issues in future work. 

\newpage
\appendix 
\begin{center}
\textbf{\Large{Supplementary Materials for
``On Barycenter Computation: Semi-Unbalanced Optimal Transport-based Method on Gaussians''}}
\end{center}

\setcounter{tocdepth}{1}  
\tableofcontents  
\addtocontents{toc}{\protect\setcounter{tocdepth}{1}}

\vspace{1 em}
\noindent
Overall, \Cref{proof for bary is gaussian} gives the proof for Theorem \ref{theo: bary is gaussian} about the Gaussian form of the SUOT-based Barycenter. \Cref{proof for closed-form} contains the proofs for key theoretical results about closed form in Theorem \ref{UOT Plan} and Theorem \ref{UOT Entropic}, alongside supporting propositions and lemmas. The proof for the closed-form of the Wasserstein gradient (Theorem \ref{UOT Derivative}) is given in \Cref{proof for uot derivative}.
\Cref{proof for convergence} presents convergence guarantees under specific conditions, as \Cref{proof for exact converges} provides a detailed derivation of convergence for the Exact Geodesic Gradient Descent algorithm and \Cref{proof for hybrid converges} discusses the Hybrid Bures-Wasserstein, which are central to the numerical methods proposed in the paper on the Bures-Wasserstein manifold. \Cref{Ablation study} includes the ablation study on the impact of the parameter $\tau$ on our barycenters, as well as comparisons with standard Wasserstein Barycenters.
\section{Proof for Theorem \ref{theo: bary is gaussian}} \label{proof for bary is gaussian}
First, we have the following Proposition. 
\begin{proposition} \label{OT Plan}
    Let $\Sigma_1$ and $\Sigma_2$ be two positive definite matrices in $\mathbb{R}^d$. Consider the problem 
    \begin{align*}
        \inf_{\mu \in  \mathcal{P}(\Sigma_1), \nu \in \mathcal{P}(\Sigma_2)} W_2^2(\mu,\nu),
    \end{align*}
    where $\mathcal{P}(\Sigma_1),\mathcal{P}(\Sigma_2)$ be the set of zero-mean probability distribution in $\mathbb{R}^d$ having covariance matrix $\Sigma_1,\Sigma_2$, respectively. Then, the result of this optimization of this problem is $\tr\left(\Sigma_1+\Sigma_2 - (\Sigma_1^{1/2}\Sigma_2\Sigma_1^{1/2})^{1/2}\right)$, and the problem admits optimal solution. 
\end{proposition}
\begin{proof}
    The Wasserstein distance can be written as
    \begin{align*}
        W_2^2(\mu,\nu) = \inf_{(X,Y)\in \Pi(\mu,\nu)}\mathbb{E}[\|X\|^2] + \mathbb{E}[\|Y\|^2] -  2\mathbb{E}[X^\top Y].
    \end{align*}
    We have $\mathbb{E}\left(\|X\|^2\right)=\tr\left(\Sigma_1\right)$ and $\mathbb{E}\left(\|Y\|^2\right)=\tr\left(\Sigma_2\right)$. To find the supermum for $\mathbb{E}[X^\top Y]$, let $U \Sigma V^{\top}$ be the SVD decomposition of $\Sigma_1^{\frac{1}{2}} \Sigma_2^{\frac{1}{2}}$. Let $\bar{X}=U^{\top} \Sigma_1^{-1 / 2} X, \bar{Y}=V^{\top} \Sigma_2^{-1 / 2} Y$, then $X=\Sigma_1^{1 / 2} U \bar{X}, Y=\Sigma_2^{1 / 2} V \bar{Y}$. The covariance matrices of $\tilde{X}$ is
    \begin{align*}
        \mathbb{E}[\tilde{X}\tilde{X}^\top] & = \mathbb{E}[U^\top\Sigma_1^{-1/2}XX^{\top}\Sigma_1^{-1/2}U] = \mathbb{E}[U^\top\Sigma_1^{-1/2}\Sigma_1^{\top}\Sigma_1^{-1/2}U] \\
        & = \mathbb{E}[U^\top U] = \Idsf,
    \end{align*}
    and similarly, $ \mathbb{E}[\tilde{Y}\tilde{Y}^\top] = \Idsf$. On the other hand, 
    \begin{equation*}
         \mathbb{E}[X^\top Y] = \mathbb{E}[\tilde{X}^{\top}U^\top\Sigma_1^{1/2}\Sigma_2^{1/2}V\tilde{Y}]  = \mathbb{E}[\tilde{X}^{\top}\Sigma \tilde{Y}] = \sum_{i=1}^d \lambda_i(\Sigma)\mathbb{E}[\tilde{X}_i\tilde{Y}_i]. \\ 
    \end{equation*} 
    The third equality arises from the fact that we can express $\tilde{X}^{\top}\Sigma \tilde{Y} = \sum_{i=1}^d \tilde{X}^{\top} \tilde{\Sigma_i} \tilde{Y}$, where $\tilde{\Sigma_i}$ is a matrix filled with zeros except for the $ii$-th element, which is $\lambda_i$. 
    By Cauchy-Schwarz inequality, we have 
    $$
    \sum_{i=1}^d \lambda_i(\Sigma)\mathbb{E}[\tilde{X}_i\tilde{Y}_i] \leq \sum_{i=1}^d \lambda_i(\Sigma) \left(\mathbb{E}[\tilde{X}_i^2]\mathbb{E}[\tilde{Y}_i^2]\right)^{\nicefrac{1}{2}} = \tr(\Sigma) = \tr\left((\Sigma_1^{1/2}\Sigma_2\Sigma_1^{1/2})^{1/2}\right).
    $$
    Since $\mathbb{E}[\tilde{X}\tilde{X}^\top] = (E[X_i X_j])_{i \times j}$ and $\mathbb{E}[\tilde{X}\tilde{X}^\top] = \Idsf$, this implies $\mathbb{E}[\tilde{X}_i^2] = 1$ for all $i$ (similarly, $\mathbb{E}[\tilde{Y}_i^2] = 1$), leading to the second equality. The final equality happens due to $\tr\left((\Sigma_1^{1/2}\Sigma_2\Sigma_1^{1/2})^{1/2}\right)= \tr\left(\left[\Sigma_1^{1/2}\Sigma_2^{1/2}(\Sigma_1^{1/2}\Sigma_2^{1/2})^{\top}\right]^{1/2}\right)$ which equals to the sum of singular values of $\Sigma_1^{1/2}\Sigma_2^{1/2}$. The equality for Cauchy-Schwartz occurs if and only if $\tilde{X}_i = \tilde{Y}_i$ (a.s), that mean $\tilde{X}= \tilde{Y}$, or $U^\top\Sigma_1^{-1/2}X = V^\top\Sigma_2^{-1/2}Y$. Particularly, in the case of Gaussian, i.e. $\mu$ follows Gaussian distribution, the value of $\nu$ that minimizes the optimization problem is also a Gaussian distribution. 
\end{proof}
Now the proof for \Cref{theo: bary is gaussian} is given as follows
\begin{proof}
    Suppose that $\tilde{\beta}$ be an optimal solution of \Cref{uot framework}, then from the form of SUOT distance, there exists probability measures $\left(\tilde{x}_i\right)_{i=1}^n$ such that 
    \begin{equation*}
        \tilde{\beta},\left(\tilde{x}_i\right)_{i=1}^n = \argmin \frac{1}{n}\sum_{i=1}^n W_2^2(\beta,x_i) + \tau \KL (x_i \| \alpha_i).
    \end{equation*}
    Let $\beta^*, \left(x^*_i\right)_{i=1}^n$ be the Gaussian probability measure having the same mean and variance as $\tilde{\beta}, \left(\tilde{x}_i\right)_{i=1}^n$, respectively. Then, using Lemma 3.3 \cite{le2022entropic} about KL divergence minimum and \Cref{OT Plan}, we achieve 
    \begin{equation*}
        W_2^2(\beta^*, x_i^*) \leq  W_2^2(\tilde{\beta},\tilde{x}_i), \quad \text{and} \quad  
        \KL (x_i^* \| \alpha_i) \leq \KL ( \tilde{x}_i\| \alpha_i).
    \end{equation*}
    The equalities hold if and only if $\tilde{x}_i$ is Gaussian measure, which implies that  $\tilde{\beta}$ is also a Gaussian measure. Thus, we can conclude that the optimal solution of \Cref{uot framework} is Gaussian measure. 
\end{proof}

\section{Proofs for Theorem \ref{UOT Plan} and Theorem \ref{UOT Entropic}} \label{proof for closed-form}
\subsection{Proof for Theorem \ref{UOT Plan}} \label{proof for uot_distance}
Before giving full proof for Theorem \ref{UOT Plan}, we state some necessary lemmas:
\begin{lemma} \label{KL between Gaussians}
    Let $\alpha=m_\alpha \mathcal{N}\left(\mathbf{0}, \Sigma_\alpha\right)$ and $\beta=m_\beta \mathcal{N}\left(\mathbf{0}, \Sigma_\beta\right)$ be scaled Gaussian measures while $\bar{\alpha}=\mathcal{N}\left(0, \Sigma_\alpha\right)$ and $\bar{\beta}=\mathcal{N}\left(0, \Sigma_\beta\right)$ be their normalized versions. Then, the generalized KL divergence between $\alpha$ and $\beta$ is decomposed as
    \begin{align*}
        \mathrm{KL}(\alpha \| \beta)=m_\alpha \mathrm{KL}(\bar{\alpha} \| \bar{\beta})+\mathrm{KL}\left(m_\alpha \| m_\beta\right).
    \end{align*}
\end{lemma}

\begin{proof}
Let $f\left(x,\Sigma_{\alpha} \right)$ and $f\left(x, \Sigma_{\beta} \right)$ be the distribution functions of $\mathcal{N} \left(\mathbf{0}, \Sigma_{\alpha}\right)$ and $\mathcal{N}\left(\mathbf{0}, \Sigma_\beta\right)$, respectively. We have
\begin{align*}
\mathrm{KL}(\alpha \| \beta) & =\int_{\mathbb{R}^d} \log \left(\frac{m_\alpha f\left(x, \Sigma_\alpha\right)}{m_\beta f\left(x, \Sigma_\beta\right)}\right) d \alpha(x)-m_\alpha+m_\beta \\
& =\int_{\mathbb{R}^d} \log \left(\frac{f\left(x, \Sigma_\alpha\right)}{f\left(x, \Sigma_\beta\right)}\right) m_\alpha d \bar{\alpha}(x)+\log \left(\frac{m_\alpha}{m_\beta}\right) m_\alpha-m_\alpha+m_\beta \\
& =m_\alpha \mathrm{KL}(\bar{\alpha} \| \bar{\beta})+\mathrm{KL}\left(m_\alpha \| m_\beta\right) .
\end{align*}
Moreover, the expression for $\mathrm{KL}(\bar{\alpha} \| \bar{\beta})$ is given by
\begin{align*}
    \frac{1}{2}\left\{\tr\left(\Sigma_\alpha \Sigma_\beta^{-1}\right)-d+\log \left(\frac{\operatorname{det}\left(\Sigma_\beta\right)}{\operatorname{det}\left(\Sigma_\alpha\right)}\right)\right\} .
\end{align*}
\end{proof}

\begin{lemma} \label{Optimizer}
    For positive constants $a, \tau$ and $\Upsilon$, the function
    \begin{align*}
        f(x)=\Upsilon x+\tau \mathrm{KL}(x \| a)
    \end{align*}
    attains its minimum at $x^*=a \exp \left\{\frac{-\Upsilon}{ \tau}\right\}$.
\end{lemma}
\begin{proof}
    By taking the derivative of $f(x)$, we have
    \begin{align*}
        f^{\prime}(x)=\Upsilon+\tau\{\log (x)-\log (a)\}.
    \end{align*}
Solving the equation $f^{\prime}(x)=0$, we obtain
\begin{align*}
    x^*=\exp \left\{\frac{\tau \log (a)-\Upsilon}{\tau}\right\}=a \exp \left\{\frac{-\Upsilon}{ \tau}\right\}.
\end{align*}
\end{proof}

Now we are ready to give the proof for Theorem \ref{UOT Plan}.

\begin{proof}[Proof of Theorem \ref{UOT Plan}]    
 Recall that $\pi$ is a positive measure such that $\pi=m_\pi \overline{\pi}$ with $\overline{\pi}$ is a probability measure with mean $\left(\mathbf{a}_x, \mathbf{b}\right)$ and covariance matrix
\begin{align*}
\Sigma_\pi=\left(\begin{array}{cc}
\Sigma_x &K_{x \beta} \\
K_{x \beta}^{\top} & \Sigma_y
\end{array}\right).
\end{align*}
Here the two marginals of $\piol$ are denoted by $\piol_x$ and $\piol_y$ where $\piol_x$ has mean $\mathbf{a}_x$ and covariance matrix $\Sigma_x$ while $\piol_y$ has mean $\mathbf{b}$ and covariance matrix $\Sigma_y = \Sigma_\beta$. Let $\pi_x$ and $\pi_y$ be two marginals of $\pi$, then $\pi_x=m_\pi \piol_x$ and $\pi_y=m_\pi \piol_y$.
Note that by our \Cref{OT Plan} and Lemma 3.3 in \cite{le2022entropic}, $\pi$ needs to be Gaussian. We also have
\begin{align*}
\mathbb{E}_\pi\|X-Y\|_2^2=m_\pi\left\{\left\|\mathbf{a}_x-\bbf\right\|_2^2+\tr\left(\Sigma_x\right)+\tr\left(\Sigma_\beta\right)-2 \tr\left(K_{x \beta}\right)\right\} .
\end{align*}
According to Lemma \ref{KL between Gaussians},
\begin{align*}
\KL\left(\pi_x \| \alpha\right) & =m_\pi \KL\left(\piol_x \| \alphaol\right)+\KL\left(m_\pi \| m_\alpha\right).
\end{align*}
Combining the above results, the objective function between $\alpha$ and $\beta$ reads as
\begin{align*}
W^2_{2_{\operatorname{SUOT}}}(\alpha, \beta;\tau) :=\min _{\pi \in \Mcal^+\left(\mathbb{R}^d \times \mathbb{R}^d\right)}  & m_\pi\left\{\left\|\abf_x-\bbf\right\|_2^2+\tr\left(\Sigma_x\right)+\tr\left(\Sigma_\beta\right)-2 \tr\left(K_{x \beta}\right) \right\} \\
& + \tau \mathrm{KL}\left(\piol_x \| \alphaol\right) +\tau \KL\left(m_\pi \| m_\alpha\right).
\end{align*}
Denote 
\begin{align*}\Upsilon = \left\|\mathbf{a}_x-\bbf\right\|_2^2+\tr\left(\Sigma_x\right)+\tr\left(\Sigma_\beta\right)-2 \tr\left(K_{x \beta}\right) + \tau \KL\left(\piol_x \| \alphaol\right).
\end{align*}
Due to the independence of $\Upsilon$ with $m_\pi$, we could minimize $\Upsilon$ then find $m_\pi$. \\
\textbf{Minimization of $\Upsilon$:}
 For the KL term, due to Lemma \ref{KL between Gaussians} we have
\begin{align*}
\tau \mathrm{KL}\left(\piol_x \| \alphaol\right) = \frac{\tau}{2}\left[\tr\left(\Sigma_\alpha^{-1} \Sigma_x\right)-d+\log \left(\frac{\det\left(\Sigma_\alpha\right)}{\det\left(\Sigma_x\right)}\right)\right] + \frac{\tau}{2} \left(\abf_x- \abf \right)^{\top} \Sigma_\alpha^{-1}\left( \abf_x - \abf \right).
\end{align*}
Now $\Upsilon$ reads
\begin{align*}
\Upsilon &= \tr\left(\Sigma_x\right)-2 \tr\left(K_{x \beta}\right) + \frac{\tau}{2}\left[\tr\left(\Sigma_\alpha^{-1} \Sigma_x\right)+\log \left(\frac{\det\left(\Sigma_\alpha\right)}{\det\left(\Sigma_x\right)}\right)\right] + \left\|\abf_x-\bbf\right\|_2^2 + \\
&\qquad \frac{\tau}{2} \left(\abf_x- \abf \right)^{\top} \Sigma_\alpha^{-1}\left( \abf_x - \abf \right) +\tr\left(\Sigma_\beta\right) - \frac{\tau d}{2}.
\end{align*}
\textit{\textbf{Means part:}} We first work with the terms involving $\mathbf{a}_x$. Let $\mathbf{a}_x - \mathbf{a} = \widetilde{\abf}_x$, then
 \begin{align*}
\left\|\mathbf{a}_x-\mathbf{b}\right\|_2^2 & =\left\|\widetilde{\mathbf{a}}_x+\mathbf{a}-\mathbf{b}\right\|_2^2 \\ 
& =\left\|\widetilde{\mathbf{a}}_x\right\|_2^2 + 2 \widetilde{\mathbf{a}}_x^{\top}(\mathbf{a}-\mathbf{b})+\|\mathbf{a}-\mathbf{b}\|^2.
\end{align*}
Hence, sum of all terms which include $\widetilde{\mathbf{a}}_x$ is equal to
\begin{align*}
    \Upsilon_{\mathbf{a},\mathbf{b}} = \left\|\widetilde{\mathbf{a}}_x\right\|_2^2 + 2 \widetilde{\mathbf{a}}_x^{\top}(\mathbf{a}-\mathbf{b})+\|\mathbf{a}-\mathbf{b}\|^2 + \frac{\tau}{2}\widetilde{\mathbf{a}}_x^{\top} \Sigma_\alpha^{-1} \widetilde{\mathbf{a}}_x.
\end{align*}
Taking derivative according to $\widetilde{\mathbf{a}}_x$ and set equation to $\mathbf{0}$, we have
\begin{align*} 
2 \widetilde{\mathbf{a}}_x + 2(\mathbf{a} - \mathbf{b}) + \tau  \Sigma_\alpha^{-1}\widetilde{\abf}_x = 0,
\end{align*}
which turns into
\begin{align*} \left( \Idsf + \frac{\tau}{2} \Sigma_\alpha^{-1} \right) \widetilde{\mathbf{a}}_x = \mathbf{b}-\mathbf{a} .
\end{align*}
Denote $\Sigma_{\alpha, \tau} = \Idsf + \frac{\tau}{2} \Sigma_{\alpha}^{-1}$, we obtain
\begin{align*}
    \widetilde{\mathbf{a}}_x = \Sigma_{\alpha, \tau}^{-1} (\mathbf{b}-\mathbf{a}),
\end{align*}
then $\mathbf{a}_x = \Sigma_{\alpha, \tau}^{-1} (\mathbf{b}-\mathbf{a}) + \mathbf{a}$. Pull it back to $\Upsilon_{\mathbf{a,b}}$ gives
\begin{align*}
\Upsilon_{\mathbf{a^*,b^*}} = (\mathbf{a} - \mathbf{b})^{\top}\left\{ \left( \left\|\Sigma_{\alpha, \tau}^{-1}\right\|_F^2 + 1 \right ) \Idsf -2 \Sigma_{\alpha, \tau}^{-1}+\frac{\tau}{2} \Sigma_{\alpha, \tau}^{-1} \Sigma_\alpha^{-1} \Sigma_{\alpha, \tau}^{-1}\right\}(\mathbf{a}-\mathbf{b}).
\end{align*}
\textit{\textbf{Covariance matrix part: }}
The second part is to group all factors of $\Sigma_x$ and $K_{x \beta} $ , which is
\begin{align*}
    \Upsilon_\Sigma & = \tr\left(\Sigma_x\right)-2 \tr\left(K_{x \beta}\right) + \frac{\tau}{2}\left[\tr\left(\Sigma_\alpha^{-1} \Sigma_x\right)-\log \left(\det\big(\Sigma_x\big)\right)\right] \\
    & = \tr \left( \Sigma_x \left (\Idsf + \frac{\tau}{2} \Sigma_\alpha^{-1} \right) \right) - 2 \tr (K_{x\beta}) - \frac{\tau}{2}\log \left(\detsf\left(\Sigma_x\right)\right) \\
    & = \tr \left( \Sigma_x \Sigma_{\alpha, \tau} \right) - 2 \tr (K_{x\beta}) -  \frac{\tau}{2}\log \left(\detsf\left(\Sigma_x\right)\right).
\end{align*}
First, we embark on the task of maximizing $\tr(K_{x\beta})$ while fixing $\Sigma_x$. Note that matrices $\Sigma_x, K_{x \beta}$ must satisfy condition that covariance block matrix $\left(\begin{array}{cc}
\Sigma_x & K_{x\beta} \\
K_{x\beta}^{\top} & \Sigma_\beta
\end{array}\right)$ is SPD. This can be equivalently expressed as
\begin{align*}
& \Sigma_x -K_{x \beta} \Sigma_\beta^{-1} K_{x \beta}^{\top} \succeq 0 \\
\Leftrightarrow \quad & \Idsf -\left(\Sigma_x^{-\frac{1}{2}} K_{x \beta} \Sigma_\beta^{-\frac{1}{2}}\right)\left(\Sigma_\beta^{-\frac{1}{2}} K_{x \beta}^T \Sigma_x^{-\frac{1}{2}}\right) \succeq 0.
\end{align*}
By denoting $\Sigma_x^{-\frac{1}{2}} K_{x \beta} \Sigma_\beta^{-\frac{1}{2}}$ as $H$, this condition can be reformulated as
\begin{align*}\Idsf - H H^{\top} & \succeq 0.
\end{align*}
Noting that $K_{x \beta}=\Sigma_x^{\frac{1}{2}} H \Sigma_\beta^{\frac{1}{2}}$, it follows that $\tr\left(K_{x \beta}\right)=\tr\left(H \Sigma_\beta^{\frac{1}{2}} \Sigma_x^{\frac{1}{2}}\right)$. Let $\left(\lambda_i(\Sigma_\beta^{\frac{1}{2}} \Sigma_x^{\frac{1}{2}})\right)_{i=1}^d$ and $\left(\lambda_i(H)\right)_{i=1}^d$ represent the singular values of $\Sigma_{\beta}^{\frac{1}{2}} \Sigma_{x}^{\frac{1}{2}}$ and $H$ in descending order, respectively. Applying von Neuman inequality, we get
\begin{align*}
    \tr(K_{x \beta}) \leq \sum_{i=1}^d \lambda_i\big(\Sigma_\beta^{\frac{1}{2}}\Sigma_x^{\frac{1}{2}}\big) \lambda_i(H) \leq \sum_{i=1}^d \lambda_i\big(\Sigma_\beta^{\frac{1}{2}}\Sigma_x^{\frac{1}{2}} \big).
\end{align*}
The second inequality is a consequence of the Lemma 3.4 in \cite{le2022entropic}, which states that all singular values of $\Sigma_x^{-\frac{1}{2}} K_{x \beta} \Sigma_\beta^{-\frac{1}{2}}$ lie between $0$ and $1$. Additionally, if we consider the singular value decomposition $U\Lambda V$ of $\Sigma_\beta^{\frac{1}{2}} \Sigma_x^{\frac{1}{2}}$, we can choose $H = V^{\top} U^{\top}$, satisfying the condition for $H$ mentioned earlier and 
\begin{align*}
\tr\left(H \Sigma_\beta^{\frac{1}{2}} \Sigma_x^{\frac{1}{2} }\right)=\tr(\Lambda)=\sum_i^d \lambda_i\big(\Sigma_\beta^{\frac{1}{2}}\Sigma_x^{\frac{1}{2}} \big).
\end{align*}
It confirms that when $\Upsilon_\Sigma$ achieves its maximum value, $\tr(K_{x \beta})$ corresponds to the sum of all singular values of $\Sigma_\beta^{\frac{1}{2}} \Sigma_x^{\frac{1}{2}}$. In such a scenario, we obtain
\begin{align*}
\Upsilon_{\Sigma} = \tr\left(\Sigma_x \Sigma_{\alpha, \tau} \right)-2 \tr\Big( \big[ \Sigma_\beta^{\frac{1}{2}} \Sigma_x \Sigma_\beta^{\frac{1}{2}}\big]^\frac{1}{2} \Big) - \frac{\tau}{2}\log \left(\det\left(\Sigma_x\right)\right).
\end{align*}
Let us define $\widetilde{\Sigma}=\left(\Sigma_{\beta}^{\frac{1}{2}} \Sigma_x \Sigma_{\beta}^{\frac{1}{2}}\right)^{\frac{1}{2}}$, which can be expressed as $\Sigma_x=\Sigma_\beta^{-\frac{1}{2}}\widetilde{\Sigma}^2\Sigma_\beta^{-\frac{1}{2}}$. Consequently, we obtain
\begin{align*}
\tr(\Sigma_x \Sigma_{\alpha, \tau})  = \tr\left( \widetilde{\Sigma}^2 \Sigma_{\beta}^{-\frac{1}{2}} \Sigma_{\alpha, \tau} \Sigma_\beta^{-\frac{1}{2}}\right)  = \tr\big(\Sigmawtd^2 \Sigma_{\alpha,\tau,\beta} \big).
\end{align*}
Additionally, we have
\begin{align*}
    \det(\Sigma_x) = \frac{\det(\widetilde{\Sigma})^2}{\det(\Sigma_{\beta})}.
\end{align*}
Now $\Upsilon_\Sigma$ turns into 
\begin{align*}
\Upsilon_\Sigma=\tr\left(\Sigmawtd^2 \Sigma_{\alpha,\tau,\beta} \right)-2 \tr(\Sigmawtd)-\frac{\tau}{2} \log\big( \frac{\det(\Sigmawtd)^2}{\det(\Sigma_\beta)} \big).
\end{align*}
Note that $\widetilde{\Sigma}$ is symmetric, so there exists a diagonal matrix $\Lambda$ that $\widetilde{\Sigma}$ is similar to $\Lambda$ (which implies $\Sigmawtd^2$ is similar to $\Lambda^2$. Assume that $\Lambda = \operatorname{diag}(\lambda_i(\Sigmawtd))_{i = 1}^d$ which is decreasingly ordered and $(\lambda_{i}(\Sigma_{\alpha,\tau,\beta}))_{i = 1}^d$ are eigenvalues of $\Sigma_{\alpha, \beta, \tau}$ in ascending order, by Ruhe’s trace inequality
\begin{align*}
\tr\left(\Sigmawtd^2 \Sigma_{\alpha,\tau,\beta} \right)
& \geq \ \sum_{i=1}^d \lambda_i(\Sigmawtd)^2 \lambda_{d-i+1}(\Sigma_{\alpha,\tau,\beta}),
\end{align*}
where the equality holds when $\Sigmawtd^2$ and $\Sigma_{\alpha, \beta, \tau}$ are commuting. The optimization part now is calculated on eigenvalues of $\Lambda$, because
\begin{align*}
\Upsilon_\Sigma & \geq \sum_{i=1}^d \lambda_i^2(\Sigmawtd) \lambda_{n-i+1}(\Sigma_{\alpha,\tau,\beta}) -2 \sum_{r=1}^d \lambda_i(\Sigmawtd)-\tau \log \left(\prod_{i=1}^d \lambda_i(\Sigmawtd)\right) \\
& =\sum_{i=1}^d \Big(\lambda_{n-i+1}(\Sigma_{\alpha,\tau,\beta}) \lambda_i^2(\Sigmawtd)-2 \lambda_i(\Sigmawtd)-\tau \log \big(\lambda_i(\Sigmawtd)\big)\Big).
\end{align*}
Consider the function 
\begin{align*}
    f(v)=u v^2-2 v-\tau \log (v).
\end{align*}
Take the derivative and set it to $0$, we get
\begin{align*}
\frac{2 u v^2-2v-\tau}{v}=0,
\end{align*}
which has unique positive solution $v^* = \frac{1+\sqrt{1+2u \tau}}{2 u}$. 
It verifies that, to attain the minimization of $\Upsilon_\Sigma$,
\begin{align*}\lambda_i(\Sigmawtd) = \frac{1+\sqrt{1+2\tau \lambda_{n-i+1}(\Sigma_{\alpha,\tau,\beta})}}{2 \lambda_{n-i+1}(\Sigma_{\alpha,\tau,\beta})}.
\end{align*}
This results in the following expression
\begin{align*}
\Sigmawtd \quad \text{is similar to}\quad  \Lambdaol:= \operatorname{diag} \left( \frac{1+\sqrt{1+2\tau \lambda_{n-i+1}(\Sigma_{\alpha,\tau,\beta})}}{2 \lambda_{n-i+1}(\Sigma_{\alpha,\tau,\beta})} \right).
\end{align*}
Since $\Sigmawtd$ and $\Sigma_{\alpha,\tau,\beta}$ are commuting, the eigenvalue of $\Sigmawtd$ can be computed from the eigenvalues of $\Sigma_{\alpha,\tau,\beta}$, we get
\begin{align*}
\Sigmawtd^* = \frac{1}{2}\Sigma_{\alpha,\tau,\beta}^{-1} + \frac{1}{2} \Big[ \Sigma_{\alpha,\tau,\beta}^{-2} + 2\tau \Sigma_{\alpha,\tau,\beta}^{-1}\Big]^{\frac{1}{2}}.
\end{align*}
The equation $2u (v^*)^2 - 2v^* - \tau = 0$ deduces that $(v^*)^2 = \frac{v^*}{u} + \frac{\tau}{2u}$. Hence, we yield
\begin{align*}
    \big[\Sigmawtd^*\big]^2 = \frac{\tau}{2}\Sigma_{\alpha,\tau,\beta}^{-1} + \frac{1}{2}\Sigma_{\alpha,\tau,\beta}^{-2}\Big[\Idsf + \big(\Idsf + 2\tau \Sigma_{\alpha,\tau,\beta} \big)^{\frac{1}{2}} \Big].
\end{align*}
That leads to the formula for $\Sigma_x$
\begin{align*}
    \Sigma_{x} = \Sigma_\beta^{-\frac{1}{2}} \bigg[ \frac{\tau}{2}\Sigma_{\alpha,\tau,\beta}^{-1} + \frac{1}{2}\Sigma_{\alpha,\tau,\beta}^{-2}\Big[\Idsf + \big(\Idsf + 2\tau \Sigma_{\alpha,\tau,\beta} \big)^{\frac{1}{2}} \Big] \bigg] \Sigma_\beta^{-\frac{1}{2}}.
\end{align*}
For the shake of simplicity, we denote
\begin{align*}
   \Sigma_\gamma=\frac{\tau}{2} \Idsf+\frac{1}{2} \Sigma_{\alpha, \tau, \beta}^{-1}\Big[\Idsf+\left(\Idsf+2 \tau \Sigma_{\alpha, \tau, \beta}\right)^{\frac{1}{2}}\Big],
\end{align*}
then 
\begin{align*}
    \Sigma_{x} = \Sigma_\beta^{-\frac{1}{2}} \Sigma_{\alpha, \tau, \beta}^{-1} \Sigma_\gamma \Sigma_\beta^{-\frac{1}{2}}.
\end{align*}
At this optimum, the function $\Upsilon_\Sigma$ takes the value
\begin{align*}
    \Upsilon_{\Sigma}=\tr(\Sigma_\gamma) - \tr\left(\Big[\Sigma_{\alpha, \tau, \beta}^{-1} \Sigma_\gamma\Big]^{\frac{1}{2}}\right) - \frac{\tau}{2} \log \left(\det\Big[\Sigma_\gamma \Sigma_{\alpha, \tau, \beta}^{-1} \Sigma_\beta^{-1}\Big]\right),
\end{align*}
and 
\begin{align*}
    \Upsilon = & \quad \tr(\Sigma_\gamma) +  \tr(\Sigma_\beta) - \tr\left(\Big[\Sigma_{\alpha, \tau, \beta}^{-1} \Sigma_\gamma\Big]^{\frac{1}{2}}\right) - \frac{\tau}{2} \log \left(\det\Big[\Sigma_\gamma \Sigma_{\alpha, \tau, \beta}^{-1} \Sigma_\beta^{-1} \Sigma_\alpha^{-1}\Big]\right) \\
    & + (\mathbf{a} - \mathbf{b})^{\top}\left\{ \left( \left\|\Sigma_{\alpha, \tau}^{-1}\right\|_F^2 + 1 \right ) \Idsf -2 \Sigma_{\alpha, \tau}^{-1}+ \frac{\tau}{2} \Sigma_{\alpha, \tau}^{-1} \Sigma_\alpha^{-1} \Sigma_{\alpha, \tau}^{-1}\right\}(\mathbf{a}-\mathbf{b}) - \frac{\tau d}{2}.
\end{align*}
\textbf{Calculation of $m_{\pi}$}: 
Recall that in order to find $m_{\pi}$, we minimize
\begin{align*}
 m_\pi \Upsilon +\tau \KL\left(m_\pi \| m_\alpha\right).
\end{align*}
Considering the mentioned $\Upsilon$. As per Lemma \ref{Optimizer}, we obtain the value of the optimizer $m_{\pi}$ as
\begin{align*}
    m_{\pi}=m_\alpha \exp \left\{\frac{-\Upsilon}{\tau}\right\}
\end{align*}
and final objective function value
\begin{align*}
     W^2_{2_{\operatorname{SUOT}}}(\alpha, \beta, \tau) = \tau m_\alpha \left( 1 - \exp \left\{\frac{-\Upsilon}{\tau}\right\} \right).
\end{align*}
Hence, we have thus proved our claims. In the preceding proof, the optimization of the trace operator using $K_{x\beta}$ is not constrained to a single choice. An alternative option involves utilizing $\Sigma_x^{\frac{1}{2}} \Sigma_\beta^{\frac{1}{2}}$ or the geometric mean of $\Sigma_x$ and $\Sigma_\beta$. Additional details are in Lemma \ref{lemma:eqtrace}. 

\begin{lemma} \label{lemma:eqtrace}
    Given SPD matrices $A,B \in \mathbb{S}^{d}_{++}$. Moreover, $A$ and $B$ have the same unitary matric in SVD decomposition i.e. $A = S \Lambda_1^2 S^{\top}, B = S \Lambda_2^2 S^{\top}$. Then we have
    \begin{align*}
        \tr( \big[A^{\frac{1}{2}} B A^{\frac{1}{2}} \big]^{\frac{1}{2}}) = \tr (B^{\frac{1}{2}} A^{\frac{1}{2}}). 
    \end{align*}
\end{lemma}

\begin{proof}
    We note that
\begin{align*}
    \tr\left(\left[A^{\frac{1}{2}} B A^{\frac{1}{2}}\right]^{\frac{1}{2}}\right)=\tr\left([A B]^{\frac{1}{2}}\right).
\end{align*}
It follows that we need to prove
\begin{align*}
    \tr \left(A^{\frac{1}{2}} B^{\frac{1}{2}}\right)=\tr\left([A B]^{\frac{1}{2}}\right).
\end{align*}
We have
\begin{align*}
& A^{\frac{1}{2}}=S \Lambda_1 S^{\top}, \quad B^{\frac{1}{2}}=S \Lambda_2 S^T.
\end{align*}
Thus, we get
\begin{align*}
\tr\left(A^{\frac{1}{2}} B^{\frac{1}{2}}\right)=\tr\left(\Lambda_1 \Lambda_2\right).
\end{align*}
Furthermore, we have $A B$ and $\Lambda_1^2 \Lambda_2^2$ have the same set of eigenvalues. It gives 
\begin{align*}
\tr\left(\big[A B \big]^{\frac{1}{2}}\right) &= \sum_{i=1}^d \lambda_i \left( [AB]^{\frac{1}{2}} \right)  = \sum_{i=1}^d \lambda_i \left( [\Lambda_1^2 \Lambda_2^2]^{\frac{1}{2}} \right) = \tr\left(\Lambda_1 \Lambda_2\right).
\end{align*}
Then we deduce that
\begin{align*}
    \tr\left(\left[A^{\frac{1}{2}} B A^{\frac{1}{2}}\right]^{\frac{1}{2}}\right)=\tr\left(A^{\frac{1}{2}} B^{\frac{1}{2}}\right).
\end{align*}
\end{proof}
Now to see the final expression, from \cite{altschuler2021averaging}, we have Wasserstein distance between Gaussians $\alpha = \mathcal{N}(\mathbf{a}, \Sigma_\alpha), \beta = \mathcal{N}(\mathbf{b}, \Sigma_\beta)$ as
\begin{align*}
    W_2^2\left(\alpha, \beta \right) = \|\mathbf{a}-\mathbf{b}\|^2 + \text{tr}\left(\Sigma_\alpha+\Sigma_\beta - 2\left[\Sigma_\alpha^{\frac{1}{2}} \Sigma_\beta \Sigma_\alpha^{\frac{1}{2}}\right]^{\frac{1}{2}} \right).
\end{align*}
Next, with the assumption that $m_{\alpha} = m_{\beta}$, the SUOT distance reads
\begin{align*}
    W^2_{2_{\operatorname{SUOT}}}(\alpha, \beta; \tau) := \min \|\mathbf{a_x} - \mathbf{b}\|_2^2 + \text{tr}(\Sigma_x) + \text{tr}(\Sigma_\beta) - 2 \text{tr}(K_{x \beta}) + \tau \text{KL}(\pi_x \| \alpha).
\end{align*}
From \Cref{lemma:eqtrace}, at the optimal solution, the term $\text{tr}\left(K_{x \beta}\right)$ takes the value
$\text{tr}\Big( \big[ \Sigma_\beta^{\frac{1}{2}} \Sigma_x \Sigma_\beta^{\frac{1}{2}}\big]^\frac{1}{2} \Big)$. This leads to
\begin{align*}
    W^2_{2_{\operatorname{SUOT}}}(\alpha, \beta, \tau) = W_2^2(\pi_x, \beta) + \tau \mathrm{KL}(\pi_x \| \Sigma_\alpha).
\end{align*}
As a consequence, we obtain the full conclusion of the theorem.
\end{proof}

Upon to this formula, we have the results that our UOT distance is bounded by Wasserstein distance.
\begin{proposition} \label{lem: bound by wasserstein}
    Giving two centered Gaussians $\alpha, \beta$. There exists a positive constant $\xi_\tau$ that we have
    \begin{align*}
         \xi_\tau W_2^2(\Sigma_\alpha, \Sigma_\beta) \leq W^2_{2_{\operatorname{SUOT}}}(\Sigma_\alpha, \Sigma_\beta) \leq W_2^2(\Sigma_\alpha, \Sigma_\beta).
    \end{align*}
\end{proposition}
\begin{proof}
    The second inequality is straight forward while the first inequality happens due to Talagrand inequality \cite{otto2000generalization}, we have $\mathrm{KL}(\Sigma_x \| \Sigma_\alpha) \geq \frac{\xi}{2} W_2^2(\Sigma_x, \Sigma_\alpha)$ for a constant $\xi$. It leads to 
\begin{align*}
    & W^2_{2_{\operatorname{SUOT}}}(\Sigma_\alpha, \Sigma_\beta) \\
    & = W_2^2(\Sigma_x, \Sigma_\beta) + \tau \mathrm{KL}(\Sigma_x \| \Sigma_\alpha) \\
    & \geq W_2^2(\Sigma_x, \Sigma_\beta) + \frac{\xi \tau}{2} W_2^2(\Sigma_x, \Sigma_\alpha) \\
    & \geq \min \{1, \frac{\xi \tau}{2} \} \left(W_2^2(\Sigma_x, \Sigma_\beta) + W_2^2(\Sigma_x, \Sigma_\alpha) \right) \\
    & \geq \frac{\min \{1, \frac{\xi \tau}{2} \}}{2} \big(W_2\left(\Sigma_x, \Sigma_\beta \right) + W_2\left(\Sigma_x, \Sigma_\alpha \right) \big)^2 \\
    & \geq \underbrace{\frac{\min \{1, \frac{\xi \tau}{2} \}}{2}}_{\xi_\tau} W_2^2(\Sigma_\alpha, \Sigma_\beta).
\end{align*}
As a consequence, we obtain the conclusion of the proposition. 
\end{proof}

\subsection{Proof for Theorem \ref{UOT Entropic}}
\begin{proof}
    We follow the proof of Theorem (\ref{UOT Plan}) to obtain the explicit form of the minimizer and objective function. Here the two marginals of $\piol$ are denoted by $\piol_x$ and $\piol_y$ where $\piol_x$ has mean $\textbf{0}$ and covariance matrix $\Sigma_x$ while $\piol_y$ has mean $\textbf{0}$ and covariance matrix $\Sigma_y = \Sigma_\beta$. Let $\pi_x$ and $\pi_y$ be two marginals of $\pi$, then $\pi_x=m_\pi \piol_x$ and $\pi_y=m_\pi \piol_y$. We also have
\begin{align*}
\mathbb{E}_\pi\|X-Y\|_2^2=m_\pi\Big\{\tr\left(\Sigma_x\right)+\tr\left(\Sigma_\beta\right)-2 \tr\left(K_{x \beta}\right)\Big\} .
\end{align*}
According to Lemma \ref{KL between Gaussians},
\begin{align*}
\KL\left(\pi_x \| \alpha\right) & =m_\pi \KL\left(\piol_x \| \alphaol\right)+\KL\left(m_\pi \| m_\alpha\right) \\
\KL\left(\pi_x \| \alpha \otimes \beta \right) & =m_\pi \KL\left(\piol_x \| \alphaol \otimes \bar{\beta} \right)+\KL\left(m_\pi \| m_\alpha m_\beta\right).
\end{align*}
Combining the above results, the entropic objective function between $\alpha$ and $\beta$ reads as
\begin{align*}
W^2_{2_{\operatorname{SUOT}}, \delta}(\alpha, \beta;\tau):=\min _{\pi \in \Mcal^+\left(\mathbb{R}^d \times \mathbb{R}^d\right)}  & m_\pi\left\{\tr\left(\Sigma_x\right)+\tr\left(\Sigma_\beta\right)-2 \tr\left(K_{x \beta}\right) + \tau \mathrm{KL}\left(\piol_x \| \alphaol\right) + \delta \mathrm{KL}\left(\piol \| \alphaol \otimes \betaol \right) \right\} \\
& +\tau \KL\left(m_\pi \| m_\alpha\right) + \delta \KL\left(m_\pi \| m_\alpha m_\beta\right).
\end{align*}
Denote 
\begin{align*}\Upsilon = \tr\left(\Sigma_x\right)+\tr\left(\Sigma_\beta\right)-2 \tr\left(K_{x \beta}\right) + \tau \KL\left(\piol_x \| \alphaol\right) + \delta \mathrm{KL}\left(\piol \| \alphaol \otimes \betaol \right).
\end{align*}
Due to the independence of $\Upsilon$ with $m_\pi$, we could minimize $\Upsilon$ then find $m_\pi$. For a fixed $\Upsilon$, take derivative according to $m_\pi$ of \\
\begin{align*}
    m_\pi \Upsilon +\tau \KL\left(m_\pi \| m_\alpha\right) + \delta \KL\left(m_\pi \| m_\alpha m_\beta\right).
\end{align*}
and set equation to $0$, we obtain the value of the optimizer $m_{\pi}$ as
\begin{align*}
    m_{\pi} = m_\alpha m_\beta^{\frac{\delta}{\tau + \delta}} \exp \left(\frac{-\Upsilon}{\tau + \delta} \right).
\end{align*}
\textbf{Minimization of $\Upsilon$:}
 For the KL term, due to Lemma \ref{KL between Gaussians} we have
\begin{align*}
\mathrm{KL}\left(\piol_x \| \alphaol\right) & = \frac{1}{2}\left[\tr\left(\Sigma_\alpha^{-1} \Sigma_x\right)-d+\log \left(\frac{\det\left(\Sigma_\alpha\right)}{\det\left(\Sigma_x\right)}\right)\right] \\
\mathrm{KL}\left(\piol \| \alphaol \otimes \betaol \right) & = \frac{1}{2}\left[\tr\left(\Sigma_\alpha^{-1} \Sigma_x\right)-d+\log \left(\frac{\det\left(\Sigma_\alpha\right)}{\det\left(\Sigma_x\right)}\right)\right] - \frac{1}{2} \sum_{i = 1}^d \log \left( 1- \lambda_i(H)  \right),
\end{align*}
where $\left(\lambda_i(H)\right)$ is the $i$-th largest singular value of $H := \Sigma_x^{-\frac{1}{2}} K_{x \beta} \Sigma_\beta^{-\frac{1}{2}}$ for all $i \in[d]$.
Now $\Upsilon$ reads
\begin{align*}
\Upsilon &= \tr\left(\Sigma_x\right)-2 \tr\left(K_{x \beta}\right) + \frac{(\tau + \delta)}{2}\left[\tr\left(\Sigma_\alpha^{-1} \Sigma_x\right)+\log \left(\frac{\det\left(\Sigma_\alpha\right)}{\det\left(\Sigma_x\right)}\right)\right] - \frac{\delta}{2} \sum_{i = 1}^d \log \left( 1- \lambda_i(H)  \right) \\
& + \tr(\Sigma_\beta)  - \frac{(\tau + \delta) d}{2}.
\end{align*}
First, we embark on the task of maximizing $\tr(K_{x\beta})$ while fixing $\Sigma_x$. Noting that $K_{x \beta}=\Sigma_x^{\frac{1}{2}} H \Sigma_\beta^{\frac{1}{2}}$, it follows that $\tr\left(K_{x \beta}\right)=\tr\left(H \Sigma_\beta^{\frac{1}{2}} \Sigma_x^{\frac{1}{2}}\right)$. Let $\left(\lambda_i(\Sigma_\beta^{\frac{1}{2}} \Sigma_x^{\frac{1}{2}})\right)_{i=1}^d$ represent the singular values of $\Sigma_{\beta}^{\frac{1}{2}} \Sigma_{x}^{\frac{1}{2}}$ in descending order. Applying von Neuman inequality, we get
\begin{align*}
    \tr(K_{x \beta}) \leq \sum_{ i = 1}^d \lambda_i\big(\Sigma_\beta^{\frac{1}{2}}\Sigma_x^{\frac{1}{2}}\big) \lambda_i(H).
\end{align*}
Now our task turns into maximizing 
\begin{align*}
    \sum_{i =1}^d \Big[ 2\lambda_i \left(\Sigma_\beta^{\frac{1}{2}}  \Sigma_x^{\frac{1}{2}} \right) \lambda_i(H) + \frac{\delta}{2} \log \left( 1 -\lambda_i(H) \right) \Big].
\end{align*} 
Consider the function $f(\lambda_i(H)) = 2 \lambda_i \left(\Sigma_\beta^{\frac{1}{2}} \Sigma_x^{\frac{1}{2}} \right) \lambda_i(H) + \frac{\delta}{2} \log \left( 1 -\lambda_i(H) \right)$. Taking derivative and setting it to $0$ we get
\begin{align*}
    \lambda_i \left(\Sigma_\beta^{\frac{1}{2}} \Sigma_x^{\frac{1}{2}} \right)  - \frac{\delta}{4} \frac{1}{1 - \lambda_i(H)}  = 0 \qquad 
    \Leftrightarrow \qquad \lambda_i(H)  = 1 - \frac{\delta}{4 \lambda_i \left(\Sigma_\beta^{\frac{1}{2}} \Sigma_x^{\frac{1}{2}} \right)}.
\end{align*}
Note that due to Lemma 3.4 \cite{le2022entropic}, all values of $\lambda_i(H)$ should be in interval $[0,1]$. Hence,
\begin{itemize}
    \item If $\lambda_i \left(\Sigma_\beta^{\frac{1}{2}} \Sigma_x^{\frac{1}{2}} \right) \geq \frac{\delta}{4}$ (equivalent to $ 1 - \frac{\delta}{4 \lambda_i \left(\Sigma_\beta^{\frac{1}{2}} \Sigma_x^{\frac{1}{2}}  \right)} \geq 0$), maximum value attains at 
    \begin{align*}
        \lambda_i^*(H) = 1 - \frac{\delta}{4 \lambda_i \left(\Sigma_\beta^{\frac{1}{2}} \Sigma_x^{\frac{1}{2}}  \right)}.
    \end{align*}
    \item If $\lambda_i \left(\Sigma_\beta^{\frac{1}{2}} \Sigma_x^{\frac{1}{2}} \right) < \frac{\delta}{4}$ (equivalent to $ 1 - \frac{\delta}{4 \lambda_i \left(\Sigma_\beta^{\frac{1}{2}} \Sigma_x^{\frac{1}{2}}  \right)} < 0$), maximum value attains at 
    \begin{align*}
        \lambda_i^*(H) = 0.
    \end{align*}
\end{itemize}
In conclusion
\begin{align*}
    \lambda_i^*(H)=\left\{\begin{array}{cl}1-\frac{\delta}{4} \lambda_i^{-1}\left(\Sigma_\beta^{\frac{1}{2}} \Sigma_x^{\frac{1}{2}}\right) & \text { if } \lambda_i\left(\Sigma_\beta^{\frac{1}{2}} \Sigma_x^{\frac{1}{2}}\right) \geq \frac{\delta}{4} \\ 0 & \text { otherwise }\end{array}\right..
\end{align*}
In this proof, we assume that $\delta$ is small enough that $\lambda_i^*(H) = 1-\frac{\delta}{4} \lambda_i^{-1}\left(\Sigma_\beta^{\frac{1}{2}} \Sigma_x^{\frac{1}{2}}\right) \forall i$. Since the equality of von Neuman inequality holds when $H$ and $\Sigma_\beta^{\frac{1}{2}} \Sigma_x^{\frac{1}{2}}$ are commuting, the eigenvalues of $H$ could be calculated from eigenvalues of $\Sigma_\beta^{\frac{1}{2}} \Sigma_x^{\frac{1}{2}}$; we obtain
\begin{align*}
    H = \Idsf - \frac{\delta}{4} \Sigma_x^{-\frac{1}{2}} \Sigma_\beta^{-\frac{1}{2}},
\end{align*}
which gives
\begin{align*}
    K_{x \beta} = \Sigma_x^{\frac{1}{2}} \Sigma_\beta^{\frac{1}{2}} - \frac{\delta}{4} \Idsf,
\end{align*}
and
\begin{align*}
    \sum_{i = 1}^d \log \left( 1 - \lambda_i(H) \right) & = \log \left( \Pi_{i=1}^d \left(1 - \lambda_i \left(H \right) \right)\right) \\
    & = \log \big( \det \left( \Idsf - H \right) \big) \\
    & = \log \left( (\frac{\delta}{4})^d \det \left( \Sigma_x^{-\frac{1}{2}} \Sigma_\beta^{-\frac{1}{2}} \right) \right) \\
    & = d \log (\frac{\delta}{4}) + \log \left( \det \left( \Sigma_x^{-\frac{1}{2}}\right) \right) + \log \left( \det \left( \Sigma_\beta^{-\frac{1}{2}}\right) \right).
\end{align*}
In such a scenario, we care about minimizing
\begin{align*}
\Upsilon = & \quad \tr\left(\Sigma_x\right)-2 \tr\left(\Sigma_x^{\frac{1}{2}} \Sigma_\beta^{\frac{1}{2}} - \frac{\delta}{4} \Idsf \right) + \frac{(\tau + \delta)}{2}\left[\tr\left(\Sigma_\alpha^{-1} \Sigma_x\right)+\log \left(\frac{\det\left(\Sigma_\alpha\right)}{\det\left(\Sigma_x\right)}\right)\right] \\ 
& - \frac{\delta}{2} \Big[ d \log (\frac{\delta}{4}) + \log \left( \det \left( \Sigma_x^{-\frac{1}{2}}\right) \right) + \log \left( \det \left( \Sigma_\beta^{-\frac{1}{2}}\right) \right) \Big] + \tr(\Sigma_\beta)  - \frac{(\tau + \delta) d}{2} . \\
= & \quad \tr\left(\Sigma_x\right)-2 \tr\left(\Sigma_x^{\frac{1}{2}} \Sigma_\beta^{\frac{1}{2}} \right) + \frac{(\tau + \delta)}{2}\left[\tr\left(\Sigma_\alpha^{-1} \Sigma_x\right) - \log \left(\det\left(\Sigma_x\right)\right)\right] - \frac{\delta}{2} \log \left( \det \left( \Sigma_x^{-\frac{1}{2}}\right) \right) \\ 
&  + \frac{\tau + \delta}{2} \log \left( \det ( \Sigma_\alpha) \right)- \frac{\delta}{2} \log \left( \det \left( \Sigma_\beta^{-\frac{1}{2}}\right) \right) + \tr(\Sigma_\beta)  - \frac{\tau d}{2} - \frac{\delta d}{2} \log(\frac{\delta}{4}).
\end{align*}
Letting $\Sigma_x^{\frac{1}{2}} \Sigma_\beta^{\frac{1}{2}} = \Sigmawtd$, $\Sigma_{\alpha, \beta, \tau, \delta} = \Sigma_\beta^{-\frac{1}{2}} \left( \Idsf + \frac{\tau + \delta}{2} \Sigma_\alpha^{-1}  \right) \Sigma_\beta^{-\frac{1}{2}}$, then
\begin{align*}
    \Upsilon = & \tr\left(\Sigmawtd^2 \Sigma_{\alpha,\beta, \tau, \delta}\right)-2 \tr\left(\Sigmawtd \right) - (\tau + \frac{3}{2} \delta) \log \left(\det\left(\Sigmawtd\right)\right) - \delta \log \left( \det \left( \Sigma_\beta^{\frac{1}{2}}\right) \right)\\ 
    &  + \frac{\tau + \delta}{2} \Big[ \log \left( \det ( \Sigma_\alpha) \right) + \log \left(\det \left(\Sigma_\beta \right) \right) \Big] + \tr(\Sigma_\beta)  - \frac{\tau d}{2} - \frac{\delta d}{2} \log(\frac{\delta}{4}).
\end{align*}
The only left part now is minimizing
\begin{align*}
    \tr\left(\Sigmawtd^2 \Sigma_{\alpha,\beta, \tau, \delta}\right)-2 \tr\left(\Sigmawtd \right) - (\tau + \frac{3}{2} \delta) \log \left(\det\left(\Sigmawtd\right)\right).
\end{align*}
Note that $\widetilde{\Sigma}$ is symmetric, so there exists a diagonal matrix $\Lambda$ that $\widetilde{\Sigma}$ is similar to $\Lambda$ (which implies $\Sigmawtd^2$ is similar to $\Lambda^2$. Assume that $\Lambda = \operatorname{diag}\big(\lambda_i(\Sigmawtd)\big)_{i = 1}^d$ which is decreasingly ordered and $\big(\lambda_{i}(\Sigma_{\alpha,\beta, \tau, \delta})\big)_{i = 1}^d$ are eigenvalues of $\Sigma_{\alpha, \beta, \tau, \delta}$ in ascending order, by Ruhe’s trace inequality
\begin{align*}
\tr\left(\Sigmawtd^2 \Sigma_{\alpha,\beta,\tau, \delta} \right)
& \geq \ \sum_{i=1}^d \lambda_i(\Sigmawtd)^2 \lambda_{d-i+1}(\Sigma_{\alpha,\beta, \tau, \delta}),
\end{align*}
where the equality holds when $\Sigmawtd^2$ and $\Sigma_{\alpha, \beta, \tau}$ are commuting. The optimization part now is calculated on eigenvalues of $\Lambda$
\begin{align*}
& \sum_{i=1}^d \lambda_i^2(\Sigmawtd) \lambda_{n-i+1}(\Sigma_{\alpha,\beta,\tau, \delta}) -2 \sum_{r=1}^d \lambda_i(\Sigmawtd)- (\tau + \frac{3}{2} \delta) \log \left(\prod_{i=1}^d \lambda_i(\Sigmawtd)\right) \\
& =\sum_{i=1}^d \left(\lambda_{n-i+1}(\Sigma_{\alpha,\beta,\tau, \delta}) \lambda_i^2(\Sigmawtd)-2 \lambda_i(\Sigmawtd)- (\tau + \frac{3}{2} \delta) \log \left(\lambda_i\big(\Sigmawtd\big)\right)\right).
\end{align*}
Consider the function
\begin{align*}
    f(v)=u v^2-2 v- \theta \log (v).
\end{align*}
Take the derivative and set it to $0$, we get
\begin{align*}
\frac{2 u v^2-2v-\theta}{v}=0,
\end{align*}
which has unique positive solution $v^* = \frac{1+\sqrt{1+2u \theta}}{2 u}$. 
It verifies that, to attain the minimization of $\Upsilon_\Sigma$,
\begin{align*}\lambda_i(\Sigmawtd) = \frac{1+\sqrt{1+(2\tau + 3\delta) \lambda_{n-i+1}(\Sigma_{\alpha,\beta,\tau, \delta})}}{2 \lambda_{n-i+1}(\Sigma_{\alpha,\beta,\tau, \delta})}.
\end{align*}
This results in the following expression
\begin{align*}
\Sigmawtd \quad \text{is similar to}\quad  \Lambdaol:= \operatorname{diag} \left( \frac{1+\sqrt{1+ (2 \tau + 3\delta) \lambda_{n-i+1}(\Sigma_{\alpha,\beta,\tau, \delta})}}{2 \lambda_{n-i+1}(\Sigma_{\alpha,\beta,\tau, \delta})} \right).
\end{align*}
Since $\Sigmawtd$ and $\Sigma_{\alpha,\beta,\tau, \delta}$ are commuting, the eigenvalue of $\Sigmawtd$ can be computed from the eigenvalues of $\Sigma_{\alpha,\beta,\tau, \delta}$, we get
\begin{align*}
\Sigmawtd^* = \frac{1}{2}\Sigma_{\alpha,\beta,\tau, \delta}^{-1} + \frac{1}{2} \Big[ \Sigma_{\alpha,\beta,\tau, \delta}^{-2} + (2\tau + 3\delta) \Sigma_{\alpha,\beta,\tau, \delta}^{-1}\Big]^{\frac{1}{2}}.
\end{align*}
The equation $2u (v^*)^2 - 2v^* - \theta = 0$ deduces that $(v^*)^2 = \frac{v^*}{u} + \frac{\theta}{2u}$. Hence, we have
\begin{align*}
    \big[\Sigmawtd^*\big]^2 = \frac{\tau}{2}\Sigma_{\alpha,\beta,\tau, \delta}^{-1} + \frac{1}{2}\Sigma_{\alpha,\beta,\tau, \delta}^{-2}\Big[\Idsf + \big(\Idsf + (2\tau + 3\delta)  \Sigma_{\alpha,\beta,\tau, \delta} \big)^{\frac{1}{2}} \Big].
\end{align*}
That leads to the formula for $\Sigma_x$
\begin{align*}
    \Sigma_{x} = \Sigma_\beta^{-\frac{1}{2}} \bigg[ \frac{\tau}{2}\Sigma_{\alpha,\beta,\tau, \delta}^{-1} + \frac{1}{2}\Sigma_{\alpha,\beta,\tau, \delta}^{-2}\Big[\Idsf + \big(\Idsf + (2\tau + 3\delta) \Sigma_{\alpha,\beta,\tau, \delta} \big)^{\frac{1}{2}} \Big] \bigg] \Sigma_\beta^{-\frac{1}{2}}.
\end{align*}
Hence we complete the proof.
\end{proof}


\section{Proof for Theorem \ref{UOT Derivative}} \label{proof for uot derivative}
First we need these below lemmas
\begin{lemma}[Trace] \label{lemma:trace_SPD}
Let $A$ and $B$ be SPD matrices of the same size. Then
\begin{enumerate}
    \item $\tr\big(\big[A B\big]^{\frac{1}{2}}\big) = \tr\big(\big[B A\big]^{\frac{1}{2}}\big) = \tr\big(\big[A^{\frac{1}{2}}B A^{\frac{1}{2}} \big]^{\frac{1}{2}} \big) = \tr\big( \big[B^{\frac{1}{2}}A B^{\frac{1}{2}}\big]^{\frac{1}{2}}\big).$
    \item $\tr\Big(\Big[(AB)^2 + \tau(AB)\Big]^{\frac{1}{2}} \Big) = \tr\Big( \Big\{ \big[A^{\frac{1}{2}}BA^{\frac{1}{2}} \big]^2+ \tau \big[A^{\frac{1}{2}}BA^{\frac{1}{2}}\big] \Big\}^{\frac{1}{2}}\Big)$.
\end{enumerate}
\end{lemma}
\begin{proof} For part (\textbf{1}),
first we note that
\begin{align*}
    \big[AB\big]^{\frac{1}{2}}
= A^{\frac{1}{2}}\big[A^{\frac{1}{2}}B A^{\frac{1}{2}}\big]^{\frac{1}{2}} A^{-\frac{1}{2}}.
\end{align*}
Thus, we get
\begin{align*}
     \tr\big(\big[AB\big]^{\frac{1}{2}}\big) &= \tr\big(\big[A^{\frac{1}{2}}B A^{\frac{1}{2}} \big]^{\frac{1}{2}} \big), \\
     \tr\big(\big[BA\big]^{\frac{1}{2}}\big) &= \tr\big(\big[B^{\frac{1}{2}}A B^{\frac{1}{2}} \big]^{\frac{1}{2}} \big).
\end{align*}
The above equations also means that $\big[A B\big]^{\frac{1}{2}}$ and $\big[A^{\frac{1}{2}}B A^{\frac{1}{2}}\big]^{\frac{1}{2}} $
are similar matrices. That follows that their eigenvalues sets are coincidence. Note that the eigenvalues of $\big[A^{\frac{1}{2}}B A^{\frac{1}{2}}\big]^{\frac{1}{2}}$ are square root of eigenvalues of $\big[A^{\frac{1}{2}}B A^{\frac{1}{2}}\big]$. Furthermore,
\begin{align*}
    A^{\frac{1}{2}}B A^{\frac{1}{2}} = A^{\frac{1}{2}}B^{\frac{1}{2}} B^{\frac{1}{2}}A^{\frac{1}{2}}.
\end{align*}
Note that matrices $UV$ and $VU$ have the same set of eigenvalues. Hence,  $A^{\frac{1}{2}}B A^{\frac{1}{2}}$ and $B^{\frac{1}{2}}A B^{\frac{1}{2}}$ have the same set of eigenvalues. Finally, all four matrices have the same sets of eigenvalues. 

For part (\textbf{2}), since the formula in part (a) between  $[AB]^{\frac{1}{2}}$ and $\big[A^{\frac{1}{2}} BA^{\frac{1}{2}}\big]^{\frac{1}{2}}$, we have
\begin{align*}
    (AB)^2 + \tau (AB) &= A^{\frac{1}{2}} \big[A^{\frac{1}{2}}BA^{\frac{1}{2}}\big]^2 A^{-\frac{1}{2}} + \tau A^{\frac{1}{2}} \big[A^{\frac{1}{2}}BA^{\frac{1}{2}}\big] A^{-\frac{1}{2}} \\
    &= A^{\frac{1}{2}} \Big\{ \big[A^{\frac{1}{2}}BA^{\frac{1}{2}}\big]^2 + \tau \big[A^{\frac{1}{2}}BA^{\frac{1}{2}}\big] \Big\}A^{-\frac{1}{2}}.
\end{align*}
It follows that
\begin{align*}
    \Big[(AB)^2 + \tau(AB)\Big]^{\frac{1}{2}} = A^{\frac{1}{2}}
\Big\{ \big[A^{\frac{1}{2}}BA^{\frac{1}{2}}\big]^2 + \tau \big[A^{\frac{1}{2}}BA^{\frac{1}{2}}\big] \Big\}^{\frac{1}{2}}A^{-\frac{1}{2}}.
\end{align*}
We deduce that
\begin{align*}
    \tr\Big(\Big[(AB)^2 + \tau(AB)\Big]^{\frac{1}{2}} \Big) = \tr\Big( \Big\{ \big[A^{\frac{1}{2}}BA^{\frac{1}{2}}\big]^2 + \tau \big[A^{\frac{1}{2}}BA^{\frac{1}{2}}\big] \Big\}^{\frac{1}{2}}\Big).
\end{align*}
\end{proof}

\begin{lemma} \label{derivative of lod det}
Let $A$ and $B$ be two symmetric matrices of the same size. Then
\begin{align*}
    \frac{\partial \log \det(A+tB)}{\partial t}\Big|_{t=0} 
    &= \tr\big(A^{-\frac{1}{2}} BA^{-\frac{1}{2}}\big).
\end{align*}
\end{lemma}
\begin{proof}
Note that $A + t B = A^{\frac{1}{2}}\big[\Idsf + tA^{-\frac{1}{2}}BA^{-\frac{1}{2}}\big]A^{\frac{1}{2}}$. Then
\begin{align*}
    \log \det(A + tB) &= \log \det(A) + \log \det\big(\Idsf + t A^{-\frac{1}{2}} BA^{-\frac{1}{2}}\big) \\
    &= \log \det(A) + \sum_{i=1}^d \log\Big(1 + t\lambda_i\big(  A^{-\frac{1}{2}} BA^{-\frac{1}{2}}\big)\Big).
\end{align*}
Taking derivative with respect to $t$ of both sides
\begin{align*}
    \frac{\partial \log \det(A+tB)}{\partial t}\Big|_{t=0} &= \sum_{i=1}^d\frac{\lambda_i\big(A^{-\frac{1}{2}} BA^{-\frac{1}{2}} \big)}{1+ t\lambda_i\big(A^{-\frac{1}{2}} BA^{-\frac{1}{2}} \big)}\Big|_{t=0} = \sum_{i=1}^d \lambda_i\big(A^{-\frac{1}{2}} BA^{-\frac{1}{2}}\big)\\
    &= \tr\big(A^{-\frac{1}{2}} BA^{-\frac{1}{2}}\big).
\end{align*}
\end{proof}

The next lemma is about the Taylor expansion for trace of square root matrix.
\begin{lemma} \label{lem: Taylor of trace for square root}
Given $B \in \mathbb{S}^{d}_{++}$, we define the Lyapunov's operator $\mathcal{L}$ as:
\begin{align*}
    \mathcal{L}:  \quad  \mathbb{S}^{d}_{++} &\to \mathbb{S}^{d}_{++} \\
     A &\mapsto \mathcal{L}_{B}[A]
\end{align*}
where $\mathcal{L}_{B}[A]$ is the matric satisfying $\mathcal{L}_{B}[A] B + B \mathcal{L}_{B}[A] = A$. Then let $\Sigma_0$ and $A$ be SPD matrices in $\mathbb{S}^{d}_{++}$,
\begin{align*}
    \tr\big(\big[\Sigma_0 + tA\big]^{\frac{1}{2}} \big) = \tr\big(\Sigma_0^{\frac{1}{2}} \big) + \frac{1}{2} t \tr\big( \Lcal_{\Sigma_0^{\frac{1}{2}}}[A] \big)  + o(t).
\end{align*}
Moreover, we have
\begin{align*}
    \tr\big(\Lcal_{\Sigma_0^{\frac{1}{2}}}[A]\big) = \tr\big(\Sigma_0^{-\frac{1}{2}}A\big).
\end{align*}
\end{lemma}
\begin{proof}
    Assume that $\big[\Sigma_0 + t A\big]^{\frac{1}{2}} = \Sigma_0^{\frac{1}{2}} + t X$, then
    \begin{align*}
        &\Sigma_0 + tA = \big[\Sigma_0^{\frac{1}{2}} + tX\big]^2 = \Sigma_0 + t\big[\Sigma_0^{\frac{1}{2}} X + X \Sigma_0^{\frac{1}{2}}\big] + t^2X^2\\
        \Rightarrow \quad &A = \Sigma_0^{\frac{1}{2}} X + X\Sigma_0^{\frac{1}{2}}.
    \end{align*}
Then the Lyapunov's operator produces $X$ from $\Sigma_0$ and $A$ is equal to 
$X = \Lcal_{\Sigma_0^{\frac{1}{2}}}[A]$. We have 
\begin{align*}
    &\Sigma_0^{-\frac{1}{2}}A = X+ \Sigma_0^{-\frac{1}{2}} X \Sigma_0^{\frac{1}{2}},
\end{align*}
that leads to $\frac{1}{2}\tr(\Sigma_0^{-\frac{1}{2}}A) = \tr(X)$.
\end{proof}

Now we are ready to give the full proof for Theorem \ref{UOT Derivative}. 
\begin{proof}
Use the notation $\Sigma_{x}$ as an optimal solution to problem \ref{UOT Plan}. Denote
\begin{align*}
    \Big[\Sigma_\beta^{\frac{1}{2}} \Sigma_{x} \Sigma_\beta^{\frac{1}{2}}\Big]^{\frac{1}{2}} &= \Sigma_{x,\beta}.
\end{align*}
Since $\Sigma_{x,\beta}$ and $\Sigma_{\alpha,\tau,\beta}$ share the same set of eigenvectors, we have
\begin{align*}
\Sigma_{x,\beta} = \frac{1}{2}\Sigma_{\alpha,\tau,\beta}^{-1} + \frac{1}{2}\Big[ \Sigma_{\alpha,\tau,\beta}^{-2} + 2\tau \Sigma_{\alpha,\tau,\beta}^{-1}\Big]^{\frac{1}{2}}.
\end{align*}
The objective function $W^2_{2_{\operatorname{SUOT}}}(\alpha, \beta, \tau)$ is equal to
\begin{align*}
    \tr\big( \Sigma_\beta\big) - \tr\big(\Sigma_{x,\beta}\big) - \frac{\tau}{4} \log \det(\Sigma_{x})+ \mathsf{constant}.
\end{align*}
Let $\gamma(t) = \Sigma_{\beta\rightarrow z,t}$ for $t\in [0,1]$ be the geodesic on Bures-Wasserstein manifold from $\Sigma_\beta$ to $\Sigma_z$, then   
\begin{align*}
    \Sigma_{\beta\rightarrow z, t} &= \Big[\Idsf + t\big(T_{\Sigma_\beta\rightarrow \Sigma_z}- \Idsf\big) \Big] \Sigma_\beta \Big[\Idsf + t\big(T_{\Sigma_\beta\rightarrow \Sigma_z}- \Idsf\big) \Big]\\
    &= \Sigma_\beta + t \big(T_{\beta z} \Sigma_{\beta}  + \Sigma_{\beta}T_{\beta z}\big) + t^2 T_{\beta z}\Sigma_\beta T_{\beta z},
\end{align*}
where $T_{\beta z} = T_{\Sigma_\beta \rightarrow \Sigma_z} - \Idsf$. We consider the derivative of the loss function on the geodesic $\gamma(t)$. 

\noindent\textbf{Derivative of $\tr(\Sigma_\beta)$:}
By the above formula,
\begin{align*}
    \frac{\partial \tr(\Sigma_{\beta\rightarrow z,t})}{\partial t}\Big|_{t=0} &= 2 \tr(\Sigma_\beta T_{\beta z}) = \big\langle 2\Idsf,T_{\beta z}\big\rangle_{\Sigma_\beta}.\\
\end{align*}
\textbf{Derivative of $\tr\big(\Sigma_{x,\beta} \big)$:} We recall the formula of $\Sigma_{x,\beta}$ as
\begin{align*}
    \tr\big(\Sigma_{x,\beta}\big) &= \tr\bigg(\frac{1}{2}\Big\{\Sigma_{\alpha,\tau,\beta}^{-1} + \Big[\Sigma_{\alpha,\tau,\beta}^{-2} + \tau \Sigma_{\alpha,\tau,\beta}^{-1} \Big]^{\frac{1}{2}} \Big\}\bigg).
\end{align*}
We deal with each term separately. We start with $\Sigma_{\alpha,\tau,\beta}$.
\begin{align*}
    &\Sigma_{\alpha,\tau,\beta} = \Sigma_{\beta}^{-\frac{1}{2}} \Big( \Idsf + \frac{\tau}{2}\Sigma_{\alpha}^{-1}\Big) \Sigma_{\beta}^{-\frac{1}{2}}\\
    \Rightarrow \quad & \Sigma_{\alpha,\tau,\beta}^{-1} = \Sigma_{\beta}^{\frac{1}{2}}\Big( \Idsf + \frac{\tau}{2}\Sigma_{\alpha}^{-1}\Big)^{-1}\Sigma_{\beta}^{\frac{1}{2}} \\
    \Rightarrow \quad &\tr\big(\Sigma_{\alpha,\tau,\beta}^{-1} \big) = \tr\Big( \Sigma_{\beta}\big[ \Idsf + \frac{\tau}{2}\Sigma_{\alpha}^{-1}]^{-1} \Big) = \tr\big( \Sigma_\beta \Sigma_{\alpha,\tau}^{-1}\big).
\end{align*}
It follows that
\begin{align*}
    \frac{\partial \tr(\Sigma_{\alpha,\tau,\beta\rightarrow z,t}^{-1})}{\partial t}\bigg|_{t=0} &= \big\langle 2\Sigma_{\alpha,\tau}^{-1},T_{\beta z}\big\rangle_{\Sigma_\beta}. \\
\end{align*}
We move to the next term. We first recall that
\begin{align*}
\tr(\Sigma_{x,\beta}) &= \tr\Big(\Big[\Sigma_\beta^{\frac{1}{2}}\Sigma_x \Sigma_\beta^{\frac{1}{2}}\Big]^{\frac{1}{2}} \Big) = \tr\Big( \Big[\Sigma_x^{\frac{1}{2}}\Sigma_\beta \Sigma_{x}^{\frac{1}{2}}\Big]^{\frac{1}{2}}\Big)\\
    \Sigma_{\alpha,\tau,\beta}^{-1} &= \Sigma_\beta^{\frac{1}{2}}\Big(\Idsf + \frac{\tau}{2}\Sigma_{\alpha}^{-1}\Big)^{-1} \Sigma_{\beta}^{\frac{1}{2}} = \Sigma_\beta^{\frac{1}{2}}\Sigma_{\alpha,\tau}\Sigma_{\beta}^{\frac{1}{2}}.
\end{align*}
If we define
\begin{align*}
    \Sigma_{\alpha,\beta,\tau}^{-1} = \Sigma_{\alpha,\tau}^{-\frac{1}{2}}\Sigma_\beta \Sigma_{\alpha,\tau}^{-\frac{1}{2}},
\end{align*}
then  $\Sigma_{\alpha,\tau,\beta}^{-1}$ and  $\Sigma_{\alpha,\beta,\tau}^{-1} $ share the same set of eigenvalues. Hence,
\begin{align*}
    \tr\Big( \Big[\Sigma_{\alpha,\tau,\beta}^{-2} + \tau \Sigma_{\alpha,\tau,\beta}^{-1}\Big]^{\frac{1}{2}}\Big) = \tr\Big( \Big[\Sigma_{\alpha,\beta,\tau}^{-2} + \tau \Sigma_{\alpha,\beta,\tau}^{-1}\Big]^{\frac{1}{2}}\Big).
\end{align*}
Instead of working with the LHS, we work with the RHS. Then, we have
\begin{align*}
 \tr\Big(\Big[\Sigma_{\alpha,\beta,\tau}^{-2}  + \tau \Sigma_{\alpha,\beta,\tau}^{-1}\Big]^{\frac{1}{2}} \Big) &= \tr\Big(\Big\{ \big[\Sigma_{\alpha,\tau}^{-\frac{1}{2}}\Sigma_\beta \Sigma_{\alpha,\tau}^{-\frac{1}{2}} \big]^2 + \tau \big[\Sigma_{\alpha,\tau}^{-\frac{1}{2}}\Sigma_\beta \Sigma_{\alpha,\tau}^{-\frac{1}{2}}\big]\Big\}^{\frac{1}{2}}\Big).
 \end{align*}
Replace $\Sigma_{\beta}$ by $\Sigma_{\beta\rightarrow z,t} = \Sigma_\beta + t\big(T_{\beta z}\Sigma_\beta + \Sigma_\beta T_{\beta z} \big) + t^2T_{\beta z}\Sigma_{\beta}T_{\beta z}$, we have
\begin{align*}
     \Sigma_{\alpha,\tau}^{-\frac{1}{2}}\Sigma_{\beta \rightarrow z,t}\Sigma_{\alpha,\tau}^{-\frac{1}{2}} = & \Sigma_{\alpha,\tau}^{-\frac{1}{2}} \Sigma_{\beta}\Sigma_{\alpha,\tau}^{-\frac{1}{2}} + t \Sigma_{\alpha,\tau}^{-\frac{1}{2}}\big(T_{\beta z} \Sigma_\beta + \Sigma_\beta T_{\beta z}\big) \Sigma_{\alpha,\tau}^{-\frac{1}{2}} + t^2 \Sigma_{\alpha,\tau}^{-\frac{1}{2}}T_{\beta z}\Sigma_\beta T_{\beta z}\Sigma_{\alpha,\tau}^{-\frac{1}{2}}\\
     \Big[\Sigma_{\alpha,\tau}^{-\frac{1}{2}}\Sigma_{\beta \rightarrow z,t}\Sigma_{\alpha,\tau}^{-\frac{1}{2}} \Big]^2 & = \Big[\Sigma_{\alpha,\tau}^{-\frac{1}{2}} \Sigma_{\beta}\Sigma_{\alpha,\tau}^{-\frac{1}{2}}\Big]^2  + \mathcal{O}(t^3) + \\
     & \hspace{- 7 em}  t\Big\{ \big[\Sigma_{\alpha,\tau}^{-\frac{1}{2}}\big(T_{\beta z} \Sigma_\beta + \Sigma_\beta T_{\beta z}\big) \Sigma_{\alpha,\tau}^{-\frac{1}{2}} \big] \big[\Sigma_{\alpha,\tau}^{-\frac{1}{2}} \Sigma_{\beta}\Sigma_{\alpha,\tau}^{-\frac{1}{2}}\big] + \big[\Sigma_{\alpha,\tau}^{-\frac{1}{2}} \Sigma_{\beta}\Sigma_{\alpha,\tau}^{-\frac{1}{2}} \big] \big[\Sigma_{\alpha,\tau}^{-\frac{1}{2}}\big(T_{\beta z} \Sigma_\beta + \Sigma_\beta T_{\beta z}\big) \Sigma_{\alpha,\tau}^{-\frac{1}{2}}\big]\Big\} + \\
      & \hspace{- 7 em} t^2\Sigma_{\alpha,\tau}^{-\frac{1}{2}}\Big[T_{\beta z}\Sigma_\beta T_{\beta z} \Sigma_{\alpha,\tau}^{-1}\Sigma_\beta + \Sigma_\beta \Sigma_{\alpha,\tau}^{-1}T_{\beta z}\Sigma_\beta T_{\beta z} + \big(T_{\beta z}\Sigma_\beta + \Sigma_\beta T_{\beta z} \big)\Sigma_{\alpha,\tau}^{-1}\big(T_{\beta z}\Sigma_\beta + \Sigma_\beta T_{\beta z} \big) \Big] \Sigma_{\alpha,\tau}^{-\frac{1}{2}} .
\end{align*}
Denote
\begin{align*}
    \bigg\{\Big[\Sigma_{\alpha,\tau}^{-\frac{1}{2}}\Sigma_{\beta}\Sigma_{\alpha,\tau}^{-\frac{1}{2}} \Big]^2 + \tau \Big[\Sigma_{\alpha,\tau}^{-\frac{1}{2}}\Sigma_{\beta }\Sigma_{\alpha,\tau}^{-\frac{1}{2}} \Big]\bigg\}^{\frac{1}{2}} = \widetilde{\Sigma}_{\beta,\alpha,\tau}.
\end{align*}
Taking derivative gives
\begin{align*}
    & 2\frac{\partial \tr\Big(  \Big\{ \big[\Sigma_{\alpha,\tau}^{-\frac{1}{2}}\Sigma_{\beta\rightarrow z,t} \Sigma_{\alpha,\tau}^{-\frac{1}{2}} \big]^2 + \tau \big[\Sigma_{\alpha,\tau}^{-\frac{1}{2}}\Sigma_{\beta\rightarrow z,t} \Sigma_{\alpha,\tau}^{-\frac{1}{2}}\big]\Big\}^{\frac{1}{2}} \Big)}{\partial t}\bigg|_{t=0} \\
    & = \tau \tr\Big(\widetilde{\Sigma}_{\beta,\alpha,\tau}^{-1}\big[\Sigma_{\alpha,\tau}^{-\frac{1}{2}}\big(T_{\beta z} \Sigma_\beta + \Sigma_\beta T_{\beta z}\big) \Sigma_{\alpha,\tau}^{-\frac{1}{2}}\big] \Big). \\ 
& \quad +\tr\Big(\widetilde{\Sigma}_{\beta,\alpha,\tau}^{-1} \Big[ \Sigma_{\alpha,\tau}^{-\frac{1}{2}}\big(T_{\beta z} \Sigma_\beta + \Sigma_\beta T_{\beta z}\big) \Sigma_{\alpha,\tau} \Sigma_{\beta}\Sigma_{\alpha,\tau}^{-\frac{1}{2}} + \Sigma_{\alpha,\tau}^{-\frac{1}{2}} \Sigma_{\beta}\Sigma_{\alpha,\tau}\big(T_{\beta z} \Sigma_\beta + \Sigma_\beta T_{\beta z}\big) \Sigma_{\alpha,\tau}^{-\frac{1}{2}}\Big] \Big)
    \end{align*}
The first term of the RHS is equal to
\begin{align*}
    \tr \Big(\Sigma_{\alpha,\tau}^{-\frac{1}{2}}\widetilde{\Sigma}^{-1}_{\beta,\alpha,\tau}\Sigma_{\alpha,\tau}^{-\frac{1}{2}} \big[T_{\beta z}\Sigma_\beta + \Sigma_\beta T_{\beta z}\big]\Big) = \Big\langle 2 M, T_{\beta z} \Big\rangle_{\Sigma_\beta},
\end{align*}
where we denote
\begin{align*}
    M = \Sigma_{\alpha,\tau}^{-\frac{1}{2}}\widetilde{\Sigma}^{-1}_{\beta,\alpha,\tau}\Sigma_{\alpha,\tau}^{-\frac{1}{2}}.
\end{align*}
Similarly, the second term of the RHS is equal to
\begin{align*}
\tr\Big( \big[\Sigma_{\alpha,\tau}^{-1}\Sigma_\beta M+M\Sigma_\beta \Sigma_{\alpha,\tau}^{-1}\big]  \big[ T_{\beta z}\Sigma_\beta + \Sigma_\beta T_{\beta z}\big]\Big) = \Big\langle 2 U, T_{\beta z} \Big\rangle_{\Sigma_\beta},
\end{align*}
where we denote
\begin{align*}
    U = \Sigma_{\alpha,\tau}^{-1}\Sigma_\beta M+M\Sigma_\beta \Sigma_{\alpha,\tau}^{-1}.
\end{align*}
Hence, the derivative of RHS is equal to
\begin{align*} 
    \big\langle \tau M + U,T_{\beta z}\big\rangle_{\Sigma_\beta}
\end{align*}
and final derivative of $\tr(\Sigma_{x,\beta})$ is
\begin{align*}
    \big\langle \Sigma_{\alpha,\tau}^{-1} + \frac{1}{2} \big[\tau M + U \big],T_{\beta z} \big\rangle_{\Sigma_\beta} .
\end{align*}
\textbf{Derivative of the log det:} We recall some formulas
\begin{align*}
\det(\Sigma_{x}) = \det(\Sigma_{\beta\rightarrow z,t})^{-1}\det\big(\Sigma_{x^*,\beta \rightarrow z,t} \big)^2
\end{align*}
Taking the logarithm of both sides
\begin{align*}
    \log \det(\Sigma_{x}) = -\log\det(\Sigma_{\beta \rightarrow z,t}) + 2 \log \det(\Sigma_{x^*,\beta\rightarrow z,t}).
\end{align*}
For the second term,
\begin{align*}
\Sigma_{x^*,\beta\rightarrow z,t} = \Sigma^{-1}_{\alpha,\tau,\beta\rightarrow z,t} \Big\{\Idsf + \big[\Idsf + \tau \Sigma_{\alpha,\tau,\beta\rightarrow z,t} \big]^{\frac{1}{2}}\Big\}
\end{align*}
then 
\begin{align*}
    \log \det(\Sigma_{x^*,\beta\rightarrow z,t}) = - \log \det\big(\Sigma_{\alpha,\tau,\beta\rightarrow z,t}\big) + \log \det \big( \Idsf + \big[\Idsf + \tau \Sigma_{\alpha,\tau,\beta\rightarrow z,t} \big]^{\frac{1}{2}} \big).
\end{align*}
For the first sub-terms
\begin{align*}
\log \det\big(\Sigma_{\alpha,\tau,\beta\rightarrow z,t}\big) =  \log \det\big( \Sigma_{\beta\rightarrow z,t}\big) + \det(\Sigma_{\alpha,\tau}^{-1}),
\end{align*}
then we have
\begin{align*}
    \log \det(\Sigma_{x}) = -3\log\det(\Sigma_{\beta \rightarrow z,t}) - 2\det(\Sigma_{\alpha,\tau}^{-1}) + 2 \log \det \big( \Idsf + \big[\Idsf + \tau \Sigma_{\alpha,\tau,\beta\rightarrow z,t} \big]^{\frac{1}{2}} \big).
\end{align*}
For the first term RHS of the new expression,
\begin{align*}
    \log \det(\Sigma_{\beta\rightarrow z,t}) = \log\det (\Sigma_{\beta}) + 2 \log \det(\Idsf + tT_{\beta z}).
\end{align*}
We also have
\begin{align*}
  \frac{\partial \log \det(\Idsf + tT_{\beta x}) }{\partial t}\bigg|_{t=0} &= \sum_{i=1}^d\frac{\lambda_i(T_{\beta z})}{ 1+t\lambda_i(T_{\beta z})}\bigg|_{t=0} = \tr(T_{\beta z}) = \big\langle \Sigma_\beta^{-1}, T_{\beta z}\big\rangle_{\Sigma_\beta}.\\
\end{align*}
then $\frac{\partial \log \det(\Sigma_{\beta\rightarrow z,t}) }{\partial t}\bigg|_{t=0} = \big\langle 2 \Sigma_\beta^{-1}, T_{\beta z}\big\rangle_{\Sigma_\beta}$. 
For the second sub-terms, observe that
\begin{align*}
    \Big[(\Idsf + tT_{\beta z})\Sigma_{\beta}(\Idsf + tT_{\beta z})\Big]^{-1} & = (\Idsf + tT_{\beta z})^{-1}\Sigma_{\beta}^{-1}(\Idsf + tT_{\beta z})^{-1}  \\
    & = \big( \Idsf - tT_{\beta z} + t^2 T_{\beta z}^2\big) \Sigma_{\beta}^{-1}\big(\Idsf - tT_{\beta z} +t^2 T_{\beta z}^2\big) + O(t^3)\\
    & = \Sigma_{\beta}^{-1} - t\big(T_{\beta z} \Sigma_{\beta}^{-1}+ \Sigma_{\beta}^{-1}T_{\beta z} \big) + t^2\Big[T_{\beta z}^2 \Sigma_\beta^{-1} + \Sigma_{\beta}^{-1} T_{\beta z}^2 + T_{\beta z}\Sigma_\beta^{-1} T_{\beta z}\Big] + \mathcal{O}(t^3).
\end{align*}
Hence
\begin{align*}
    \Sigma_{\alpha,\tau,\beta\rightarrow z,t} &= \Sigma_{\alpha,\tau}^{\frac{1}{2}} \Sigma_{\beta}^{-1} \Sigma_{\alpha,\tau}^{\frac{1}{2}}  - t  \Sigma_{\alpha,\tau}^{\frac{1}{2}} \big(T_{\beta z} \Sigma_{\beta}^{-1}+ \Sigma_{\beta}^{-1}T_{\beta z} \big)   \Sigma_{\alpha,\tau}^{\frac{1}{2}}  \\
    &= \Sigma_{\alpha,\tau,\beta} -t  \Sigma_{\alpha,\tau}^{\frac{1}{2}} \big(T_{\beta z} \Sigma_{\beta}^{-1}+ \Sigma_{\beta}^{-1}T_{\beta z} \big)   \Sigma_{\alpha,\tau}^{\frac{1}{2}}.
\end{align*}
Put it into the form
\begin{align*}
    \Idsf + \tau \Sigma_{\alpha,\tau,\beta\rightarrow z,t} = \Idsf + \tau \Sigma_{\alpha,\tau}^{\frac{1}{2}} \Sigma_{\beta}^{-1} \Sigma_{\alpha,\tau}^{\frac{1}{2}} - t \tau \Sigma_{\alpha,\tau}^{\frac{1}{2}} \big(T_{\beta z} \Sigma_{\beta}^{-1}+ \Sigma_{\beta}^{-1}T_{\beta z} \big)   \Sigma_{\alpha,\tau}^{\frac{1}{2}}. 
\end{align*}
Then by Lemma \ref{lem: Taylor of trace for square root}
\begin{align*}
     \Big[\Idsf + \tau \Sigma_{\alpha,\tau,\beta\rightarrow z,t}\Big]^{\frac{1}{2}} = V - t\Lcal_{V}\Big[\tau \Sigma_{\alpha,\tau}^{\frac{1}{2}} \big(T_{\beta z} \Sigma_{\beta}^{-1}+ \Sigma_{\beta}^{-1}T_{\beta z} \big)   \Sigma_{\alpha,\tau}^{\frac{1}{2}}\Big] + o(t).
\end{align*}
where 
 $\Big[ \Idsf + \tau \Sigma_{\alpha,\tau}^{\frac{1}{2}} \Sigma_{\beta}^{-1} \Sigma_{\alpha,\tau}^{\frac{1}{2}}\Big]^{\frac{1}{2}} = V$. Therefore, applying Lemma \ref{derivative of lod det}, we get the derivative of the last term as
  \begin{align*}
      -\tr\bigg(\big[\Idsf + V\big]^{-1} \Big[\tau \Sigma_{\alpha,\tau}^{\frac{1}{2}} \big(T_{\beta z} \Sigma_{\beta}^{-1}+ \Sigma_{\beta}^{-1}T_{\beta z} \big)   \Sigma_{\alpha,\tau}^{\frac{1}{2}}\Big]\bigg) 
     =  - \tau \big\langle P+Q,T_{\beta z}\big\rangle_{\Sigma_\beta},
 \end{align*}
 where
\begin{align*}
     P & = \Sigma_\beta^{-1} \Sigma_{\alpha,\tau}^{\frac{1}{2}} \big[\Idsf + V\big]^{-1}\Sigma_{\alpha,\tau}^{\frac{1}{2}}\Sigma_\beta^{-1} \\
     Q & = \Sigma_{\alpha,\tau}^{\frac{1}{2}} \big[\Idsf + V\big]^{-1}  \Sigma_{\alpha,\tau}^{\frac{1}{2}} \Sigma_\beta^{-2}.
 \end{align*}
 and the derivative of $\log \det$ term
 \begin{align*}
     &  \big\langle -6 \Sigma_\beta^{-1} - 2 \tau (P+Q) , T_{\beta z}\big\rangle_{\Sigma_\beta}.
 \end{align*} 
Hence, in conclusion, we get the first order Wasserstein gradient according to $\Sigma_\beta$ of the objective function
\begin{align*}
     2 \Idsf - \Big( 2\Sigma_{\alpha,\tau}^{-1} + \frac{1}{2} (U+ \tau M)\Big ) + \frac{3}{2} \tau \Sigma_\beta^{-1} + \frac{\tau^2}{2}(P+Q).
\end{align*}
\end{proof}

\section{Proof for Theorem \ref{theo: exact converges} and \ref{theo: hybrid converges}} \label{proof for convergence}
\subsection{Proof for Convexity} \label{proof for convexity}

\begin{lemma} \label{lem: compa eigen}
    Given two SPD matrices $A,B \in \mathbb{S}_{++}^{d}$ Recall that $\left\{\lambda_i(\Sigma)\right\}_{i=1}^d$ is the eigenvalues of a matrix $\Sigma$ in descending order. Then we have $\tr(AB) \geq \lambda_d(A) \tr(B)$.
\end{lemma}
\begin{proof}
    Note that this is equivalent to proving that $\tr \Big( \big[A - \lambda_d(A) \Idsf\big] B\Big) \geq 0$. Let $A - \lambda_d(A) \Idsf = U\Lambda U^{\top}$ be its spectral decomposition, where $\Lambda = \mathbf{diag}(a_i)$. Let $C= [c_{ij}] = U^{\top}B U$. Then
\begin{align*}
    \tr\Big( \big[A- \lambda_d(A)\Idsf\big] B\Big) = \tr\big(\Lambda U^{\top} B U \big) = \tr(\Lambda C) = \sum_i a_i c_{ii} \geq 0,
\end{align*}
since both $\Lambda$ and $C$ are symmetric positive semi-definite matrices.
\end{proof}

\begin{lemma} \label{lem: linear operator of derivative}
    Consider the function
\begin{align*}
     f: \mathbb{S}_{++}^d &\rightarrow \mathbb{R} \\ 
     \Sigma &\mapsto -\tr\left(\big[\Sigma^2+2\tau \Sigma\big]^{\frac{1}{2}}\right).
\end{align*}
Assume that $\Sigma \in \mathcal{K}_{[\nicefrac{1}{\rho}, \rho]}$. Then we have $f''(\Sigma): \mathbb{S}_{++}^d \times \mathbb{S}_{++}^d \rightarrow \mathbb{R}$ is positive definite. Specifically, $f$ is strongly convex with coefficient $\dfrac{\tau^2}{(\rho^2+2\tau\rho)^{3/2}}$.
\end{lemma}
\begin{proof}
    The square root function and $x \mapsto x^2+2\tau x$ respectively be the analytic function in $(0,\infty)$ and $\mathbb{R}$, thus function $x \mapsto -\sqrt{x^2+2\tau x}$ is an analytic function in $(0, \infty)$. Thus, we can define the function $f$ in the set of of SPD matrix using the Taylor expansion. Thus, the derivative of $f(\sigma)$ can be calculated as 
    \begin{equation*}
        \dfrac{df}{d\Sigma} = f'(\Sigma) =  \left(\Sigma + \tau \Idsf \right)\big[\Sigma^2+2\tau \Sigma\big]^{-\frac{1}{2}},
    \end{equation*}
    or in the operator viewpoint, 
\begin{equation*}
    \dfrac{df}{d\Sigma} : \Sigma_1 \mapsto - \tr\left(\left(\Sigma + \tau \Idsf \right)\big[\Sigma^2+2\tau \Sigma\big]^{-\frac{1}{2}} \Sigma_1 \right).
\end{equation*}
For the second derivative, we have to take the derivative respect to $\Sigma$ of the function $\dfrac{df}{d\Sigma}(\Sigma_1)$. In the interval $(-1,1)$, consider the well-known Taylor expansion
\begin{equation*}
    \sqrt{1-x}=1-\sum_{k=0}^{\infty} \frac{2}{4^{k+1}(k+1)}\binom{2 k}{k} x^{k+1}.
\end{equation*}
Taking the first and second derivatives of both sides, we achieve for $x \in (-1,1)$
\begin{align*}
    \dfrac{1}{\sqrt{1-x}} &= \sum_{k=0}^\infty \dfrac{1}{4^k}\binom{2k}{k}x^k,\\
    \dfrac{1}{(1-x)\sqrt{1-x}} &= \sum_{k=1}^\infty \dfrac{2}{4^k}\binom{2k}{k}kx^{k-1}.\\
\end{align*}
Let $g(x) = \frac{x+\tau}{x^2+2\tau x}$, then $g(x) = \frac{1}{\sqrt{1-\frac{\tau^2}{(x+\tau)^2}}}$.  Using the Taylor expansion formula, the calculation of the second derivative of $f$ can be implemented as 
\begin{align*}
    &\frac{f^{\prime}\left(\Sigma+t\Sigma_2\right) \Sigma_1-f^{\prime}(\Sigma) \Sigma_1}{t} \\
    &= \lim_{t \to 0} \frac{-1}{t} \tr \left( \{g(\Sigma+t\Sigma_2) - g(\Sigma)\} \Sigma_1 \right) \\ 
   &= \lim_{t \to 0} \dfrac{1}{t}\tr \left( \sum_{k=0}^\infty \dfrac{\tau^{2k}}{4^k}\binom{2k}{k} \left(\dfrac{1}{(\Sigma + \tau\Idsf)^{2k}} - \dfrac{1}{(\Sigma + t\Sigma_2 + \tau\Idsf)^{2k}}\right)\Sigma_1\right)\\
   &= \lim_{t \to 0}\dfrac{1}{t}\tr \left( \sum_{k=0}^\infty \dfrac{\tau^{2k}}{4^k}\binom{2k}{k} \left((\Sigma + \tau\Idsf)^{-2k} - (\Sigma + t\Sigma_2 + \tau\Idsf)^{-2k}\right)\Sigma_1\right)\\
    &= \lim_{t \to 0}\dfrac{1}{t}\tr \left( \sum_{k=0}^\infty \dfrac{\tau^{2k}}{4^k}\binom{2k}{k} (\Sigma + t\Sigma_2 + \tau\Idsf)^{-2k}\left((\Sigma + t\Sigma_2 + \tau\Idsf)^{2k} - (\Sigma  + \tau\Idsf)^{2k}\right)(\Sigma + \tau\Idsf)^{-2k}\Sigma_1\right)\\
    &= \lim_{t \to 0}\dfrac{1}{t}\tr \left( \sum_{k=1}^\infty \dfrac{\tau^{2k}}{4^k}\binom{2k}{k} (\Sigma + \tau\Idsf)^{-2k}\left(\sum_{q=0}^{2k-1}(\Sigma+\tau\Idsf)^{q}(t\Sigma_2)(\Sigma+\tau\Idsf)^{2k-1-q}\right)(\Sigma + \tau\Idsf)^{-2k}\Sigma_1\right)\\
    &= \tr \left( \sum_{k=1}^\infty \dfrac{\tau^{2k}}{4^k}\binom{2k}{k} \left(\sum_{q=0}^{2k-1}(\Sigma+\tau\Idsf)^{q-2k}\Sigma_2(\Sigma+\tau\Idsf)^{-1-q}\Sigma_1\right)\right).
\end{align*}
Now we prove that this bilinear form is positive definite by showing that for arbitrary $\Sigma_1 = \Sigma_2$ be symmetric matrices, it has a positive lower bound. Firstly, we can verify that $\lambda_d((\Sigma+\tau\Idsf)^{-k}) \geq (\rho+\tau)^{-k}$ for $k<0$ and $\Sigma \in \mathcal{K}_{[\nicefrac{1}{\rho}, \rho]}$. In the formulae of second derivative,  let $\Sigma_1 = \Sigma_2$, according to \Cref{lem: compa eigen}, we have for $0\leq q\leq 2k-1$
\begin{align*}
    \tr\left((\Sigma+\tau\Idsf)^{q-2k}\Sigma_1(\Sigma+\tau\Idsf)^{-1-q}\Sigma_1\right) &\geq (\rho+\tau)^{q-2k}\tr(\Sigma_1(\Sigma+\tau\Idsf)^{-1-q}\Sigma_1) \\
    &\geq (\rho+\tau)^{q-2k}\lambda^{-1-q}\tr(\Sigma_1\Sigma_1) \\
    &=(\rho+\tau)^{-1-2k}\Vert \Sigma_1\Vert^2_F.
\end{align*}

Thus, we have the lower bound for the coefficient $\alpha$-convexity of the second derivative is 
\begin{equation*}
    \sum_{k=1}^{\infty}\dfrac{\tau^{2k}}{4^k}\binom{2k}{k}2k(\rho+\tau)^{-1-2k}.
\end{equation*}

On the other hand, by substituting $x = \frac{\tau^2}{(\rho+\tau)^2}$ in the Taylor expansion formulae of $\frac{1}{(1-x)\sqrt{1-x}}$, we have 
\begin{align*}
    &\dfrac{1}{\left(1-\frac{\tau^2}{(\rho+\tau)^2}\right)\sqrt{1-\frac{\tau^2}{(\rho+\tau)^2}}} = \sum_{k=1}^\infty \dfrac{2}{4^k}\binom{2k}{k}k\left(\dfrac{\tau^2}{(\rho+\tau)^2}\right)^{k-1}
   \Leftrightarrow \dfrac{(\rho+\tau)^3}{(\rho^2+2\tau\rho)^{3/2}} =  \sum_{k=1}^\infty \dfrac{2k}{4^k}\binom{2k}{k}\left(\dfrac{\tau}{\rho+\tau}\right)^{2(k-1)}.
\end{align*}
In other words, the lower bound can be expressed as 
\begin{equation*}
     \sum_{k=1}^{\infty}\dfrac{\tau^{2k}}{4^k}\binom{2k}{k}2k(\rho+\tau)^{-1-2k} = \dfrac{\tau^2}{(\rho^2+2\tau\rho)^{3/2}}.
\end{equation*}
As a consequence, we obtain the conclusion of the lemma. 
\end{proof}

\begin{proposition} \label{convexity}
    $W_{2_{\operatorname{SUOT}}}^2(\Sigma_\alpha, \Sigma_\beta, \tau)$ is Euclidean convex with respect to $\Sigma_\beta$. Moreover, if $\Sigma \in \mathcal{K}_{[\frac{1}{\rho}, \rho]}$, then it is $\dfrac{\tau^2}{(\rho^2+2\tau\rho)^{3/2}}$-strongly convex.
\end{proposition} 

\begin{proof}
From the proof of Theorem \ref{UOT Derivative} with the same notation, we need to prove the Euclidean convexity of  
\begin{align*}
    \tr\big( \Sigma_\beta\big) - \tr\big(\Sigma_{x,\beta}\big) - \frac{\tau}{4} \log \det(\Sigma_{x}).
\end{align*}
The Euclidean convexity of $\tr(\Sigma_\beta)$ is obvious. We move to the next term $-\tr\big(\Sigma_{x,\beta}\big)$, which is known to be reformulated as
\begin{align*}
& - \tr \Big( \frac{1}{2}\Sigma_{\alpha,\beta,\tau}^{-1} + \frac{1}{2}\Big[ \Sigma_{\alpha,\beta,\tau}^{-2} + 2\tau \Sigma_{\alpha,\beta,\tau}^{-1}\Big]^{\frac{1}{2}} \Big) \\
 & = - \tr \Big( \frac{1}{2}\Sigma_{\alpha,\beta,\tau}^{-1} \Big) - \tr \Big( \frac{1}{2} \Big[\Sigma_{\alpha,\beta,\tau}^{-2} + 2\tau \Sigma_{\alpha,\beta,\tau}^{-1}\Big]^{\frac{1}{2}} \Big).
\end{align*}
Note that $\Sigma_{\alpha,\beta,\tau}^{-1} = \Sigma_{\alpha,\tau}^{-\frac{1}{2}}\Sigma_\beta \Sigma_{\alpha,\tau}^{-\frac{1}{2}} $ is linear transformation of $\Sigma_\beta$, then the first term is obviously convex. For the second term, we must prove the convexity of $-\tr(\big[ \Sigma^2 + 2\tau \Sigma\big]^{\frac{1}{2}})$. Consider the function
\begin{align*}
     f: \mathbb{S}_{++}^d &\rightarrow \mathbb{R} \\ 
     \Sigma &\mapsto -\tr\left(\big[\Sigma^2+2\tau \Sigma\big]^{\frac{1}{2}}\right).
\end{align*}
Due to \Cref{lem: linear operator of derivative}, $f$ is $\dfrac{\tau^2}{(\rho^2+2\tau\rho)^{3/2}}$-strongly convex.
Subsequently, we assert the claim. For the last term  $-\log\det(\Sigma_{x})$, we know that 
\begin{align*}
    - \log \det(\Sigma_{x}) = \log\det(\Sigma_{\beta}) - 2 \log \det(\Sigma_{x,\beta}).
\end{align*}
Similarly, the transformation $\Sigma_\beta \mapsto \Sigma_{\alpha,\beta, \tau}^{-1}$ is linear, then we can convert our problem into proving the convexity of
\begin{align*}
 \log \det(\Sigma)-2 \log \det\left(\frac{1}{2}\left(\Sigma+\left[\Sigma^2+2 \tau \Sigma\right]^{\frac{1}{2}}\right)\right) 
& =-2 \log \det\left(\frac{1}{2} \Sigma^{-\frac{1}{2}} \left(\Sigma+\left[\Sigma^2+(2 \tau \Sigma)\right]^{\frac{1}{2}}\right)\right) \\
& =-2 \log \det\left(\frac{1}{2}\left(\Sigma^{\frac{1}{2}}+ \left[\Sigma+2 \tau \Idsf \right]^{\frac{1}{2}}\right)\right) .
\end{align*}
It is equivalent to prove the concavity of 
\begin{align*}
    \Sigma \mapsto \log \det\left(\Sigma^{\frac{1}{2}}+\left[\Sigma+2 \tau \Idsf \right]^{\frac{1}{2}}\right):=f(\Sigma).
\end{align*}
Consider two matrices $\Sigma_1, \Sigma_2$, by the concavity of $\log \det$ function
\begin{align*}
    \frac{f\left(\Sigma_1\right)+f\left(\Sigma_2\right)}{2} \leq \log \det\left(\frac{\Sigma_1^{\frac{1}{2}}+\Sigma_2^{\frac{1}{2}}}{2}+\frac{\left(\Sigma_1+2 \tau \Idsf\right)^{\frac{1}{2}}+\left(\Sigma_2+2 \tau \Idsf\right)^{\frac{1}{2}}}{2}\right).
\end{align*}
By the fact that two SPD matrices $A,B$ satisfy $A^2 \succeq B^2$ then $A \succeq B$, we have inequality $\frac{\Sigma_1^{\frac{1}{2}}+\Sigma_2^{\frac{1}{2}}}{2}\preceq \Big(\frac{\Sigma_1 + \Sigma_2}{2}\Big)^{\frac{1}{2}}$. Indeed, we have $\left(\Sigma_1^{\frac{1}{2}}-\Sigma_2^{\frac{1}{2}}\right)^{2} \succeq 0$, which is equivalent to $\Sigma_1+\Sigma_2+\Sigma_1^{\frac{1}{2}} \Sigma_2^{\frac{1}{2}}+\Sigma_2^{\frac{1}{2}} \Sigma_1^{\frac{1}{2}} \preceq 2 (\Sigma_1+\Sigma_2)$. Then we have $\frac{\Sigma_1^{\frac{1}{2}}+\Sigma_2^{\frac{1}{2}}}{2} \preceq \Big(\frac{\Sigma_1 + \Sigma_2}{2}\Big)^{\frac{1}{2}}$.
Hence, we yield
\begin{align*}
    \frac{f\left(\Sigma_1\right)+f\left(\Sigma_2\right)}{2} & \leq \log \det\left(\frac{\Sigma_1^{\frac{1}{2}}+\Sigma_2^{\frac{1}{2}}}{2}+\frac{\left(\Sigma_1+2 \tau \Idsf\right)^{\frac{1}{2}}+\left(\Sigma_2+2 \tau \Idsf\right)^{\frac{1}{2}}}{2}\right) \\
    & \leq \log \det \left(\left(\frac{\Sigma_1 + \Sigma_2}{2}\right)^{\frac{1}{2}} + \big[\frac{\Sigma_1 + \Sigma_2}{2} + 2 \tau \Idsf \big]^{\frac{1}{2}}\right) \\
    & = f \left( \frac{\Sigma_1 + \Sigma_2}{2}\right).
\end{align*}
Upon achieving the required concavity, we conclude the convexity.
Hence we finish the proof.
\end{proof}

\subsection{Proof of \Cref{theo: exact converges}} \label{proof for exact converges}
\begin{proof}
    
First, we must show the $\alpha$-smoothness of $W^2_{2_{\operatorname{SUOT}}}(\alpha, \beta, \tau)$ on Bures-manifold. Consider the functional $\mathcal{F}: \mathcal{P}_{2, ac} \left(\mathbb{R}^d \times \mathbb{R}^d\right) \times \mathcal{P}_{2, ac} \left(\mathbb{R}^d \times \mathbb{R}^d\right) \rightarrow \mathbb{R}$ defined as
\begin{align*}
    \mathcal{F}\left(\Sigma_\beta, \Sigma_x\right)=W_2^2\left(\Sigma_x, \Sigma_\beta\right)+\tau \mathrm{KL}\left(\Sigma_x \| \Sigma_\alpha\right).
\end{align*}

In fact, $W^2_{2_{\operatorname{SUOT}}}(\alpha, \beta, \tau)=\min _{\Sigma_x} \mathcal{F}\left(\Sigma_\beta, \Sigma_x\right)$, and for each $\Sigma_\beta$ let
\begin{align*}
    \Sigma_{\beta, x}=\underset{\Sigma_x}{\argmin} \mathcal{F}\left(\Sigma_\beta, \Sigma_x\right).
\end{align*}
Due to inequality $3.3$ in \cite{altschuler2021averaging}, we know that $\mathcal{F}$ is $1$-smooth in the first variable.




Thus for $W^2_{2_{\operatorname{SUOT}}}(\alpha, \beta, \tau)$, consider two points $\Sigma_{\beta_0}, \Sigma_{\beta_1}$ in $\mathbb{S}_{++}^d$. Let $\Sigma_{\beta_s} (0 \leq s \leq 1)$ be geodesic connecting them. Then we have 
\begin{align*}
    & \hspace{- 4 em} (1-s) \mathcal{F}\left(\Sigma_{\beta_0}, \Sigma_{\beta_0, x}\right)+s \mathcal{F}\left(\Sigma_{\beta_1}, \Sigma_{\beta_1, x}\right)-s(1-s) W_2^2\left(\Sigma_{\beta_0}, \Sigma_{\beta_1}\right) \\
    & \leq (1-s) \mathcal{F}\left(\Sigma_{\beta_0}, \Sigma_{\beta_s, x}\right) + s \mathcal{F}\left(\Sigma_{\beta_1}, \Sigma_{\beta_s, x}\right) - s(1-s) W_2^2\left(\Sigma_{\beta_0}, \Sigma_{\beta_1}\right) \\
    & \leq  \mathcal{F}\left(\Sigma_{\beta_s}, \Sigma_{\beta_s, x}\right),
\end{align*}
where the first inequality happens due to $\Sigma_{\beta_0, x}=\underset{\Sigma_x}{\argmin} \mathcal{F}\left(\Sigma_{\beta_0}, \Sigma_x\right), \Sigma_{\beta_1, x}=\underset{\Sigma_x}{\argmin} \mathcal{F}\left(\Sigma_{\beta_1}, \Sigma_x\right)$ and the second inequality happens due to $1$-smoothness of $\mathcal{F}$ according to $\Sigma_\beta$
Then $W_2^2(\alpha, \beta, \tau)$ is $1$-smooth in $\mathcal{P}_{2, ac} \left(\mathbb{R}^d\right)$ with Bures-Wasserstein metric. Thus, 
\begin{align*}
    L\Big(\Sigma_\beta \Big) = \frac{1}{n} \sum_{i=1}^n W^2_{2_{\operatorname{SUOT}}}(\Sigma_{\alpha_i}, \Sigma_\beta, \tau)
\end{align*}
is also $1$-smooth. That means if $\Sigma_{\beta}^{(k)},\Sigma_{\beta}^{(k+1)}$ are two SPD matrices in update process, then from the 1-smoothness of the barycenter functional, we obtain the
descent step 
\begin{align*}
    L\left(\Sigma_{\beta}^{(k+1)}\right)-L\left(\Sigma_{\beta}^{(k)}\right) \leq-\eta\left(1-\frac{\eta}{2}\right)\left\|\operatorname{grad} L\left(\Sigma_{\beta}^{(k)}\right)\right\|^2_{\Sigma_{\beta}^{(k)}}.
\end{align*}
Summing up gives
\begin{align*}
    L\left(\Sigma_{\beta}^{(0)}\right)-L\left(\Sigma_{\beta}^{(k+1)}\right) \geq \eta\left(1-\frac{\eta}{2}\right) \sum_{t=1}^k \| \operatorname{grad} L\left(\Sigma_{\beta}^{(k)}\right) \|^2 _{\Sigma_{\beta}^{(k)}}.
\end{align*}
It is equivalent that $\sum_{t=1}^k \| \operatorname{grad} L\left(\Sigma_{\beta}^{(k)}\right) \|^2 _{\Sigma_{\beta}^{(k)}}$ has the upper bound as $L\left(\Sigma_{\beta}^{(0)}\right)$, it is also non-decreasing sequences then $\lim_{k\rightarrow \infty}\left\|\operatorname{grad} L\left(\Sigma_{\beta}^{(k)}\right)\right\|_{\Sigma_{\beta}^{(k)}}^2=0$. It leads to that $\Sigma_{\beta}^{(k)}$ converges to the optimal point $\Sigma_\beta^*$. \\

Now to see the convergence rate, first we prove that if $\Sigma_\beta \in \mathcal{K}_{[1 / \rho, \rho]}$, then

\begin{align*}
& L(\Sigma_\beta)-L\left(\Sigma_{\beta}^*\right) \leq \frac{\rho}{8\tau^2}(\rho^2 + 2 \tau \rho)^{\frac{3}{2}}\left\|\operatorname{grad} L(\Sigma_\beta)\right\|_{\Sigma_\beta}^2.
\end{align*}

Indeed, from the second claim in \Cref{convexity}, and since $\mathcal{K}_{[1 / \rho, \rho]}$ is convex with respect to Euclidean geodesics, we see that for $\Sigma \in \mathcal{K}_{[1 / \rho, \rho]}$

\begin{align*}
L(\Sigma_\beta)-L\left(\Sigma_\beta^{\star}\right) & \leq\left\langle\nabla L(\Sigma_\beta), \Sigma_\beta-\Sigma_\beta^{\star}\right\rangle-\frac{1}{2}\dfrac{\tau^2}{(\rho^2+2\tau\rho)^{3/2}}\left\|\Sigma_\beta-\Sigma_\beta^{\star}\right\|_{\mathrm{F}}^2 \\
& =\frac{1}{2}\left\langle\ \operatorname{grad} L(\Sigma_\beta), \Sigma_\beta-\Sigma_\beta^{\star}\right\rangle_{\Sigma_\beta}-\frac{1}{2}\dfrac{\tau^2}{(\rho^2+2\tau\rho)^{3/2}}\left\|\Sigma_\beta-\Sigma_\beta^{\star}\right\|_{\mathrm{F}}^2,
\end{align*}

where the last line uses Appendix A.5 of \cite{altschuler2021averaging}. Next we observe that by combining Cauchy-Schwarz with Young's inequality we get that for all $r>0$,

\begin{align*}
\frac{1}{2}\left\langle\operatorname{grad} L(\Sigma_\beta), \Sigma_\beta-\Sigma_\beta^{\star}\right\rangle_{\Sigma_\beta} & \leq \frac{1}{2}\left\|\operatorname{grad} L(\Sigma_\beta)\right\|_{\Sigma_\beta}\left\|\Sigma_\beta-\Sigma_\beta^{\star}\right\|_{\Sigma_\beta^{-1}} \\
& \leq \frac{r}{16}\left\|\operatorname{grad} L(\Sigma_\beta)\right\|_{\Sigma_\beta}^2+\frac{1}{r}\left\|\Sigma_\beta-\Sigma_\beta^{\star}\right\|_{\Sigma_\beta^{-1}}^2 \\
& \leq \frac{r}{16}\left\|\operatorname{grad} L(\Sigma_\beta)\right\|_{\Sigma_\beta}^2+\frac{\rho}{r}\left\|\Sigma_\beta-\Sigma_\beta^{\star}\right\|_{\mathrm{F}}^2,
\end{align*}
where in the last inequality, we note that $\left\|\Sigma_\beta-\Sigma_\beta^{\star}\right\|_{\Sigma_\beta^{-1}}^2 = \tr\left( (\Sigma_\beta-\Sigma_\beta^{\star})^T \Sigma_\beta^{-1}(\Sigma_\beta-\Sigma_\beta^{\star})\right) = \tr\left( (\Sigma_\beta-\Sigma_\beta^{\star})(\Sigma_\beta-\Sigma_\beta^{\star})^T \Sigma_\beta^{-1}\right) \leq \lambda_1(\Sigma_\beta^{-1}) \tr\left( (\Sigma_\beta-\Sigma_\beta^{\star})(\Sigma_\beta-\Sigma_\beta^{\star})^T \right) \leq \rho \left\|\Sigma_\beta - \Sigma_\beta^{\star}\right\|_{\mathrm{F}}^2$. Putting $r=\frac{2}{\tau^2}\rho(\rho^2 + 2 \tau \rho)^{\frac{3}{2}}$ yields the result. Then, with the assumption that all updated covariance matrices lie in $\mathcal{K}_{[\nicefrac{1}{\rho}, \rho]}$ throughout the optimization trajectory, we have
\begin{align*}
L\left(\Sigma_\beta^{(k+1)}\right)-L\left(\Sigma_\beta^{\star}\right) & =L\left(\Sigma_\beta^{(k+1)}\right)-L\left(\Sigma_\beta^{(k)}\right)+L\left(\Sigma_\beta^{(k)}\right)-L\left(\Sigma_\beta^{\star}\right) \\
& \leq-\eta\left(1-\frac{\eta}{2}\right)\left\|\operatorname{grad} L\left(\Sigma_{\beta}^{(k)}\right)\right\|^2_{\Sigma_{\beta}^{(k)}}+L\left(\Sigma_\beta^{(k)}\right)-L\left(\Sigma_\beta^{\star}\right) \\
& \leq\left(1- \eta(1 - \frac{\eta}{2})\frac{8\tau^2}{\rho (\rho^2 + 2 \tau \rho)^{\frac{3}{2}}}\right)\left\{L\left(\Sigma_\beta^{(k)}\right)-L\left(\Sigma_\beta^{\star}\right)\right\}.
\end{align*}
Then integrating gives
\begin{align*}
    L\left(\Sigma_\beta^{(k)}\right)-L\left(\Sigma_\beta^{\star}\right) \leq \left(1- \frac{8\tau^2 \eta(1 - \frac{\eta}{2})}{\rho(d\rho^2 + 2 \tau \rho)^{\frac{3}{2}}}\right)^k \left\{L\left(\Sigma_\beta^{(0)}\right)-L\left(\Sigma_\beta^{\star}\right)\right\}.
\end{align*}
As a consequence, we obtain the conclusion of the theorem. 
\end{proof}

\subsection{Proof of \Cref{theo: hybrid
converges}} \label{proof for hybrid converges}
\begin{proof}
    We rewrite the SUOT objective function as 
\begin{align*}
    \min_{\Sigma_{\beta}, \Sigma_{x_{i (i= 1, \dots,n)}}}\sum_{i=1}^n W_2^2(\Sigma_{x_i},\Sigma_{\beta}) + \tau \text{KL}(\Sigma_{x_i}\|\Sigma_{\alpha_i}).
\end{align*}

First we prove its Euclidean convexity in $n+1$ variable. Indeed, convexity of $\mathrm{KL}\left(\Sigma_{x_i} \| \Sigma_{\alpha_i}\right)$ according to $\Sigma_{x_i}$ is well known. We only need to prove the convexity in two variable of $W_2^2\left(\Sigma_{x_i}, \Sigma_\beta\right)$.
Consider $(\Sigma_1, \Sigma_{1}^{\prime}),(\Sigma_2, \Sigma_{2}^{\prime}) \in \mathbb{S}_{++}(\mathbb{R}^d) \times \mathbb{S}_{++}(\mathbb{R}^d)$. Taking samples $\left(X_1, X_1^{\prime}\right), \left(X_2, X_2^{\prime}\right)$ such that
\begin{align*}
& X_1 \sim \mathcal{N}\left(0, \Sigma_1\right), X_1^{\prime} \sim \mathcal{N}\left(0, \Sigma_1^{\prime}\right) \\
& X_2 \sim \mathcal{N}\left(0, \Sigma_2\right), X_2^{\prime} \sim \mathcal{N}\left(0, \Sigma_2^{\prime}\right),
\end{align*}
where $\left(X_1, X_1^{\prime}\right)
$ and $\left(X_2, X_2^{\prime}\right)$ satisfy the coupling minimizing; $\left(X_1, X_1^{\prime}\right)$ and $\left(X_2, X_2^{\prime}\right)$ are independent.
Then we have
\begin{align*}
    & t W_2^2\left(\Sigma_1, \Sigma_1^\prime\right)+(1-t) W_2^2\left(\Sigma_2, \Sigma_2^\prime\right) \\
& =\mathbb{E}\left[t\left|X_1-X_1^{\prime}\right|^2+(1-t) \| X_2-X_2^2 \| \right] \\
& =\mathbb{E}\left[\left(\sqrt{t}\left(X_1-X_1^{\prime}\right)+\sqrt{(1-t)}\left(X_2-X_2^{\prime}\right)\right]^2\right] \text { (due to the independency). } \\
& \geq W_2^2\left(t\Sigma_1+(1-t) \Sigma_{2}, t\Sigma_1^{\prime}+(1-t)\Sigma_2^{\prime}\right).
\end{align*}
Then it is Euclidean convex by the definition. Next, we see that the barycenter objective function $F$ decreases during the iterations (suppose that all the updated matrices lie on a compact set $\mathcal{K}_{[\frac{1}{\rho}, \rho]}$). Indeed, the scheme of Hybrid Bures-Wassertein algorithm could be seen as a Block Coordinate Descent on SPD Manifolds \cite{peng2023block}. At each iteration, we fix $\Sigma_{\beta}$ to update $\{\Sigma_{x_i}\}_{i=1}^n$, then fix $\{\Sigma_{x_i}\}_{i=1}^n$ to update $\Sigma_\beta$. Both these updates should be done on SPD manifolds. In detail, at iteration $k$-th,
\begin{itemize}
    \item If we fix $\Sigma_{\beta} = \Sigma_{\beta}^{(k)}$, then $\Sigma_{x_i}^{(k)}$s which are minimizer of $\sum_{i=1}^n W_2^2(\Sigma_{x_i},\Sigma_{\beta}^{(k)}) + \tau \text{KL}(\Sigma_{x_i}\|\Sigma_{\alpha_i})$ must satisfy the forms in \Cref{UOT Plan} of our paper.
    \item If we fix $\Sigma_{x_i} = \Sigma_{x_i}^{(k)} \quad \forall i=1, \dots,n$, we need to minimize $\sum_{i=1}^n W_2^2(\Sigma_{x_i},\Sigma_{\beta})$ according to $\Sigma_\beta$. The updates from \cite{chewi2020gradient} will give the minimizer for this problem as $\Sigma_{\beta}^{(k+1)}$ theoretically supported by Theorem 7 \cite{chewi2020gradient}.
\end{itemize}

Now, the objective function decreases with corresponding solutions sequence $\left\{\Sigma_\beta^{(k)}\right\}_{k=1}^{\infty}$. Moreover, the sequences $\left\{\Sigma_\beta^{(k)}\right\}_{k=1}^{\infty}$ should not tend to infinity; otherwise, the objective function will also tend to infinity. Thus, we can extract a subsequence $\left\{\Sigma_\beta^{(k_ n)}\right\}_{k=1}^{\infty}$ which converges to a limit $\tilde{\Sigma}_\beta$. $\Sigma_{x_i}^{(k_n)}$ will also converge to the corresponding $\tilde{\Sigma}_{x_i}$. Both $\tilde{\Sigma}_\beta$ and $\tilde{\Sigma}_{x_i}$ satisfy that $\frac{\partial L}{\partial \tilde{\Sigma}_\beta}=0, \frac{\partial L}{\partial \tilde{\Sigma}_{x_i}}=0$, then they will be the stationary points. Following \Cref{proof for convexity}, $W^2_{2_{\operatorname{SUOT}}}$ is strictly Euclidean convex then these points are unique optimal points. Moreover, $\left\{L\left(\Sigma_\beta^{(k_n)}, \Sigma_{x_i}^{( k_n)}\right)\right\}_{k=1}^{\infty}$ and $\left\{L\left(\Sigma_\beta^{(k)}, \Sigma_{x_i}^{(k)}\right)\right\}_{k=1}^{\infty}$ are decreasing sequences then $\{L(\Sigma_\beta^{(k)}, \Sigma_{x_i}^{(k)})\}_{k=1}^{\infty}$ decreases to this optimal value, leading to $(\Sigma_\beta^{(k)}, \Sigma_{x_i}^{(k)})$ converges to the optimal points $\tilde{\Sigma}_\beta, \tilde{\Sigma}_{x_i}$. 
\end{proof}

\section{Ablation Study} \label{Ablation study}
\begin{proposition} \label{lem: rela to extreme point}
Considering the optimizer $\Sigma_x$ of Theorem \ref{UOT Plan}, the following convergences hold
\begin{enumerate}
    \item 
        $\Sigma_x \stackrel{\tau \to 0}{\rightarrow} \Sigma_\beta$,
    \item 
        $\left\|\Sigma_x-\Sigma_\alpha\right\|_{F} \stackrel{\tau \to \infty}{\rightarrow} 0$.
\end{enumerate}
\end{proposition}
\begin{proof}
    The first part is trivial. When $\tau$ goes to zero, we could easily compute that $\Sigma_{\alpha, \tau}^{-1}$ goes to $\Idsf$, then $\Sigma_{\alpha, \tau, \beta}^{-1}$ goes to $\Sigma$ and $\Sigma_x$ goes to $\Sigma_\beta$. For the second part, we note that
    \begin{align*}
        \left\|\Sigma_{\alpha, \tau}^{-1}-\frac{2}{\tau} \Sigma_\alpha\right\|_2 \stackrel{\tau \to \infty}{\rightarrow} 0.
    \end{align*}
Actually, we have
\begin{align*}
\Sigma_{\alpha, \tau}^{-1}  =\left(\Idsf+\frac{\tau}{2} \Sigma_\alpha^{-1}\right)^{-1} =\frac{2}{\tau} \Sigma_\alpha\left(\frac{2}{\tau} \Sigma_\alpha+ \Idsf \right)^{-1},
\end{align*}
so
\begin{align*}
    \left\|\Sigma_{\alpha, \tau}^{-1}-\frac{2}{\tau} \Sigma_\alpha\right\|_F & =\left\|\frac{2}{\tau} \Sigma_\alpha\left[\left(\frac{2}{\tau} \Sigma_\alpha+\Idsf\right)^{-1}-\Idsf\right]\right\|_F \\
& \leq \frac{2}{\tau}\left\|\Sigma_\alpha\right\|_F\left\|\left(\frac{2}{\tau} \Sigma_\alpha+\Idsf\right)^{-1}-\Idsf\right\|_F.
\end{align*}
When $\tau$ goes to infinity, $\frac{2}{\tau} \Sigma_\alpha+\Idsf$ goes to $\Idsf$ then limit of RHS is 0. It follows that our comment is true. Then we have
\begin{align*}
    \left\|\frac{\tau}{2} \Sigma_{\alpha, \tau, \beta}^{-1}-\Sigma_\beta^{\frac{1}{2}} \Sigma_\alpha \Sigma_\beta^{\frac{1}{2}}\right\|_F \stackrel{\tau \rightarrow \infty}{\rightarrow} 0
\end{align*}
due to $\Sigma_{\alpha, \tau, \beta}^{-1}=\Sigma_\beta^{\frac{1}{2}} \Sigma_{\alpha, \tau}^{-1} \Sigma_\beta^{\frac{1}{2}}$. Hence, we have
\begin{align*}
    \left\|\Sigma_x-\Sigma_\alpha\right\|_F \leq \left\|\Sigma_\beta^{-\frac{1}{2}}\right\|_F^2\left(\left\|\frac{\tau}{2} \Sigma_{\alpha, \tau, \beta}^{-1}-\Sigma_\beta^{\frac{1}{2}} \Sigma_\alpha \Sigma_\beta^{\frac{1}{2}}\right\|_F+\mathcal{O}\left(\tau^{-\frac{3}{2}}\right)\right).
\end{align*}
It verifies our conclusion.
\end{proof}

The first figure visualizes our barycenter for various values of $\tau$. The experimental settings are the same as those described in Section \ref{sec:numerical_experiment} but with the variable $\tau$ ranging from 0.005 to 100. The figure shows that when $\tau$ is large enough, SUOT-based Barycenter produced by our method resembles the normal Wasserstein Barycenter (as shown in the last subfigure), thereby verifying our theoretical analysis in Proposition \ref{lem: rela to extreme point}.

\begin{figure}[!htp] 
    \centering
    \subfloat{\includegraphics[width=0.33\linewidth]{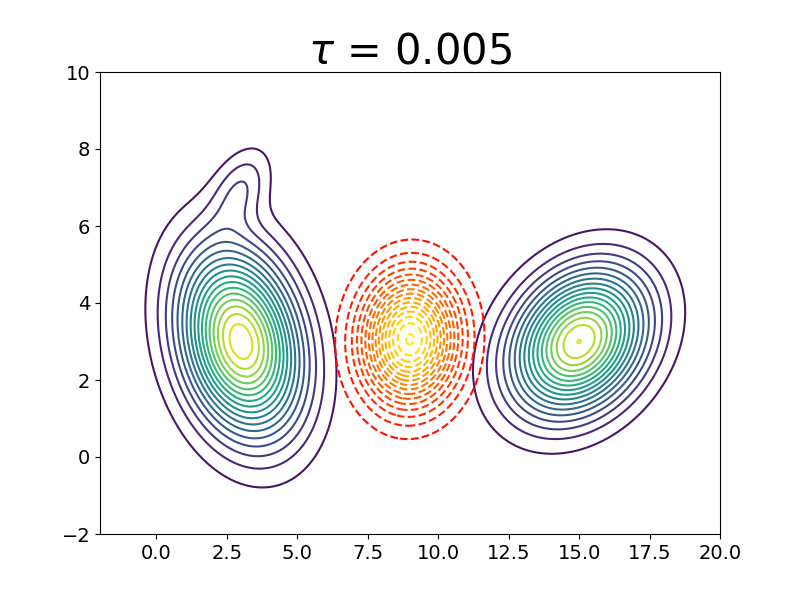}}
    \subfloat{\includegraphics[width=0.33\linewidth]{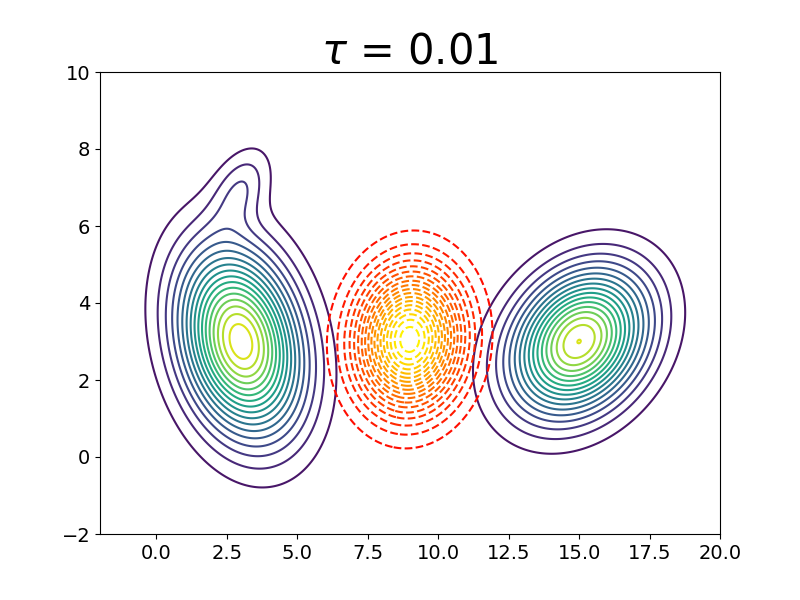}}
    \subfloat{\includegraphics[width=0.33\linewidth]{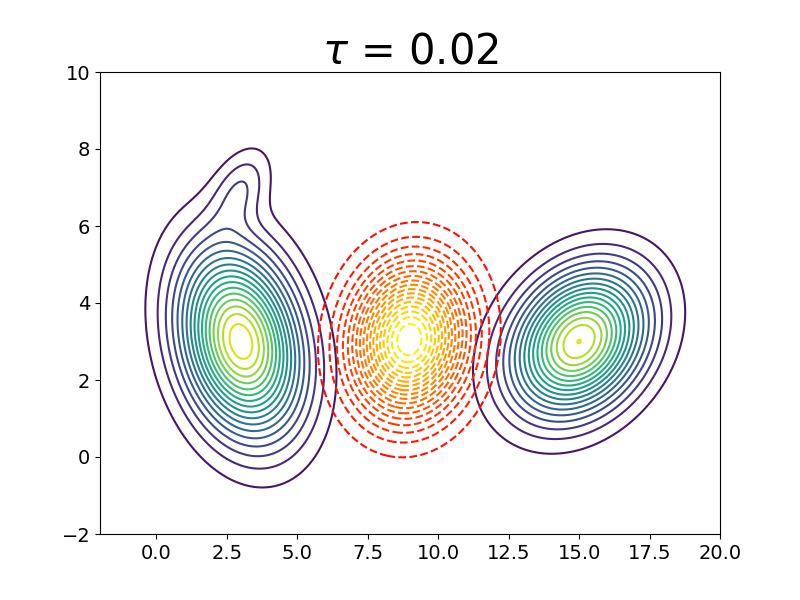}} \\
    \subfloat{\includegraphics[width=0.33\linewidth]{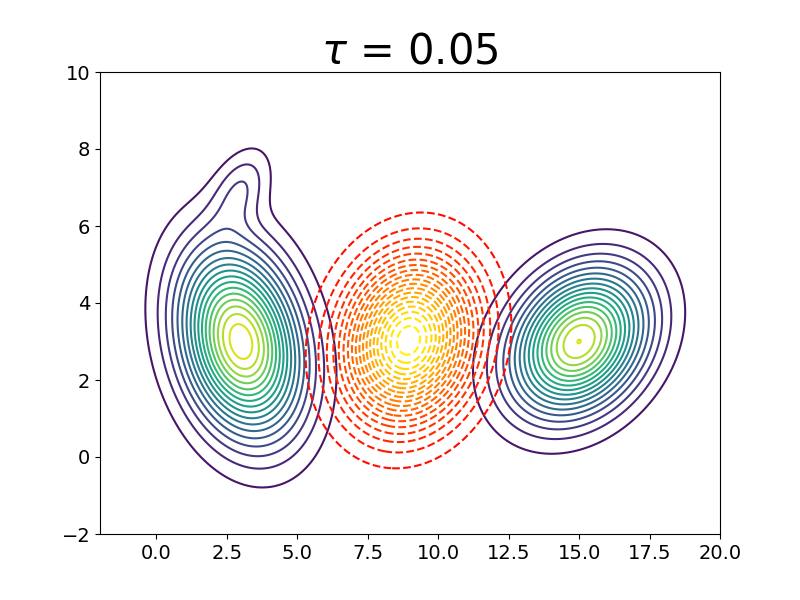}} 
    \subfloat{\includegraphics[width=0.33\linewidth]{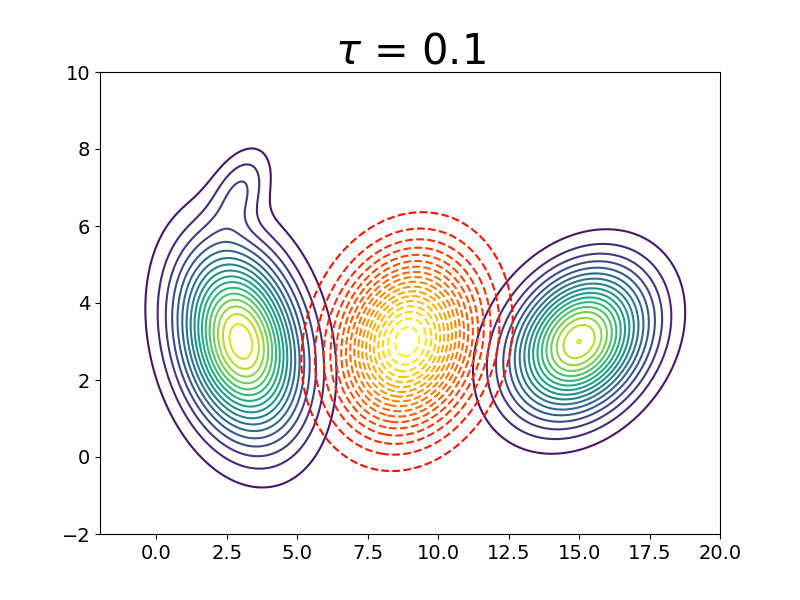}}
    \subfloat{\includegraphics[width=0.33\linewidth]{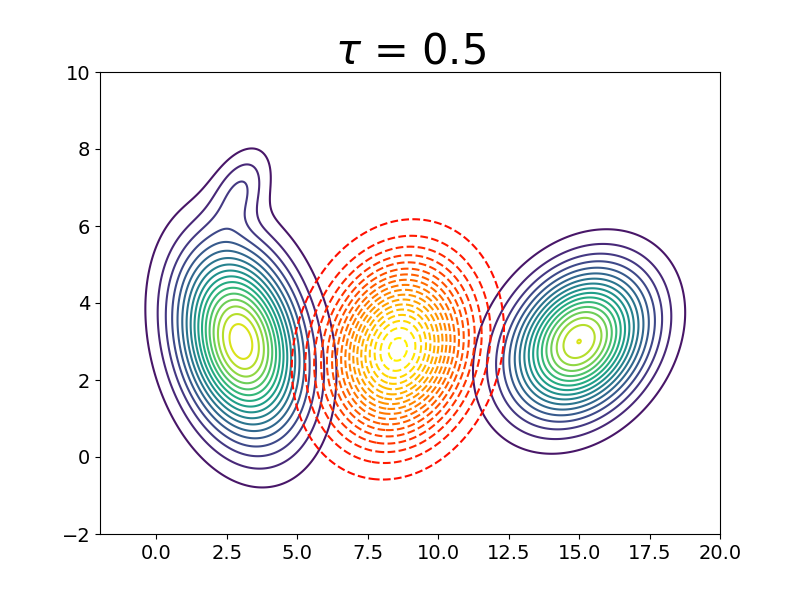}} \\
    \subfloat{\includegraphics[width=0.33\linewidth]{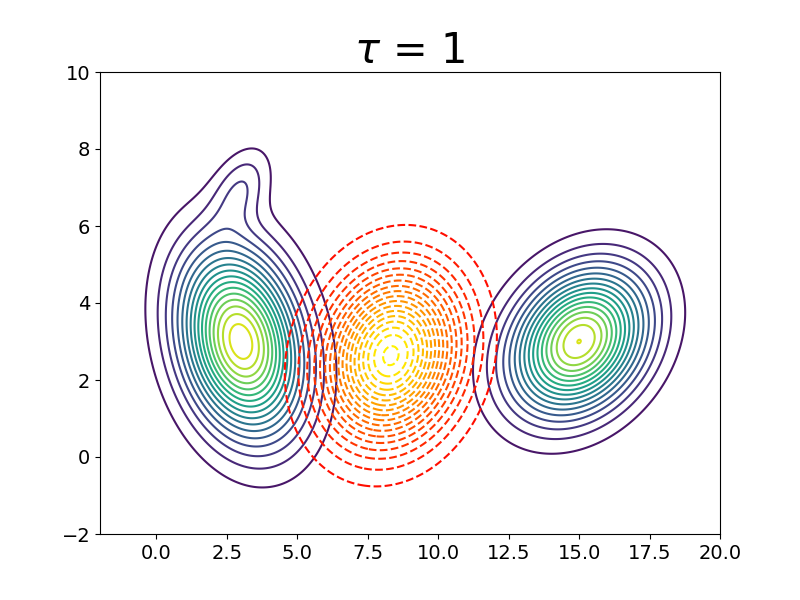}}
    \subfloat{\includegraphics[width=0.33\linewidth]{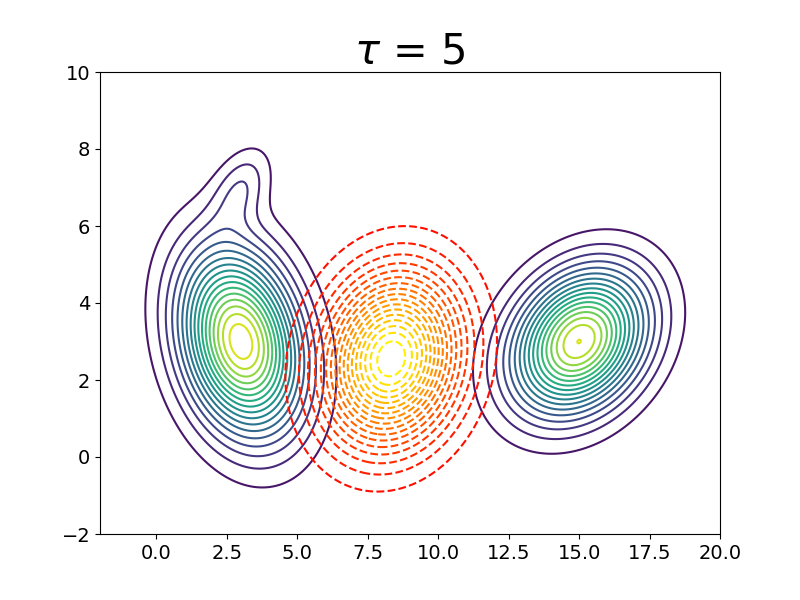}} 
    \subfloat{\includegraphics[width=0.33\linewidth]{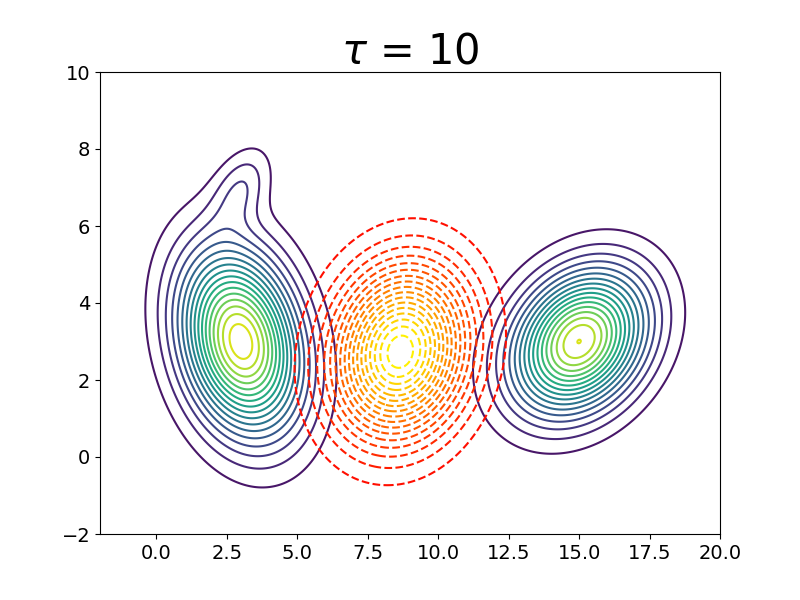}} \\
    \subfloat{\includegraphics[width=0.33\linewidth]{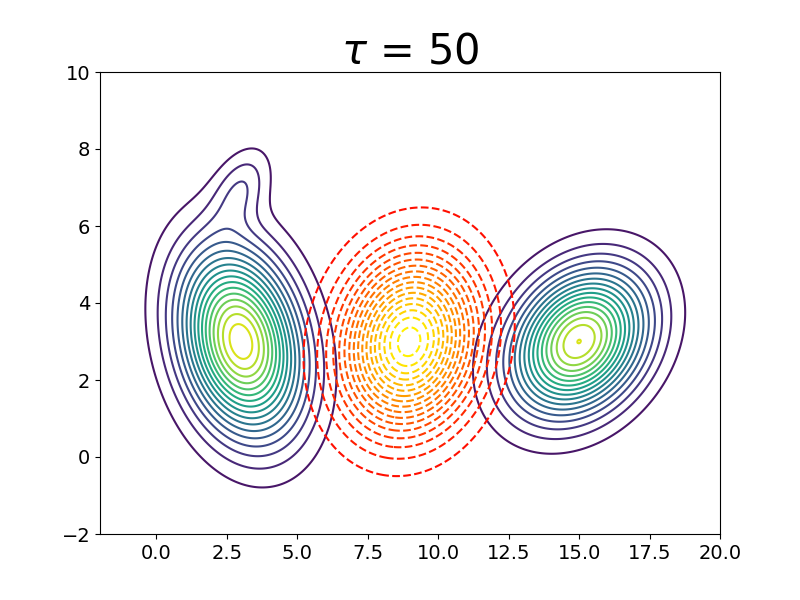}}
    \subfloat{\includegraphics[width=0.33\linewidth]{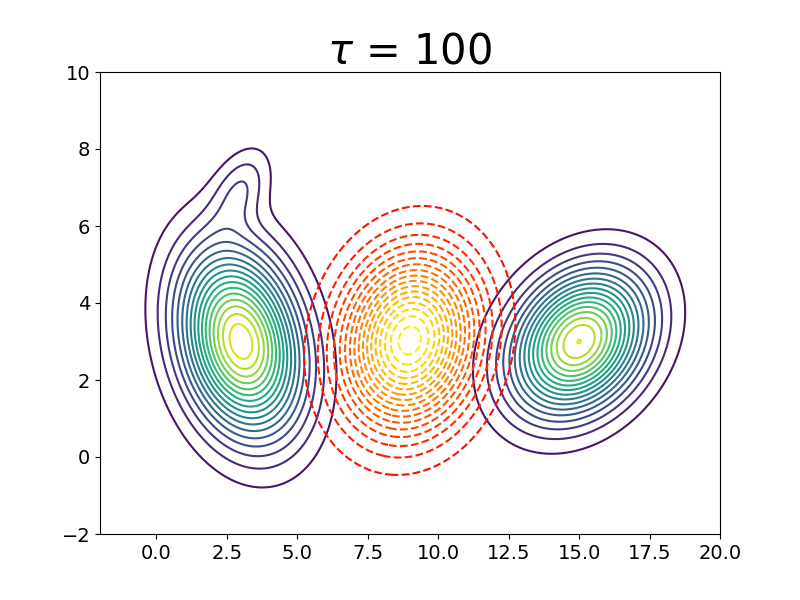}}
    \subfloat{\includegraphics[width=0.29\linewidth]{contour_ot_bary.png}}
    \caption{Ablation study of the dependence between the SUOT-based Barycenter and parameter $\tau$. From top to bottom, left to right, we calculate the barycenter with values of $\tau$ as $0.005, 0.01, 0.02, 0.05, 0.1, 0.5, 1, 5, 10, 50, 100$. The bottom right corner subfigure is normal Wasserstein Barycenter in Figure 2 of the main manuscript}
\end{figure}

\begin{figure}[!htp]
    \centering
    \includegraphics[width=0.75\linewidth]{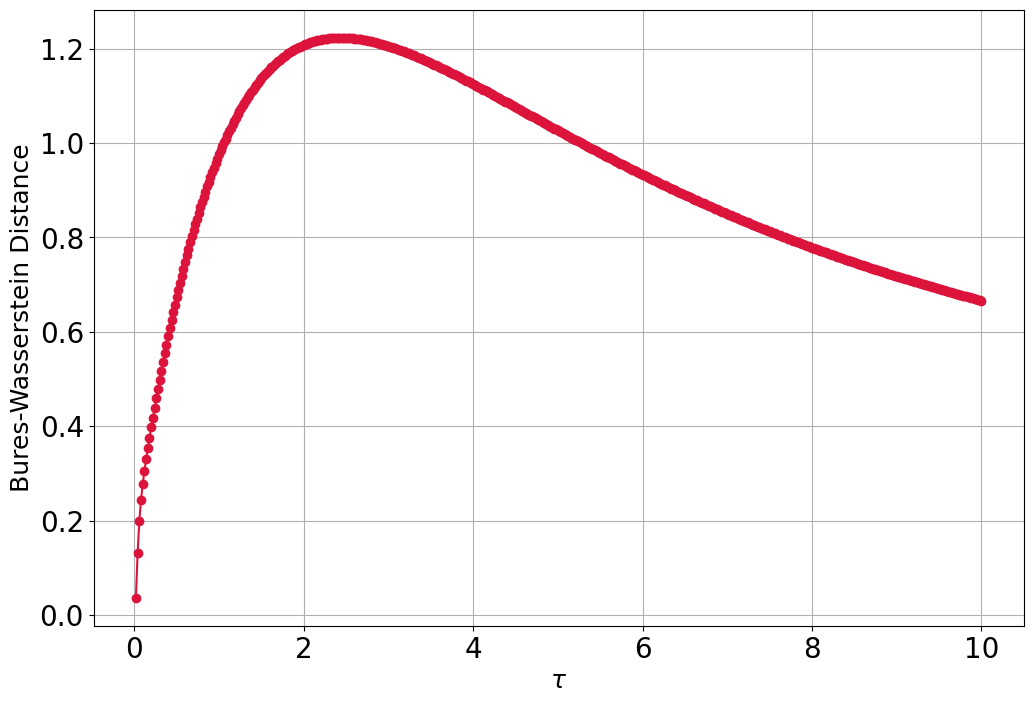}
    \caption{Bures-Wasserstein Distance from SUOT-based Barycenters to the barycenter learnt by \cite{chewi2020gradient}}
\end{figure}

To quantify this, we calculate the distances between the standard Wasserstein Barycenter learned by the method of \cite{chewi2020gradient} and our SUOT-based Barycenters using the OT distance between Gaussians, as described in Proposition 3 of Appendix A.1. The second figure illustrates the distances under the effects of varying $\tau$.

\newpage
\bibliography{sample}
\bibliographystyle{abbrv}

\end{document}